\documentclass[twoside,11pt]{article}

\usepackage{blindtext}
\usepackage{titletoc}
\usepackage{lipsum}
\usepackage{enumitem}
\usepackage{threeparttable}
\usepackage{pifont}
\usepackage{jmlr2e}

\usepackage{lastpage}
\ShortHeadings{Any-Time Regret-Guaranteed LQ Control}{Abbaszadeh Chekan and Langbort}
\firstpageno{1}

\begin{document}

\title{Any-Time Regret-Guaranteed Algorithm for Control of Linear Quadratic Systems}

\author{\name Jafar Abbaszadeh Chekan \email jafar2@illinois.edu \\
       \addr Coordinated Science Laboratory\\
       University of Illinois Urbana-Champaign\\
       Urbana, IL 61801, USA
       \AND
       \name Cedric Langbort \email langbort@illinois.edu \\
       \addr Coordinated Science Laboratory\\
      University of Illinois Urbana-Champaign\\
       Urbana, IL 61801, USA}

\editor{My editor}

\maketitle

\begin{abstract}%   <- trailing '%' for backward compatibility of .sty file
We propose a computationally efficient algorithm that achieves anytime regret of order $\mathcal{O}(\sqrt{t})$, with explicit dependence on the system dimensions and on the solution of the Discrete Algebraic Riccati Equation (DARE). Our approach builds on the SDP-based framework of \cite{cohen2019learning}, using an appropriately tuned regularization and a sufficiently accurate initial estimate to construct confidence ellipsoids for control design. A carefully designed input-perturbation mechanism is incorporated to ensure anytime performance.
We develop two variants of the algorithm. The first enforces a notion of strong sequential stability, requiring each policy to be stabilizing and successive policies to remain close. However, enforcing this notion results in a suboptimal regret scaling. The second class of algorithms removes the sequential-stability requirement and instead requires only that each generated policy be stabilizing. Closed-loop stability is then preserved through a dwell-time–inspired policy-update rule, adapting ideas from switched-systems control to carefully balance exploration and exploitation. This class of algorithms also addresses key shortcomings of most existing approaches—including certainty-equivalence–based methods—which typically guarantee stability only in the Lyapunov sense and lack explicit uniform high-probability bounds on the state trajectory expressed in system-theoretic terms. Our analysis explicitly characterizes the trade-off between state amplification and regret, and shows that partially relaxing the sequential-stability requirement yields optimal regret.
Finally, our method eliminates the need for any \emph{a priori} bound on the norm of the DARE solution—an assumption required by all existing computationally efficient optimism in the face of uncertainty (OFU) based algorithms—and thereby removes the reliance of regret guarantees on such external inputs.
\end{abstract}


\section{Introduction}

The problem of designing learning-based control algorithms for Linear Quadratic Regulator (LQR) systems with unknown dynamics—particularly under non-asymptotic performance guarantees—has received significant attention over the past decade. The LQR framework is a particularly important paradigm in the controls community, primarily due to its closed-form solution. In LQR setting the state of the system evolves according the linear dynamics  
\begin{align}
x_{t+1} &=A_{*}x_{t} + B_{*}u_{t}+\omega_{t+1},\label{eq:dyn_atttt} 
\end{align}
where matrices $A_{*}\in \mathbb{R}^{n\times n}$ and $B_{*}\in \mathbb{R}^{n\times m}$ are system matrices. The process noise $\omega_{t+1}$ usually is assumed to be drawn i.i.d from a zero-mean gaussian distribution.

The incurred quadratic stage cost $c_t$, which penalizes the state $x_t$ and control input $u_t$, is defined as
\begin{align}
c_t = x_t^\top Q x_t + u_t^\top R u_t, \label{eq:costrep}
\end{align}
where the cost matrices satisfy $Q \succeq 0$ and $R \succ 0$.

With full knowledge of system parameters $A_*$, $B_*$, and assuming controllability of $(A_*, B_*)$, it is known that the optimal sequence of control inputs $u^*=\{u^*_t\}_{t=1}^{T}$ minimizing average expected cost 
\begin{align}
J_{\mathcal{A}} =& \lim_{T\rightarrow \infty}\frac{1}{T}\mathbb{E}[\sum_{t=1}^{T} c_t]\label{eq:avExpC}
\end{align}
subject to dynamics (\ref{eq:dyn_atttt}) has a close-form solution. This sequence indeed is in the form of static state feedback $u^*_t=K_*x_t$ where $K_*$ is computed using the solution of the Discrete Algebraic Riccati Equation (DARE), $P_*$. The optimal average expected cost is \( J_{*}=P_* \bullet W =\operatorname{Tr}(P_*W)\), where \( W \) is the covariance matrix of the Gaussian noise from which the process noise \( \omega_t \) is drawn.

% \begin{align}
% K_*(\Theta_*)= -(B_*^\top P(\Theta_*)B_*+R)^{-1}B_*^\top P(\Theta_*)A_*. \label{OptFeed}
% \end{align}
% In (\ref{OptFeed}), $P(\Theta_*)$ denotes the unique solution for the Discrete Algebraic Riccati Equation (DARE), computed by
% \begin{align}
% P(\Theta_*) = Q + A_*^\top P(\Theta_*)A_* - A_*^\top P(\Theta_*)B_(B_*^\top P(\Theta_*)B_*+R)^{-1}B_*^\top P(\Theta_*)A_*. \label{eq:DARE}
% \end{align}
% It is known that the optimal average expected cost $J_{*} = P(\Theta_*) \bullet W$.

In the absence of knowledge of the parameters $A_*$ and $B_*$, the core challenge is to design an algorithm that stabilizes the system while minimizing the accumulated quadratic cost, quantified by a finite-time performance measure. In this study, we adopt \emph{regret} as the performance measure, defined as follows. For an arbitrary algorithm $\mathcal{A}$, which generates the input sequence $u = \{u_t\}_{t=1}^{T}$, regret is given by
\begin{align}
    R_{\mathcal{A}}(T) = \sum_{t=1}^{T} \left( x_t^\top Q x_t + u_t^\top R u_t - J_{*} \right), \label{eq:Reg}
\end{align}
which captures the difference between the accumulated cost incurred by the algorithm and the best achievable cost, i.e., the optimal cost.

The overarching goal of this paper, as in the related literature, is to design an algorithm that minimizes the regret~\eqref{eq:Reg} subject to the system dynamics~\eqref{eq:dyn_atttt}, when the parameters $A_*$ and $B_*$ are unknown.

\subsection{Contributions and Informal Statements of Results} 

The first algorithm proposed for the control of LQR with a worst-case regret bound of $\mathcal{O}(\sqrt{T})$ is the \textit{Optimism in the Face of Uncertainty (OFU)-based algorithm} introduced by \cite{abbasi2011regret}. This algorithm performs system identification by constructing confidence ellipsoids and designs the control policy by playing optimistically with respect to this set. Although the algorithm enjoys strong regret guarantees, its nonconvexity and the resulting computational burden has spurred subsequent research into more efficient formulations.

Despite success in addressing the computational complexity gap, most existing approaches in the literature are horizon-dependent, as they are designed under the assumption that the time horizon \(T\) is known \emph{a priori} for an episode. These methods typically provide a regret guarantee of order \( \mathcal{O}(\sqrt{T}) \) over the specified horizon, usually evaluated at the end of the episode. To extend such performance to longer horizons, the commonly adopted remedy is the doubling trick. However, even with this technique, there remains a reliance on prior knowledge of the epochs. This is particularly problematic in control applications, where the horizon is often unknown or subject to change, making horizon-dependent algorithms less suitable for real-world deployment. Therefore, algorithms with anytime guarantees, which do not require such prior knowledge, are generally more practical. In this work, we propose a new \emph{anytime} regret-minimizing algorithm that does not require knowledge of the horizon in advance. Our approach builds on the OFU-based method introduced by \cite{cohen2019learning}.

Another property that may be important for making learning-based algorithms transferable across control domains is the provision of system-theoretic guarantees on the state-norm trajectory. Establishing such bounds, however, requires a thorough analysis of the closed-loop stability induced by the proposed algorithms. Generating individually stabilizing policies is not sufficient to ensure closed-loop stability, as instability (or state explosion) may still occur due to policy updates, a well known fact in literature on \textit{switched-systems} control. Previous works have attempted to address this challenge through various approaches. For instance, certainty-equivalence (CE) methods \cite{simchowitz2020naive, jedra2022minimal} employ a common Lyapunov function to guarantee Lyapunov stability of the closed-loop system.
However, it is important to note that Lyapunov decay alone does not necessarily provide explicit high-probability, uniform control of the state norm in the presence of process noise. Under zero-mean Gaussian disturbances, Lyapunov decay guarantees only that the expected state norm, $\mathbb{E}\left[\|x_t\|^2\right]$, remains bounded, and deriving a high-probability bound requires additional effort. Although a high-probability state-norm bound is provided by \cite{simchowitz2020naive} (see their Lemma~5.3 and its proof), this bound is conservative and does not accurately capture the true scale of possible state-norm excursions.
 Therefore, while Lyapunov stability is sufficient for controlling accumulated regret—since a finite number of such spikes can typically be absorbed in regret bounds—it is insufficient for obtaining high-probability guarantees on the state norm, which are essential for safety-critical systems.
Another work that provides a high-probability state-norm bound is \cite{abeille2020efficient}, but the bound is not fully system-theoretic and depends on a priori bounds, which may be loose. In another work, \cite{cohen2019learning} introduced the notion of strong sequential stability, which enables an explicit high-probability guarantee on the state norm. This notion imposes a stricter requirement than standard stability: the algorithm must not only output stabilizing policies, but also ensure that any two consecutively generated policies remain close in a suitable sense, preventing the state norm from blowing up due to policy updates. While this approach provides a high-probability bound on the state norm, in this work we show that the strong sequential-stability requirement is conservative from the perspective of regret guarantees, leading to an $\mathcal{O}(\|P_*\|^{8})$ dependence on the operator norm of the DARE solution. This issue motivates us to propose a method that incorporates a mechanism inspired by \textit{dwell-time} design in switched-systems control. Using this mechanism, we carefully regulate the exploration–exploitation trade-off to improve regret bounds while explicitly restricting state amplification and ensuring uniform control of the closed-loop trajectories. A natural first idea would be to require the algorithm to generate only stabilizing policies and then rely on a dwell-time condition to control both the state norm and the regret. However, sequential stability is not a binary property; rather, it forms a tunable continuum. Our analysis shows that fully enforcing sequential stability leads to suboptimal regret, while discarding it entirely and relying solely on a dwell-time–type design for exploration–exploitation tuning also degrades regret. The optimal performance arises at an intermediate relaxation level, where stability and exploration are balanced. We identify a specific tuning that guarantees an $\mathcal{O}(\|P_*\|^{6})$ dependence in the regret bound, while maintaining control of the state norm with a dependence of order $\mathcal{O}(\|P_*\|^{2.5})$. Our analysis reveals a clear trade-off between the state-norm bound and the regret bound in terms of their dependence on $\|P_*\|$, and it provides the flexibility to tune this trade-off as needed.

Another challenge in designing regret-guaranteed algorithms for LQ control is the reliance on \textit{a priori information} as an input. In this paper, we eliminate the need for prior access to the optimal average cost $J_*$, which is commonly required in existing computationally efficient OFU-based algorithms. While this issue has been resolved for certainty-equivalence methods, as studied in~\cite{simchowitz2020naive, jedra2022minimal}, it has remained unaddressed in current computationally efficient OFU-based algorithms, including~\cite{abeille2020efficient, cohen2019learning}. Such a requirement makes the performance of these algorithms dependent on an external bound that may be loose, which is reflected in their regret upper bounds. We address this limitation by introducing a warm-up phase that runs a dual SDP problem in the loop to detect when the parameter estimate enters the desired ball, whose radius depends on $\|P_*\|$.

\subsubsection{Contributions.}

Our core contributions are threefold: (i) We introduce new \emph{computationally efficient anytime-regret} algorithms for LQ control, motivated by the need for methods that remain effective when the implementation horizon is not known a priori.
(ii) We emphasize the importance of providing a high-probability, system-theoretic bound on the state norm. We offer a new theoretical understanding of strong sequential stability, highlight its limitations, and demonstrate how relaxing this notion can improve algorithm performance. A key component of this development is a novel mechanism, inspired by dwell-time design in switched-systems control, which enables principled tuning of the exploration–exploitation trade-off while ensuring both favorable regret bounds and closed-loop stability.
(iii) We eliminate the need for a priori knowledge of the optimal average cost $J_*$, an assumption required by existing computationally efficient OFU-based algorithms.

\subsubsection{Informal statements of our main results.} 

\textbf{$(1)$ Any-time regret-guaranteed algorithm with system-theoretic bounds} We propose an any-time, regret-guaranteed, SDP-based LQ control (ARSLO) algorithm built upon the SDP-based algorithm proposed by \cite{cohen2019learning}, OSLO which is a horizon dependent algorithm. The OSLO algorithm when running for $T$ time steps requires $T$-dependent regularization parameters used for least square estimation and confidence bound construction as well as an initial parameter estimate in $\mathcal{O}(1/T^{1/4})$ neighborhood of true but unknown system dynamic parameters. The latter requires running a warm-up algorithm for $\mathcal{O}(\sqrt{T})$ time steps. To extend to a longer horizon using the doubling trick, such a warm-up phase must be implemented for the extended version in each epoch.
 Demonstrating this issue, we propose an any-time regret guaranteed algorithm by choosing an appropriate regularization parameter for the least square estimation and applying additive input perturbation $\eta_t \sim \mathcal{N}(0, \Gamma_t)$, where $\|\Gamma_t\|_{op}=\mathcal{O}(1/\sqrt{t})$. Therefore, our ARSLO algorithm, provides \emph{anytime} regret guarantees while also ensuring the strong sequential stability property of the generated policies (see Definition~\ref{def:sequentially}). Specifically, ARSLO achieves the regret bound
\begin{align*}
{R_{{ARSLO}}(t) \leq \mathcal{O}\big(\sqrt{ {n^2(n+m)^3} {\|P_*\|_{op}^{16}} \;{t} \log^4\frac{t}{\delta}} \big)}.
\end{align*}

\textbf{$(2)$ Doing away with strong sequential stability notion and improving regret bound} 

The regret bound of the ARSLO algorithm exhibits a high dependence on the operator norm of the DARE solution, \(\|P_*\|_{\mathrm{op}}\). 
We identify the notion of strong sequential stability, introduced by \cite{cohen2019learning}, as the source of this high dependence on \(\|P_*\|_{\mathrm{op}}\).
 This notion asserts that not only must every 
generated policy, be stabilizing, but \textit{any} two consecutive policies produced 
by the algorithm must also remain sufficiently close to each other to 
guarantee closed-loop stability. Considering the well-established fact that switching between stabilizing policies does not necessarily preserve closed-loop stability, \cite{liberzon2003switching, hespanha1999stability, liberzon2002basic}, the underlying idea here is that if two successive policies are sufficiently close, then updating the policy at any time will not compromise control over the growth of the state norm. Although strong sequential stability notion is presented as a definition, the meticulous measures required to enforce it effectively turn it into a restrictive assumption. We show that, for this notion to hold, the algorithm must perturb the designed feedback controller with Gaussian noise whose covariance exhibits a higher-order dependence on $\|P_*\|_{\mathrm{op}}$. This requirement introduces higher-order dependence on the DARE solution in the resulting system-theoretic regret bounds. In this work, we demonstrate that such a conservative measure is unnecessary. Instead, if the algorithm is designed so that every generated policy is strongly stabilizing, closed-loop stability can be ensured by appropriately tuning the timing of policy updates. Such appropriate tuning of policy updates is analogous to dwell-time design in switched systems \cite{liberzon2003switching, hespanha1999stability}, where a carefully chosen delay between successive policies is required to preserve stability. In our setting, this delay manifests as postponing policy updates, which we enforce by appropriately controlling the exploration–exploitation trade-off. At first glance, it may seem that completely dropping the enforcement of closeness between successive policies and relying solely on a dwell-time approach could be beneficial in terms of regret; however, our analysis shows that this is not optimal. We observe a continuum of behaviors, and to capture this, we introduce a parameter $\bar{\rho}$, which specifies any case within the spectrum between enforcing strong sequential stability and fully discarding the closeness requirement. By applying the techniques described above, we introduce the second class of algorithms named ARSLO$^+(\bar{\rho})$, which achieves an optimal regret bound of
\begin{align*}
{R_{{ARSLO^+(\bar{\rho})}}(t) \leq \mathcal{O}\big(\sqrt{ {n^2(n+m)^3} {\|P_*\|_{op}^{12}} \;{t} \log^4\frac{t}{\delta}} \big)}.
\end{align*}
for a specific value of $\bar{\rho}=2$.

Related relaxations of the sequential-stability requirement have been explored in prior work. For example, CE–based methods \cite{simchowitz2020naive, jedra2022minimal} rely on a common Lyapunov function approach. However, these methods do not provide precise, explicit, system-theoretic high-probability bounds on the state-norm trajectory, as they guarantee stability only in the Lyapunov sense. Consequently, while such approaches optimize overall regret and achieve favorable regret guarantees, they offer limited control over transient state spikes at policy update times, which may be problematic in safety-critical settings.
For OFU-based algorithms, the high-probability state-norm bound derived in \cite{abeille2020efficient} is also not fully system-theoretic, as it depends on an a priori bound on the optimal average expected cost.

\textbf{$(3)$ Contributing to the line of OFU-based algorithms for LQR control
}
Compared to the only two existing convex and computationally efficient OFU-based algorithms by \cite{cohen2019learning} and \cite{abeille2020efficient}, our OFU-based algorithm provides \emph{anytime} regret guarantees. To our knowledge, this makes it the first convex (i.e., computationally efficient) OFU-based algorithm with an \emph{anytime} regret guarantee.

Leaving aside the \emph{anytime} regret guarantee of our proposed algorithms, their system-theoretic regret upper bound complements that of the Lag-LQR approach proposed by \cite{abeille2020efficient}. 
Lag-LQR is developed by applying techniques from extended value iteration (EVI) to the Optimistic LQR with relaxed constraints, followed by a Lagrangian relaxation to design a computationally efficient algorithm. 
This algorithm comes with the following regret upper bound:
\begin{align*}
    R_{{LAG-LQ}}\leq \mathcal{O}({D^{\frac{3}{2}}}\sqrt{{n (n+m)^2} {\|P_*\|_{op}^{4}} {T} \log^3 \frac{T}{\delta}})
\end{align*}
where $\operatorname{Tr}(P_*) \leq {D}$ is an a priori known bound used by the algorithm. A similar issue also appears in \cite{cohen2019learning}. In contrast to \textit{LAG-LQR} and \textit{OSLO}, our algorithm does not require such an a priori bound, and its regret upper bound is purely system-theoretic. We address this by jointly utilizing the solutions of the primal and dual SDP formulations to design the input perturbation noise and to tune the exploration–exploitation trade-off parameter. Furthermore, the dimension dependence of our regret bound is $\mathcal{O}\big(\sqrt{n^2 (n+m)^3}\big)$, whereas for \textit{LAG-LQR} it is $\mathcal{O}\big(\sqrt{n^4 (n+m)^2}\big)$, when accounting for the hidden dependencies of $D$ on $n$ and $\|P_*\|_{\mathrm{op}}$.
% With this careful design, all terms that depend on the a priori bound are pushed into the regret of the warm-up phase, appearing only as a constant in the overall regret upper bound.

Finally, leveraging both the primal and dual solutions allows us to design input perturbations that enable \textit{non-isotropic} exploration. 
The covariance matrix of the Gaussian input perturbation is constructed from the relaxed primal SDP solution itself (not its norm), augmented with a bias term that depends on the norm of the dual solution, thereby facilitating non-isotropic exploration.

\subsection{Related Works}

In recent years, the machine learning community has shown increasing interest in addressing control design in the context of LQR, under various scenarios of what is known and what is unknown. Notably, \cite{agarwal2019online} studied the setting with (nonstochastic) adversarial perturbations and established regret guarantees. \cite{cohen2018online} investigated the control of linear systems with a known model but adversarially chosen quadratic costs, also providing regret guarantees. When the system model matrices are unknown, which is the main focus of this paper, existing efforts can be categorized into model-free and model-based settings. Model-free approaches design controllers without explicit system identification; instead, they minimize the cost function directly, relying solely on available data (see \cite{dean2020sample, fazel2018global, abbasi2019model}). On the other hand, model-based approaches perform control design through explicit model identification or confidence set construction and typically provide regret guarantees. In this section, we briefly review the regret-guarantee literature that followed the OFU-based algorithm introduced by \cite{abbasi2011regret}, which was the first algorithm for LQ control to achieve $\mathcal{O}(\sqrt{T})$ regret guarantees, albeit at a high computational cost.

In the OFU line of research, \cite{faradonbeh2020optimism} proposed an algorithm for general stabilizable systems, while also accounting for a fairly general heavy-tailed noise distribution. Despite this extension, the proposed algorithm remained computationally burdensome. \cite{dean2018regret} introduced a computationally efficient robust adaptive algorithm with polynomial time complexity that attains only regret of $\mathcal{O}(T^{2/3})$ throughout the learning process. However, the question of whether a computationally efficient (convex) algorithm can achieve an $\mathcal{O}(\sqrt{T})$ regret remained unanswered until the works proposed by \cite{cohen2019learning}, \cite{faradonbeh2020input}, and \cite{mania2019certainty}.  \cite{faradonbeh2020input} introduced an $\epsilon$-greedy-like algorithm that is computationally efficient algorithm and achieves $\mathcal{O}(\sqrt{T})$ regret, but without explicit bounds for parameter estimate errors. Similarly \cite{mania2019certainty} applied certainity equivalence principle and $\epsilon-$greedy exploration noise and proposed a computationally efficient $\mathcal{O}({\sqrt{T}})$ regret. \cite{simchowitz2020naive} proposed a certainty equivalence-based algorithm and derived a system-theoretic regret upper and lower bound with $\mathcal{O}(\sqrt{T})$. \cite{abeille2020efficient} advanced the original OFU-based algorithm by applying ideas from extended value iteration in optimistic algorithms and proposing an $\epsilon$-optimistic control, which is computationally efficient while still guaranteeing $\mathcal{O}(\sqrt{T})$ regret. The problem is also addressed by the Thomson sampling approach. In this line of work, \cite{abeille2018improved} introduced the first algorithm tailored for scalar systems, achieving a regret of $\mathcal{O}(\sqrt{T})$. Subsequently, \cite{kargin2022thompson} addressed the issue for multidimensional systems.

A commonly used categorization of regret-guaranteed algorithms for LQ control—also adopted in \cite{jedra2022minimal}—classifies existing works into (i) self-tuning regulator–based approaches and (ii) OFU–based approaches.
Self-tuning regulator algorithms  estimate the unknown system matrices $(A_*, B_*)$, treat these estimates as the true parameters, solve the corresponding DARE, and design the control input accordingly. To ensure sufficient excitation of the system—and hence enable consistent estimation of $(A_*, B_*)$—the control inputs are typically perturbed with carefully tuned noise.
OFU-based approaches, on the other hand, do not employ perturbation noise and instead design the control policy by acting optimistically with respect to a confidence set. Our algorithm, while OFU-based, also employs input perturbations for exploration. As a result, our approach can be viewed as lying at the intersection of these two categories. A more detailed discussion and quantitative comparisons are provided in Appendix \ref{secMoreDiscu}.

This paper is organized as follows.  
In Section \ref{sec:probStat}, we present the problem formulation and the necessary preliminaries for algorithm design and analysis.  
Section \ref{sec:withsequent} introduces the ARSLO algorithm, provides detailed illustrations of its main steps, and presents the corresponding guarantees.  
Section~\ref{sec:sequenttialityrel} introduces the ARSLO$^+(\bar{\rho})$ algorithm and explains how doing away with the strong notion of sequential stability, together with a dwell-time–inspired technique, leads to an improved regret bound. 
In Section \ref{sec:warmup}, we introduce the warm-up algorithm along with its technical details.  
Section \ref{philosophy} briefly summarizes the findings and discusses the underlying conceptual intuitions. Finally, the appendices contain the technical analysis, supporting lemmas, and proofs.

\section{Problem Statement and Preliminaries}\label{sec:probStat}

\textbf{Notation.} For matrix $M$, $\|M\|_* = \operatorname{tr}(\sqrt{M^\top M})$ is the trace norm, $\|M\|_F=\sqrt{\operatorname{tr} (M^\top M)}$ is the Frobenius norm and $\|M\|$ is the spectral (operator) norm. $\lambda_{\min}(.)$ represents the minimum eigenvalue. Finally, we use $A \bullet B$ to denote the entry-wise dot product between matrices, namely $A \bullet B = \operatorname{Tr}(A^\top B)$. We write \(f \lesssim g\) to denote that \(f(x) \leq C\, g(x)\) for some universal constant \(C > 0\). The notation \(\mathcal{O}\) indicates dependencies only on the time \(t\), the confidence parameter \(\delta\), \(\|P_*\|\), and the dimensions \(n\) and \(m\).

\subsection{Problem Statement}

Consider a linear time-invariant system (\ref{eq:dyn_atttt}), equivalently rewritten as
\begin{align}
	x _{t+1} =\Theta_{*}^\top z_{t}+\omega_{t+1}, \quad z_t=\begin{pmatrix} x_t \\ u_t \end{pmatrix} \label{eq:dynam_by_theta} 
\end{align}
where ${\Theta_*}=(A_*,B_*)^\top \in \mathbb{R}^{(n+m)\times n}$ initially unknown. We assume that the process noise $\omega_{t+1}$ fulfills the following assumption: 
\begin{assumption}
\label{Assumption 1}
There exists a filtration $\mathcal{F}_{t}$ such that

$(1.1)$ $\omega_{t+1}$ is a martingale difference, i.e., $\mathbb{E}[\omega_{t+1}|\;\mathcal{F}_{t}]=0$

$(1.2)$ $\mathbb{E}[\omega_{t+1}\omega_{t+1}^\top|\;\mathcal{F}_{t}]={\sigma}_{\omega}^2I_{n}=:W$ for some ${\sigma}_{\omega}^2>0$;

% $(1.3)$ $\omega_{t}$ are component-wise sub-Gaussian, i.e., there exists $\sigma_{\omega}>0$ such that for any $\gamma \in \mathbb{R}$ and $j=1,2,...,n$ we have $\mathbb{E}[e^{\gamma\omega_{j}(t+1)}|\;\mathcal{F}_{t}]\leq e^{\gamma^2\sigma_{\omega}^2/2}$.
% \begin{align*}
% \mathbb{E}[e^{\gamma\omega_{j}(t+1)}|\;\mathcal{F}_{t}]\leq e^{\gamma^2\sigma_{\omega}^2/2}.
% \end{align*}
\end{assumption}

Our goal is to design an algorithm that minimizes regret (\ref{eq:Reg}) subject to dynamics (\ref{eq:dynam_by_theta}), i.e.,

\begin{align*}
\nonumber\textrm{minimize}_{\{u_t\}_{t=0}^{T-1}}\; \; \; & R_{\mathcal{A}}(T) = \sum_{t=0}^{T-1} \left( x_t^\top Q x_t + u_t^\top R u_t - J_{*} \right)\\
&\nonumber\textrm{S.t.}\; \; \; x_{t+1} =\Theta_*^\top z_t+\omega_{t+1}.
\end{align*}

We need the following assumptions for algorithm design purposes.

\begin{assumption} \label{Ass_2}

$ $

$(1)$ There are known constants $\alpha_0, \alpha_1, \vartheta$ such that,
\begin{align*}
& \alpha_0I\preceq Q, R\preceq \alpha_1 I,\quad   \|\Theta _{*}\|\leq \vartheta. 
\end{align*}

$(2)$ There is an initial stabilizing policy $K_0$ to start the strategy with.
\end{assumption}
It is worth mentioning that, since $Q$ and $R$ are known, determining $\alpha_0$ and $\alpha_1$ is straightforward. Furthermore, unlike \cite{cohen2018online}, which requires an \textit{a priori} bound $J_* \leq \nu$, our proposed strategy eliminates the need for such an assumption.
 This is achieved by carefully leveraging both the primal and dual solutions in the design of input perturbation noise and a smart warm-up phase that continuously runs the dual SDP problem in the loop. It is also important to highlight that a similar \textit{a priori} bound is required in the other OFU-based algorithm proposed in \cite{abeille2020efficient}, where it explicitly appears in the regret upper bound. That being said, our algorithm is the first computationally efficient OFU-based method that does not require such prior knowledge.

% This flexibility arises by obtaining an upper bound for the average expected cost of the \((\kappa_0, \gamma_0)\)-stabilizing policy \(K_0\), which can then be used as a value for \(\nu\).
%  The following proposition summarizes this claim.
%     \begin{proposition} \label{prop1}
%     Assuming access to an initialized $(\kappa_0, \gamma_0)$-strong stabilizing policy $K_0$ as stated in the assumption, we can set 
%    \vspace{-0.3cm}
%     \begin{align}
%         \nu = \frac{\alpha_1 n \kappa_0^2 (1 + \kappa_0^2)}{\gamma_0}\sigma_{\omega}^2.
%     \end{align}
% \end{proposition}

% We can relax part 2 of Assumption 2 at the cost of running a nonconvex initialization algorithm that provides an initial stabilizing policy. However, for brevity, we proceed under the assumption that we already have access to such an initial policy, as in \cite{cohen2019learning} and \cite{mania2019certainty}.

% While we commit to the assumption of $J_*\leq \nu$ for some given $\nu$, we can represent that bound in terms of some constant $\kappa_*$ where $\|K(\Theta_*)\|\leq \kappa_*$. The following corollary gives such a bound. 

% \begin{corollary}
% Let  $K(\Theta_*))$ be optimal feedback gain of system (\ref{eq:dyn_atttt}) associated with quadratic cost (\ref{eq:costrep}). Then one can write

% \begin{align}
%     J_*\leq \sigma_{\omega}^2 \alpha_1 \farc{\kappa_*^2(1+\kappa_*^2)}{\gamma_*}
% \end{align}
% where 

% \end{corollary}

    We review the notions of strong stability and strong sequential stability introduced in \cite{cohen2019learning}, which will be used throughout our analysis.

{\begin{definition} (strong stability) \label{def:stronstab}
Consider the linear time invariant plant  (\ref{eq:dyn_atttt}). The closed-loop system matrix $A_*+B_*K$ is $(\kappa, \gamma)-$ strongly stable for $\kappa>0$ and $0<\gamma<1$ if there exists $H\succ I$ and $L$ such that $A_*+B_*K=HL{H}^{-1}$ and
\begin{enumerate}
    \item $\|L\|\leq 1-\gamma$ and $\|K\|\leq \kappa$
     \item $\|H\|\|{H}^{-1}\|\leq \kappa$.
\end{enumerate}
 In that case, we say that $K$ is $(\kappa, \gamma)-$strongly stabilizing for the plant $(A_*,B_*)$. 
\end{definition}}

\begin{definition}(strong sequential stability) \label{def:sequentially}
Consider the linear time-invariant plant in (\ref{eq:dyn_atttt}). The closed-loop system matrix under a sequence of policies $K_1, K_2, \ldots$ is said to be $(\kappa,\gamma)$-strongly sequentially stable for $\kappa>0$ and $0<\gamma<1$ if there exist matrices $H_1, H_2, H_3, \ldots \succ 0$ and $L_1, L_2, \ldots$ such that
\[
A_* + B_*K_t = H_t L_t H_t^{-1}, \qquad \forall t,
\]
with the following properties:
\begin{enumerate}
    \item $\|L_t\| \leq 1 - \gamma$ and $\|K_t\| \leq \kappa$,
    \item $\|H_t\| \leq B_0$ and $\|H_t^{-1}\| \leq 1/b_0$, with $\kappa = B_0/b_0$,
    \item $\|H_{t+1}^{-1} H_t\| \leq 1 + \gamma/2$.
\end{enumerate}

Furthermore, a sequence of control gains $K_1, K_2, \ldots$ is said to be $(\kappa,\gamma)$-strongly sequentially stabilizing for the plant $(A_*, B_*)$ if the sequence of closed-loop matrices $A_* + B_*K_1, A_* + B_*K_2, \ldots$ is $(\kappa,\gamma)$-strongly sequentially stable.
\end{definition}

In Appendix~\ref{sec:kapgamdef}, it is shown that any stabilizing policy \( K \) is indeed \((\kappa, \gamma)\)-strongly stabilizing for some values of \(\kappa\) and \(\gamma\). The only nontrivial component of this definition is the \emph{sequentiality} condition, which informally requires that any two consecutively generated policies remain close to each other in a precise sense, formalized by the condition \( \|H_{t+1}^{-1} H_t\| \leq 1 + \gamma/2 \). This requirement ensures that switching from one stabilizing policy to another at any time does not cause an immediate blow-up of the state norm.

While Definition \ref{def:sequentially} may appear purely definitional, ensuring that the closed-loop system remains stable under this definition requires specific measures and mechanisms in the algorithm design process. In this paper, we discuss how enforcing this requirement can lead to inefficiencies in the regret upper bound and how relaxing the sequential-stability condition can yield improvements.

\subsection{Preliminaries} \label{sec:prelim}

% \subsubsection{Confidence Ellipsoid Construction} \label{sec:systemIdentification}
Our algorithm has two main building blocks, namely confidence ellipsoid construction and control design. In this section, we provide a brief overview of each.

\textbf{Confidence Ellipsoid Construction} To begin, we outline how a confidence ellipsoid around the true but unknown parameters $\Theta_*$ is constructed using the available measurements, leaving further details to the Appendix \ref{ap:confdcos}. Let $X_t$ and $Z_t$ denote the matrices whose rows are $x_2^\top, \ldots, x_t^\top$ and $z_1^\top, \ldots, z_{t-1}^\top$, respectively.
By applying least-square estimation, regularized with $\operatorname*{Tr}((\Theta - \Theta_0)^\top (\Theta - \Theta_0))$ for some $\Theta_0$ and regularization parameter $\lambda$, the least-square estimate of $\Theta_*$ is given by:
\begin{align}
\hat{\Theta}_{t}=V_{t}^{-1}\left(Z_t^\top X_t + \lambda \Theta_0\right) \label{eq:LSE_Sol123}
\end{align}
where the covariance matrix is defined as $V_{t}=\lambda I + \sum_{s=0}^{t-1} z_{s}z_{s}^\top$.

% is defined as:
% \begin{align}
% V_{t}=\lambda I + \sum_{s=0}^{t-1} z_{s}z_{s}^T:=\lambda I+ Z_t^\top Z_t. \label{eq:covmat0}
% \end{align}
Subsequently, assuming that the initial estimate \(\Theta_0\) such that
\begin{align}
  \operatorname*{Tr}((\Theta_*-\Theta)^\top (\Theta_*-\Theta))\leq \epsilon^2 
\end{align}
 for some $\epsilon$, the confidence ellipsoid ${\mathcal{C}}_t(\delta)$ that contains $\Theta_*$ with probability at least $1-\delta$ is built as follows:
 \begin{align}
&{\mathcal{C}}_t(\delta):=\big\{\Theta \in \mathbb{R}^{(n+m)\times n}|\operatorname{Tr}\big((\Theta-\hat{\Theta}_t)^\top V_{t}(\Theta-\hat{\Theta}_t)\big)\leq r_t\big\} \label{eq:confSet1_tighterghfff2}\\
&  r_t=\bigg( \sigma_{\omega} \sqrt{2n \log\frac{n\det(V_{t}) }{\delta \det(\lambda I)}}+\sqrt{\lambda} \epsilon\bigg)^2. \label{radius_centralEl_realTime200}
\end{align}
% where 
% \begin{align}
%     r_t=\bigg( \sigma_{\omega} \sqrt{2n \log\frac{n\det(V_{t}) }{\delta \det(\lambda I)}}+\sqrt{\lambda} \epsilon\bigg)^2. \label{radius_centralEl_realTime20}
% \end{align}
\cite{cohen2019learning} uses an additional parameter $\beta$ when constructing the confidence set with the goal of normalizing the ellipsoid, which is required for their analysis. Our algorithm follows the standard framework provided by \cite{abbasi2011regret} and does not need such complication. It is also worth mentioning that this result is significantly stronger than the statement presented above. In particular, the probability guarantee is uniform over time, meaning that, with probability at least $1-\delta$, the true parameter $\Theta_*$ belongs to every confidence ellipsoid constructed up to any time $t$. This follows from the uniform-in-time self-normalized martingale concentration bound. We refer the reader to Theorem 2 of \cite{abbasi2011improved} for further details. This property will be used later in the proof of Lemma \ref{lem:epseventprob}.

% \subsubsection{Control Design}
\textbf{Control Design} Consider the Primal and Dual SDP formulation of the LQR control problem revisited in Appendix \ref{sec:primaldualSDP}. By accounting for the confidence ellipsoid (\ref{eq:confSet1_tighterghfff2}), we can relax those SDPs, to enable control design and analysis solely relying on data. The procedure of relaxing SDP formulations given such an ellipsoid is carried out by applying some perturbation technique, which has been illustrated in Lemma \ref{lem:deriveRelaxedSDPs}. The lemma generalizes the one in \cite{cohen2019learning}, which is for a normalized confidence ellipsoid.

Given the parameter estimate in the form of confidence ellipsoid $\mathcal{C}_t(\delta)$ (or simply $\mathcal{C}_t$) defined by (\ref{eq:confSet1_tighterghfff2}) the relaxed primal SDP is formulated as follows:
\begin{align}
\begin{array}{rrclcl}
\displaystyle  \operatorname{min} & \multicolumn{1}{l}{\begin{pmatrix}
	Q & 0 \\
	0 & R
	\end{pmatrix}\bullet \Sigma}\\
\textrm{s.t.} & \Sigma_{xx}\succeq {\hat{\Theta}_t}^\top \Sigma\hat{\Theta}_t+W-\mu_t\big(\Sigma\bullet {{V}^{-1}_{t}}\big)I,\\
& \Sigma\succ 0.  &
\end{array}\label{eq:RelaxedSDP}
\end{align}
where $\mu_t= r_t+\sqrt{r_t}\vartheta \|V_{t}\|^{1/2}$. The rationale for this specific form of $\mu_t$ is explained in Lemma~\ref{lem:purturbation}.
Denoting optimal solution of program (\ref{eq:RelaxedSDP}) by $\Sigma(\mathcal{C}_t)$, the control signal extracted from solving relaxed primal SDP (\ref{eq:RelaxedSDP}), $u=K(\mathcal{C}_t)x$ is deterministic and linear in state where 
\begin{align}
K(\mathcal{C}_t)=\Sigma_{ux}(\mathcal{C}_t){\Sigma_{xx}^{-1}(\mathcal{C}_t)}. \label{eq:obtPol}
\end{align} 
% The relaxed dual SDP formulation is given as follows:
% \begin{align}
% \begin{array}{rrclcl}
% \displaystyle \max & \multicolumn{1}{l}{P(\mathcal{C}_t)\bullet W}\\
% \textrm{s.t.} & \begin{pmatrix}
% Q-P(\mathcal{C}_t) & 0 \\
% 0 & R
% \end{pmatrix}+\hat{\Theta}_t P(\mathcal{C}_t) {\hat{\Theta}_t}^\top \succeq \mu_t \|P(\mathcal{C}_t)\|_*{{V}^{-1}_{t}}\\
% &P(\mathcal{C}_t) \succeq 0 
% \end{array}.\label{eq:RedSDP_DUAL} 
% \end{align}
% where we denote its optimal solution by $P_*(\hat{\Theta}_t)$.

Furthermore, the relaxed dual SDP formulation is given as follows:
\begin{align}
\begin{array}{rrclcl}
\displaystyle \max & \multicolumn{1}{l}{P\bullet W}\\
\textrm{s.t.} & \begin{pmatrix}
Q-P & 0 \\
0 & R
\end{pmatrix}+\hat{\Theta}_t P {\hat{\Theta}_t}^\top \succeq \mu_t \|P\|{{V}^{-1}_{t}}\\
&P \succeq 0 .
\end{array}\label{eq:RedSDP_DUALP} 
\end{align}
We denote its optimal solution by $P(\mathcal{C}_t)$. Note that in \cite{cohen2018online}, the term $\mu_t \|P\|_* V_t^{-1}$ on the right-hand side of their relaxed dual SDP introduces unnecessary looseness in the analysis and it can be replaced by $\mu_t \|P\| V_t^{-1}$. Our perturbation lemma (Lemma~\ref{lem:purturbation} in Appendix~\ref{eq:rlxpdfor}), together with the corresponding lemma in their analysis, shows that such an SDP relaxation is not necessary.

 Now we proceed to propose our algorithms followed by stability and regret bound guarantees, and highlighting the contributions.

\section{Any-time Regret SDP-based LQ cOntrol (ARSLO) Algorithm} \label{sec:withsequent} 

Our proposed approach consists of two main stages: a warm-up phase and an SDP-based algorithm—either ARSLO or its optimized variant ARSLO$^{+}(\bar{\rho})$.  
The warm-up phase aims to provide ARSLO or ARSLO$^{+}(\bar{\rho})$ with an estimate \(\Theta_{0}\) lying within some desired \(\epsilon\)-neighborhood of \(\Theta_{*}\), i.e., \(\|\Theta_{0}-\Theta_{*}\|_F \leq \epsilon\).  
This warm-up procedure is presented in Algorithm~\ref{alg:warmUp} in Section~\ref{sec:warmup}.

% % \subsection{Adaptive SDP-based phase}
% \subsection{Any-time Regret SDP-based LQ cOntrol (ARSLO) algorithm} \label{sec:withsequent}

In this section, we propose the anytime–regret–guaranteed algorithm ARSLO, which applies the OFU principle to generate control policies that ensure closed-loop stability of system (\ref{eq:dynam_by_theta}) in the sense of the strongly sequentially stabilizing notion defined in Definition~\ref{def:sequentially}, while minimizing the regret (\ref{eq:Reg}). 
ARSLO requires an initial estimate \(\Theta_{0}\) within an appropriate \(\bar{\epsilon}(\bar{\kappa}_1)\)-neighborhood of \(\Theta_{*}\), which is provided by the warm-up phase.
Using the initial estimate, ARSLO algorithm builds and updates a confidence ellipsoid $\mathcal{C}_t$ around the true but unknown parameters of the system $\Theta_*$, applying an appropriate regularization parameter $\lambda$. The process for construction of $\mathcal{C}_t$ has already been illustrated in Section \ref{sec:prelim}. Using this confidence set, the algorithm solves the relaxed primal SDP problem (\ref{eq:RelaxedSDP}) and computes the policy $K(\mathcal{C}_t)$ by (\ref{eq:obtPol}). When applying the policy, the algorithm adds a zero-mean gaussian random perturbation noise $\eta_t \sim \mathcal{N}\big(0,\Gamma_t\big)$ to the input, i.e., $u_t=K(\mathcal{C}_t)x_t+\eta_t$. It is worth noting that the covariance matrix $\Gamma_t$ of the Gaussian perturbation noise computed by Theorem~\ref{Stability_thm17} is constructed leveraging the solution of the relaxed primal and dual SDPs, with a feature of non-isotropic exploration.  

It can be shown that excessive updates in the policy can cause inefficiency, so there is a need for a criterion to decide on policy updates. By choosing $\det (V_t) > (1+\beta) \det (V_{\tau})$ where $V_{\tau}$ is the covariance matrix of the most recently updated policy and arbitrarily setting $\beta = 1$, we trigger a policy update when the ellipsoid volume is roughly halved. The parameters $\lambda$, $\Gamma_t$, and $\bar{\epsilon}(\bar{\kappa}_1)$ are jointly tuned to ensure that the closed-loop system is $(\kappa_*, \gamma_*)-$strongly sequentially stable where $\kappa_*:=\sqrt{2\|P_*\|/\alpha_0}$ and $\gamma_*=\kappa_*^{-2}/2$, as demonstrated in Theorem~\ref{Stability_thm17}. The rationale for choosing Gaussian perturbation noise with a covariance matrix 
\(\Gamma_t\) such that \(\|\Gamma_t\| = \mathcal{O}(1/\sqrt{t})\) 
is to ensure stability and achieve an overall regret of order \(\mathcal{O}(\sqrt{t})\). Formally,
the reasoning behind this choice is elaborated in Appendix~\ref{prop:ordernoise}.
 The algorithm enjoys any-time regret guarantees, as neither the initial estimate $\Theta_0$, the regularization parameter $\lambda$, nor the input perturbation noise $\eta_t$ depend on the time horizon of the algorithm's execution.
% The algorithm is fully adaptive as it only requires an initial estimate $\Theta_0$ which is independent of the horizon of algorithm implementation $T$. It applies a diminishing extra exploratory Gaussian noise whose variance is diminishing with $\mathcal{O}(1/t^{0.5})$. Proving that for $0<\alpha\leq 1/2$ the algorithm achives expected regret bound of order $\mathcal{O}(T^{1-\alpha})$, we show that $\mathcal{O}(\sqrt{T})$ is achieved by $\alpha=1/2$. It is worthy to note that the stability is guaranteed for $0< \alpha\leq 1/2$ as discussed in Theorem . 
\subsection{Stability Analysis of ARSLO Algorithm} \label{sec:analysis}
In this subsection, we provide sufficient conditions for closed-loop stability under the deployment of control policies generated by Algorithm~\ref{Alg:ACOLC}. The following theorem specifies how to tune the parameters $\Gamma_t$, $\lambda$, and $\bar{\epsilon}(\bar{\kappa}_1)$ to guarantee that the system is $(\kappa_*, \gamma_*)$–strong sequentially stable. 

\RestyleAlgo{ruled} 
\begin{algorithm}[H]
\caption{Any-time Regret SDP-based LQ cOntrol (ARSLO)}\label{Alg:ACOLC}
\textbf{Inputs:} {$\Theta_0$, $\bar{\kappa}_1$, and $\bar{\epsilon}(\bar{\kappa}_1)$ (provided by Algorithm \ref{alg:warmUp}), $\delta$,  $\vartheta$}

Compute $\bar{c}$ using~(\ref{eq:barcARSLO}).

Set $\lambda = \frac{\sigma_\omega^2 \bar{c}}{40}$

Initialize $V_1 = \lambda I$, $r_1 = \lambda \bar{\epsilon}^2(\bar{\kappa}_1)$, $\hat{\Theta}_1 = \Theta_0$, and construct $\mathcal{C}_1$.

\For{$t = 1, 2, ..., T$}{
  \eIf{$\det (V_t)> 2\det (V_{\tau})$ or $t=1$}{
  Calculate  $\mu_t= r_t+\sqrt{r_t}\vartheta \|V_{t}\|^{1/2}$
 
 Solve the relaxed primal SDP problem (\ref{eq:RelaxedSDP}) and compute $K(\mathcal{C}_t)$ by (\ref{eq:obtPol})

 Compute $P(\mathcal{C}_t)$ through solving the relaxed dual SDP problem (\ref{eq:RedSDP_DUALP})
 
 Let $V_{\tau}=V_t$}{Set $K(\mathcal{C}_t)=K(\mathcal{C}_{t-1})$
  and $P(\mathcal{C}_t)=P(\mathcal{C}_{t-1})$}
  Calculate $\Gamma_t$ by (\ref{eq:Gama1}) and sample $\eta_t\sim\mathcal{N}\big(0, \Gamma_{t}\big)$
  
  Play $u_t=K(\mathcal{C}_t)x_t+\eta_t$
  
  Observe $x_{t+1}$, save $(x_{t+1}, z_t)$ to the data set

  Construct $X_t$, $Z_t$ and compute $V_t$ 

  Calculate $\hat{\Theta}_t$ by (\ref{eq:LSE_Sol123}) and $r_t$ by (\ref{radius_centralEl_realTime200}) and build confidence ellipsoid $\mathcal{C}_t$
}
\end{algorithm}

% (\frac{800(1+\kappa^2)}{\gamma^2}\alpha^2(\sigma_{\omega}+\vartheta\sigma_{\eta})^2n+2\sigma_{\eta}^2m) 

\begin{theorem}[Sequential Strong Stability of Algorithm \ref{Alg:ACOLC}] \label{Stability_thm17}
Consider Algorithm \ref{Alg:ACOLC}, where the regularization parameter $\lambda$ is set as
\begin{align}
    \lambda &:= \frac{\sigma_\omega^2 \bar{c}}{80}, \label{eq:lambdapARSLO}\\
    \bar{c} &:= a^2 \Big( 2 \log\Big(\frac{4 a^2 n}{\delta}\Big) + 1 \Big)^2, \label{eq:barcARSLO}
\end{align}
where
\begin{align}
a = \frac{40960\, \bar{\vartheta}^2_{B_*}\bar{\kappa}_1^{10}\vartheta \sqrt{\sigma_{\omega}^2n(n+m)}c_z(\bar{\kappa}_1)}{\sigma_{\omega}^2 },\quad c_z^2(\kappa)={64 \kappa^6(1+\kappa^2)}\bigg(\frac{\|x_1\|}{\log \frac{1}{\delta}}+\sqrt{20n \sigma_{\omega}^2}\bigg)^2 \label{eq:defaARSLoth}
\end{align}
and
\begin{align}
    \bar{\kappa}_1 &:= \sqrt{\kappa_1^2 + \frac{1}{2 \kappa_1^2}}, \qquad
    \bar{\vartheta}_{B_*} := \max\{1, \vartheta_{B_*}\}, \quad \vartheta_{B_*} \ge \|B_*\|. \label{eq:barkapseq_Def}
\end{align}
Let the additive perturbation noise applied to the feedback control designed via the relaxed primal SDP be
\begin{align}
    \eta_t \sim \mathcal{N}\big(0, \Gamma_t\big), \qquad
    \Gamma_t := \frac{\bar{p}_t \sigma_\omega^2 \Big( K(\mathcal{C}_t) K^\top(\mathcal{C}_t) + \frac{\|P(\mathcal{C}_t)\|}{\alpha_0}I \Big)}{\sqrt{t + \bar{c}}},\label{eq:Gama1}
\end{align}
where
\begin{align}
   \bar{p}_t= \frac{5120\, \kappa_t^{10}\vartheta \sqrt{r_t}(\lambda+\|\sum_{k=1}^t z_kz_k^\top\|)^{1/2}}{\sigma_{\omega}^2 \sqrt{t+\bar{c}}}.
\end{align}

Provided with an initial estimate $\Theta_0$ satisfying
\begin{align}
    \|\Theta_0 - \Theta_*\|_F \le \bar{\epsilon}(\bar{\kappa}_1), \qquad 
    \bar{\epsilon}(\bar{\kappa}_1) := \frac{\sigma_\omega \sqrt{2 n (n+m)}}{\sqrt{\lambda}}, \label{def:bareps}
\end{align}
for any $\delta\in(0,1/3)$, with probability at least $1-3\delta$, the closed-loop system under the policies generated by Algorithm \ref{Alg:ACOLC} is \((\kappa_*, \gamma_*)\)-sequentially strongly stable where
\begin{align}
    \kappa_* := \sqrt{\frac{2 \|P_*\|}{\alpha_0}}, \qquad \gamma_* := \frac{1}{2 \kappa_*^2}.
\end{align}
\end{theorem}

Note that $\bar{\epsilon}(\bar{\kappa}_1)$, given by~(\ref{def:bareps}), is parametrized by $\bar{\kappa}_1$, which in turn depends on the solution of the relaxed dual SDP $P(\mathcal{C}_1)$ at the first time step of ARSLO (which is also the output of the warm-up phase) or equivalently on
$\kappa_1 := \sqrt{2\|P(\mathcal{C}_1)\|/\alpha_0}$.
The purpose of this parametrization, which determines how tight the initial estimate $\Theta_0$ should be, is to eliminate the need for an \textit{a priori} bound on $\|P_*\|$, a limitation present in existing OFU-based algorithms such as \cite{abeille2020efficient} and \cite{cohen2019learning}. 
This relaxation is achieved by designing a warm-up phase that iteratively constructs a confidence ellipsoid (ball), solves the relaxed dual SDP for that, computes $\bar{\epsilon}(\bar{\kappa}_1)$ from its solution, and checks if the confidence ball falls within the desired $\bar{\epsilon}(\bar{\kappa}_1)$ neighborhood, and then out puts such an appropriate $\Theta_0$ and $\bar{\epsilon}(\bar{\kappa}_1)$ to be used in ARSLO algorithm.

This relaxation is primarily implied by Lemma~\ref{lem:closeness}, which shows that the ARSLO algorithm always satisfies
\begin{align}
    P(\mathcal{C}_t) \preceq P_* \preceq P(\mathcal{C}_t) + \frac{\alpha_0}{4\kappa_t^2} I, \quad \forall t. \label{eq:goodineqkh}
\end{align}
We refer the reader to the proof of Theorem~\ref{Stability_thm17} for further details.

For a sequence of policies generated by Algorithm~\ref{Alg:ACOLC} that guarantees
$(\kappa_*,\gamma_*)$-strong sequential stability of the closed-loop system, one can show that the
state norm remains bounded. The following proposition summarizes this statement.

\begin{proposition} \label{lemma:upperBound}
Let $\delta\in (0,1/3)$. Then provided with an initial estimate $\Theta_0$ satisfying (\ref{def:bareps}), the state norm of the closed-loop system under the deployment of the ARSLO algorithm is bounded as
\begin{align}
  \|x_t\| &\leq \kappa_* e^{-{\gamma_*} (t-1)/2}\|x_1\| 
  + \frac{2\kappa_*}{\gamma_*} \max_{1\leq s\leq t}\|B_*\eta_s+\omega_{s+1}\|
  \label{eq:neatBun}
\end{align}
with probability at least $1-3\delta$.
\end{proposition}

As an immediate consequence of Theorem \ref{Stability_thm17}, and in particular the requirement (\ref{def:bareps}), the following corollary characterizes the order of magnitude of $\bar{\epsilon}(\bar{\kappa}_1)$.

\begin{corollary}\label{cl:strongseq}
For the ARSLO algorithm to generate strongly sequentially stabilizing policies, the initial estimate $\Theta_{0}$ must lie within a $\bar{\epsilon}(\bar{\kappa}_1)$-neighborhood of $\Theta_{*}$, with $\bar{\epsilon}(\bar{\kappa}_1) = \mathcal{O}\!\left(1/\kappa_*^{14}\right)$.
\end{corollary}

\subsection{Regret Bound Analysis of ARSLO Algorithm}

In this section, we provide an upper bound for the regret of the ARSLO algorithm. For this purpose, we first consider the following sequence of good events, $\mathcal{E}_t\subseteq \mathcal{E}_{t-1}\subseteq \ldots\subseteq \mathcal{E}_2\subseteq \mathcal{E}_1$ where for each $k$ 
\begin{align}
    \mathcal{E}_k = \left\{\forall s = 1, \ldots, k,\quad \Theta_* \in \mathcal{C}_s(\delta) \quad\text{and}\quad \|z_s\|^2 \leq \zeta_k^{2}(\kappa_*) \right\}\cap \mathcal{E}_0(\bar{\epsilon}(\bar{\kappa}_1)), \label{eq:Godevent1}
\end{align}
where
\begin{align}
     \mathcal{E}_0(\epsilon):=\big\{\operatorname*{Tr}((\Theta_*-\Theta_0)^\top (\Theta_*-\Theta_0))\leq \epsilon^2\big\}, \label{eq:condinitialtheta}
\end{align}
$\mathcal{C}_s(\delta)$ is the confidence set defined in (\ref{eq:confSet1_tighterghfff2}) and $\zeta_k^{2}(\kappa_*)$ is given by (\ref{eq:zetaValll}).
The event $\mathcal{E}_k$ represents the good event up to round $k$, under which the true parameter remains within the confidence sets and the norms of $\{z_s\}_{s=1}^k$ remain appropriately bounded.

Recalling the one-step regret definition $r_k=c_k-J_*$, , our goal in this section is to upper-bound $\tilde{R}_{\text{ARSLO}}(t)=\sum_{k=1}^t r_k  1_{\mathcal{E}_k}$ which coincides with the true regret $R_{\text{ARSLO}}(t)$ with high probability.

\begin{lemma} \label{lem:epseventprob}
Let $\delta\in(0,1/5)$. Then $\mathbb{P}(\mathcal{E}_t) \ge 1 - 5\delta.$
\end{lemma}

The following theorem summarizes the regret bound of the ARSLO algorithm and is followed by discussions and comparisons with selected leading works from the literature.

\begin{theorem}\label{thm:RegretBound}
Fix $\delta \in (0,1)$. Then, with probability at least $1-6\delta$, the regret of the ARSLO algorithm satisfies
\begin{align}
   R_{\mathrm{ARSLO}}(t)
   \leq
   \mathcal{O}\!\left(
   \sqrt{
      n^2 (n+m)^3 \|P_*\|^{16}\,
      t \log^4\!\frac{t}{\delta}
   }
   \right).
   \label{eq:ARSLORegBnd}
\end{align}
\end{theorem}

It is worth mentioning that the regret terms involving $\lambda$ appear only in expressions of order $\mathcal{O}(\log \frac{t}{\delta})$. Consequently, they do not contribute to the $\mathcal{O}(\sqrt{t})$ terms, which play a pivotal role in the regret upper bound. A detailed expression of the regret upper bound can be found in Appendix \ref{Sec:regretARSLOan}.

% In Section \ref{sec:warmup} it is shown that the warm-up phase in $t_w$ time steps with probability at least $1-2\delta$ gurantees that 

Since our proposed algorithm belongs to the class of OFU-based methods, it is natural to compare it with existing convex (i.e., computationally efficient) OFU-based algorithms in the literature. To the best of our knowledge, the only two algorithms of this type are those proposed in \cite{cohen2019learning} and \cite{abeille2020efficient}. The former, known as the OSLO algorithm, is technically involved and requires a warm-up phase of length $\mathcal{O}(\sqrt{T})$ for implementation up to round $T$. In addition, its regret bound exhibits a higher-order dependence on the dimension, namely $\mathcal{O}\left(\sqrt{n^{4}(n+m)^{3}}\right)$, and depends on an \textit{a priori} bound of $\|P_*\|$. A similar issue exists in the latter work, which proposes the \textit{LAG-LQ} algorithm, achieving the regret bound
\begin{align*}
    R_{{LAG-LQ}}\leq \mathcal{O}\left({D^{\frac{3}{2}}}\sqrt{{n (n+m)^2} {\|P_*\|^{4}} {T} \log^3 \frac{T}{\delta}}\right)
\end{align*}
where $\operatorname{Tr}(P_*) \leq {D}$ is an a priori known bound. Unlike the \textit{LAG-LQ} algorithm, our method does not require such an \textit{a priori} bound; accordingly, no such quantity appears in the regret bound of ARSLO. This is indeed achieved by jointly leveraging the relaxed primal and dual SDP solutions when tuning the input perturbation covariance matrix, and by employing a carefully designed warm-up phase that repeatedly solves the relaxed dual SDP problem while using the computed upper bound for the operator norm of the DARE solution as given in (\ref{eq:goodineqkh}). In this respect, our algorithm offers a complementary viewpoint to the method of \cite{abeille2020efficient}, in addition to its primary strength of coming with anytime regret guarantees. Moreover, the use of both primal and dual solutions further facilitates non-isotropic exploration. As for the dependence on the system ambient dimension and the operator norm of the DARE solution, note that, even in favor of the \textit{LAQ-LQ} algorithm, if we let \(D = \operatorname{Tr}(P_*)\leq n \|P_*\|\), it has a dependence of \(\mathcal{O}\!\bigl(\sqrt{n^{4}(n+m)^{2}\|P_{*}\|^{7}}\bigr)\). 
On the other hand, the dependence of the regret of the ARSLO algorithm on the dimension and on the solution of the DARE is \(\mathcal{O}\bigl(\sqrt{n^{2}(n+m)^{3}\|P_{*}\|^{16}}\bigr)\). 
Leaving aside the fact that we conduct this comparison in favor of the \textit{LAG-LQ} algorithm by setting $D = \operatorname{Tr}(P_{*})$, the results nevertheless indicate that the dependence on the dimension for ARSLO is of comparable order, or in some cases even better. However, our algorithm exhibits a stronger dependence on $\|P_{*}\|$, suggesting that there may still be room for improvement.

We further compare our proposed algorithm with another class of methods, namely the Certainty-Equivalence (CE) based algorithms. For instance, \cite{simchowitz2020naive} achieves a high-probability regret bound of
\begin{align}
R_{\text{CE}} \leq \mathcal{O}\left(\sqrt{n m^2 \|P_*\|^{11} T \log \frac{T}{\delta}}\right).
\end{align}
Unlike our proposed algorithm, their method does not provide an \textit{anytime} regret guarantee. However, its dependence on the system dimensions $n$ and $m$, as well as on the operator norm of the DARE solution $P_*$, is lower than that of our method. Another CE-based approach proposed by \cite{jedra2022minimal} can run in an anytime fashion, but it only provides guarantees on the expected cost. It also improves the dependence on $\|P_*\|$ compared to \cite{simchowitz2020naive}, reducing the exponent from $5.5$ to $5.25$.

While these CE-based methods ensure Lyapunov stability, they do not explicitly control the evolution of the state with high probability. In contrast, a specific advantage of our algorithm is that it provides an explicit high-probability bound on the state trajectory. Such bounds are motivated by the fact that policy updates may induce transient spikes in the state norm—a phenomenon where switching from one stabilizing controller to another can temporarily increase the state magnitude. Although these spikes have little impact on the overall regret, as their contribution is absorbed into the accumulated cost, they may be undesirable in systems with safety constraints. Both the OSLO algorithm of \cite{cohen2019learning} and our ARSLO algorithm incorporate the notion of strong sequential stability, which ensures that consecutively generated policies remain sufficiently close. This mechanism provides control over the state trajectory and may be particularly valuable when transferring the algorithm to different control domains where such guarantees are critical. While this feature improves trajectory control, it may not yet be the most efficient in terms of the regret bound's dependence on $\|P_*\|$.

In the next section, we pursue the idea of removing the notion of strong sequential stability—which introduces conservatism into the algorithm design—while still ensuring that the closed-loop system remains stable with an explicit high-probability bound on the state-norm trajectory, along with an improved (i.e., lower) dependence of the regret bound on $\|P_*\|$.

\section{ARSLO$^+(\bar{\rho})$ Algorithm: Doing Away with the Notion of Strong Sequential Stability} \label{sec:sequenttialityrel}
The strong sequential-stability requirement mandates that each individually generated policy be stabilizing and that consecutively generated policies under the ARSLO algorithm remain sufficiently close, namely, $\|P^{-1/2}(\mathcal{C}_{t+1})P^{1/2}(\mathcal{C}_t)\|\leq 1+\gamma_*/2$. Satisfying this condition necessitates a conservative algorithmic design, implemented by enforcing 
\begin{align}
    \mu_t  \| P(\mathcal{C}_t)\| V_t^{-1}\preceq \frac{\alpha_0}{32 {\kappa_t^8}}I  \label{eq: verygoodRep}
\end{align}
for any time $t$, where $\kappa_t = \sqrt{2\|P(\mathcal{C}_t)\| / \alpha_0}>1$.
 However, this enforcement is particularly restrictive, as it induces a higher-order dependence of the regret upper bound on $\|P_*\|$.

In this section, we address this limitation by relaxing the sequentiality requirement \eqref{eq: verygoodRep} to the following condition:  
\begin{align}
    \mu_t \, \| P(\mathcal{C}_t)\| V_t^{-1} \preceq \frac{\alpha_0}{4 \kappa_t^{\bar{\rho}}} I. \label{eq: verygoodRep2}
\end{align}  
for any $t$ and some $\bar{\rho} \in [0,\, 8)$. Enforcing this condition guarantees that any policy generated by the algorithm is $(\kappa_*, \gamma_*)$–strongly stabilizing and that consecutively generated policies remain within a controlled distance of one another. However, this distance requirement is weaker than that imposed by strong sequential stability, which corresponds to the case $\bar{\rho} = 8$. Under this relaxation, the closed-loop system’s state norm may not remain adequately controlled and can exhibit undesired spikes under arbitrary switching. As established in switched-systems control theory, switching among individually stabilizing policies can still induce instability or lead to temporary explosions of the state norm.
A standard remedy in this setting is to sufficiently slow down the switching. Motivated by this idea, for a given $\bar{\rho} \in [0,\,8)$, the frequency of policy updates must be appropriately limited. This is achieved by tuning the exploration–exploitation trade-off parameter $\beta(\bar{\rho})$ in the update criterion
\begin{align*}
    \det(V_t) > (1+\beta(\bar{\rho}))\det(V_{\tau})
\end{align*}
with $\bar{\rho}=8$ corresponding to strong sequential stability while the other extreme, $\bar{\rho}=0$, imposes the minimal requirement that each generated policy be individually strongly stabilizing. The key question, then, is which value of $\bar{\rho}$ achieves the optimal regret. Intuitively, setting $\bar{\rho}$ to a smaller value increases the distance between consecutively generated policies $\|P^{-1/2}(\mathcal{C}_{t^\prime})P^{1/2}(\mathcal{C}_t)\|\leq 1+2{\kappa_*^{3-\bar{\rho}/2}}$ (by Lemma \ref{lem:closeness_only_strong} and Remark \ref{eq:goodremarkb}), which requires a longer dwell-time between policy updates, i.e., larger value of $\beta(\bar{\rho})$, potentially increasing the upper bound of the regret term $R_4(t)$ (see Appendix~\ref{eq:subsecDecomp},  (\ref{eq:R4})). Conversely, a larger value of $\bar{\rho}$ increases the order of dependence of the perturbation-noise covariance matrix $\Gamma_t$ on $\|P_*\|$, which raises the upper bound of the regret term $R_6(t)$ (see Appendix~\ref{eq:subsecDecomp}, (\ref{eq:R6})). This clearly illustrates the inherent trade-off and motivates a general analysis to identify the optimal value $\bar{\rho}_*$ that achieves the best regret bound in terms of its dependence on $\|P_*\|$. This analysis also allows for a careful examination of the trade-off between the regret bound and the high-probability state-norm guarantees. The rationale for specifying $\beta(\bar{\rho})$ is provided in (\ref{eq:choicebeta}) of Theorem \ref{prop2}.

We refer to the new algorithm equipped with the measures mentioned above as \textit{ARSLO$^+(\bar{\rho})$}. This algorithm can be easily constructed by applying the above mentioned necessary adjustments to the ARSLO algorithm (Algorithm \ref{Alg:ACOLC}), but it is provided in the Appendix (see Algorithm \ref{Alg:ACOLC+}) for the sake of completeness. The following theorem illustrates how to tune the covariance matrix $\Gamma_t$ of the input perturbation $\eta_t$, the regularization parameter $\lambda$ used in least-squares estimation and confidence ellipsoid construction, and to specify a desired neigborhood $\underline{\epsilon}(\bar{\kappa}_1,\bar{\rho})$ such that $\|\Theta_0-\Theta_*\|_F\leq \underline{\epsilon}(\bar{\kappa}_1,\bar{\rho})$, so that {ARSLO$^+(\bar{\rho})$} algorithm generates $(\kappa_*, \gamma_*)$–strongly stabilizing policies for any $\bar{\rho} \in [0,\,8)$.

\subsection{Stability Analysis of ARSLO$^+(\bar{\rho})$ Algorithm} \label{sec:analysisArsloplus}

\begin{theorem} \label{Stability_thm200} 
Fix $\bar{\rho}\in [0,\,8)$. Set the regularization parameter $\lambda$ in Algorithm \ref{Alg:ACOLC+} to
\begin{align}
    \lambda &= \frac{\sigma_{\omega}^2 \bar{c}}{80} (n+m) \log \frac{1}{\delta}, \label{eq:lambda_valuefn01}\\
    \bar{c}&= a^2 \Big( 2 \log\Big(\frac{4 a^2 n}{\delta}\Big) + 1 \Big)^2, \label{eq:cfARSLOplus}
\end{align}
where
\begin{align}
   \nonumber  a &=\frac{10240 \bar{\vartheta}_{B_*}^2\bar{\kappa}_1^{\bar{\rho}+4}\vartheta\sqrt{\sigma_{\omega}^2n (n+m)}c_z(\bar{\kappa}_1,\bar{\rho})}{\sigma_{\omega}^2},\\
    c_z^2(\kappa, \bar{\rho})&={64 \kappa^6(1+\kappa^2)} (1+\xi(\kappa,\bar{\rho}))^2\bigg(\frac{\|x_1\|}{\log \frac{1}{\delta}}+\sqrt{20n \sigma_{\omega}^2}\bigg)^2,
\end{align}
$\xi(\kappa,\bar{\rho})=2\kappa^{3-\bar{\rho}/2}$ and
\begin{align}
    \bar{\kappa}_1(\bar{\rho}) :=&\sqrt{\kappa_1^2 + 4\kappa_1^{6 - \bar{\rho}}}, \quad
    \bar{\vartheta}_{B_*} := \max\{1, \vartheta_{B_*}\}, \quad \vartheta_{B_*} \ge \|B_*\|. \label{eq:kappabar_no_seq_Def}
\end{align}

Further, let the additive perturbation noise applied to the feedback control designed via the relaxed primal SDP be
\begin{align}
    \eta_t \sim \mathcal{N}\big(0, \Gamma_t\big), \qquad
    \Gamma_t := \frac{\bar{p}_t \sigma_\omega^2 \Big( K(\mathcal{C}_t) K^\top(\mathcal{C}_t) + \frac{\|P(\mathcal{C}_t)\|}{\alpha_0}I \Big)}{\sqrt{t + \bar{c}}},\label{eq:Gamarslop}
\end{align}
where
\begin{align}
   \bar{p}_t=\frac{640 \kappa_t^{\bar{\rho}+2}\vartheta\sqrt{r_t}(\lambda+\|\sum_{k=1}^tz_kz_k^\top\|)^{1/2}}{\sigma_{\omega}^2\sqrt{t+\bar{c}}}. \label{eq:pylop}
\end{align}
Provided with an initial estimate $\Theta_0$ satisfying
\begin{align}
    \|\Theta_0 - \Theta_*\|_F \le \underline{\epsilon}(\bar{\kappa}_1, \bar{\rho}), \qquad 
   \underline{\epsilon}(\bar{\kappa}_1, \bar{\rho}) := \frac{\sigma_\omega \sqrt{2 n (n+m)}}{\sqrt{\lambda}},\label{eq:epsstst0} 
\end{align}
for any $\delta\in(0,1/3)$, with probability at least $1-3\delta$, any policy generated by Algorithm \ref{Alg:ACOLC} is \((\kappa_*, \gamma_*)\)-strongly stabilizing where
\begin{align}
    \kappa_* := \sqrt{\frac{2 \|P_*\|}{\alpha_0}}, \qquad \gamma_* := \frac{1}{2 \kappa_*^2}.
\end{align}
\end{theorem}

The rationale for expressing the parameter $\underline{\epsilon}(\bar{\kappa}_1,\bar{\rho})$ in terms of $\bar{\kappa}_1$ follows the same reasoning as in the ARSLO algorithm; this parametrization ultimately allows us to eliminate the need for any a priori bound on $\|P_*\|$ through a carefully designed warm-up phase.

The following theorem characterizes the appropriate choice of the policy update parameter~$\beta(\bar{\rho})$ that guarantees boundedness of the closed-loop system state norm.

\begin{theorem} \label{prop2} 
Let $\delta\in(0,1/3)$, and let $\underline{\epsilon}(\bar{\kappa}_1, \bar{\rho})$, $\lambda$, and $\Gamma_t$ be tuned according to Theorem \ref{Stability_thm200}. Suppose the algorithm updates the policy whenever the criterion ``$\det (V_t) > (1+\beta(\bar{\rho}))\det (V_{\tau})$'' is triggered, with $\tau $ being the last time of policy update. Let
\begin{align}
    \beta(\bar{\rho}) =\begin{cases}
        2(\kappa_{1}^2 + \kappa_{1}^{(6-\bar{\rho})})^{\frac{(6-\bar{\rho})}{4}} & \text{if } \bar{\rho}<6, \\
        2 & 6\leq\bar{\rho}<8
    \end{cases} . \label{eq:choicebeta}
\end{align}
where $\kappa_1=\sqrt{2\|P(\mathcal{C}_1)\|/\alpha_0}$. Then, the closed-loop state satisfies  
   \begin{align}
    \|x_t\|\leq \kappa_* (1-\gamma_*)^{\frac{t-1}{2}} \|x_1\|+\frac{2\kappa_*(1+\xi(\kappa_*,\bar{\rho}))}{\gamma_*}\max_{1\leq s\leq t}\|B_*\eta_s+\omega_{s+1}\|\label{eq:almoststateBound}
\end{align}
with probability at least $1-3\delta$ where $\xi(\kappa_*,\bar{\rho})=2{\kappa_*^{3-\bar{\rho}/2}}$.
\end{theorem}

\begin{remark}
    Note that for $\bar{\rho} \geq 8$, the closed-loop stability guarantees hold for any choice of $\beta(\bar{\rho})> 0$. As we observed earlier, in the case $\bar{\rho} = 8$, which corresponds exactly to the strong sequential stability scenario, we arbitrarily set $\beta(\bar{\rho}) = 1$. In fact, it can be set to any positive value if our goal is solely to ensure stability in the sense of a bounded state norm. However, as we will discuss later, the choice of $\beta(\bar{\rho})$ contributes additive terms to the regret.
\end{remark}

As an immediate consequence of Theorem \ref{Stability_thm200}, and in particular the requirement (\ref{eq:epsstst0}), the following corollary characterizes the order of magnitude of $\underline{\epsilon}(\bar{\kappa}_1,\bar{\rho})$.

\begin{corollary}\label{cl:strong}
  For ARSLO$^{+}(\bar{\rho})$ algorithm to guarantee closed-loop stability, an initial estimate $\Theta_{0}$ should lie within a $\underline{\epsilon}(\bar{\kappa}_1,\bar{\rho})$-neighborhood of $\Theta_{*}$ of order  $\mathcal{O}\!\left({1}/{\kappa_*^{h}}\right)$ where
  \begin{align}
      h=\frac{\big(\max \{3-\frac{\bar{\rho}}{2},\; 1\}\big) (22+\bar{\rho})}{2}. \label{eq:Defhforre}
  \end{align}
\end{corollary}

It is straightforward to observe that \( h \) attains its minimum value of \( h = 13 \) when \( \bar{\rho} = 4 \). This implies that if one enforces condition~\eqref{eq: verygoodRep2} by choosing \( \bar{\rho} = 4 \), then the initial estimate \( \Theta_0 \) for ARSLO\(^{+}(\bar{\rho}) \) must lie in \( \mathcal{O}(1/\kappa_*^{13}) \), which is slightly looser than the requirement \( \mathcal{O}(1/\kappa_*^{14}) \) for the ARSLO algorithm given in Corollary~\ref{cl:strongseq}. Nevertheless, as we will see in Theorem~\ref{thm:RegretBoundARSLO+}, this choice is not necessarily optimal from the perspective of regret minimization.

\subsection{Regret Bound Analysis of ARSLO$^+(\bar{\rho})$ Algorithm}

In this section, we establish an upper bound on the regret of the ARSLO$^+(\bar{\rho})$ algorithm. 

Similar to ARSLO algorithm, we first consider the following sequence of good events, $\mathcal{E}^{\bar{\rho}}_t\subseteq \mathcal{E}^{\bar{\rho}}_{t-1}\subseteq \ldots\subseteq \mathcal{E}^{\bar{\rho}}_2\subseteq \mathcal{E}^{\bar{\rho}}_1$ for any $\bar{\rho}\in [0,\,8)$ where for each $k$ 
\begin{align}
    \mathcal{E}^{\bar{\rho}}_k = \left\{\forall s = 1, \ldots, k,\quad \Theta_* \in \mathcal{C}_s(\delta) \quad\text{and}\quad \|z_s\|^2 \leq \zeta_k^{2}(\kappa_*, \bar{\rho}) \right\}\cap \mathcal{E}_0(\bar{\epsilon}(\bar{\kappa}_1, \bar{\rho})), \label{eq:Godevent12}
\end{align}
in which $\mathcal{C}_s(\delta)$ is the confidence set defined in (\ref{eq:confSet1_tighterghfff2}) and $\zeta_k^{2}(\kappa_*, \bar{\rho})$ is given by (\ref{eq:zetaValllvv}).

The event $\mathcal{E}^{\bar{\rho}}_k$ represents the good event up to round $k$, under which the true parameter remains within the confidence sets and the norms of $\{z_s\}_{s=1}^k$ remain appropriately bounded.

Our goal in this section is to upper-bound $\tilde{R}_{\text{ARSLO}^+{(\bar{\rho})}}(t)=\sum_{k=1}^t r_k  1_{\mathcal{E}^{\bar{\rho}}_k}$,
where $r_k$ is the one-step regret, which coincides with the true regret $R_{\text{ARSLO}^+{(\bar{\rho})}}(t)$ with high probability. As in the ARSLO algorithm, we can show that the good event $\mathcal{E}_t^{\bar{\rho}}$ holds with probability at least $1-5\delta$.

The regret analysis indicates that optimizing the regret bound requires balancing the upper bounds of two key components, $R_4(t)$ and $R_6(t)$, given in Appendix \ref{eq:subsecDecomp} and (\ref{eq:R4}, \ref{eq:R6}),  which leads to the choice $\bar{\rho} = 2$. The following theorem summarizes the resulting regret bound for the ARSLO$^+(\bar{\rho})$ algorithm under this consideration.

\begin{theorem} \label{thm:RegretBoundARSLO+}
Fix  $\delta \in (0,1/6)$ and let $\bar{\rho}\in[0,\,8)$. Then, with probability at least $1-6\delta$, the following holds:
\begin{align}
   R_{\text{ARSLO}^+(\bar{\rho})}(t) 
   \leq &\mathcal{O}\Big(\sqrt{\max\bigg\{\|P_*\|^{9-\frac{\bar{\rho}}{2}}\,\beta_*^2\big(\|P_*\|,\bar{\rho}\big)\,\; , \|P_*\|^{(11+\frac{\bar{\rho}}{2})}\bigg\} n^2 (n+m)^3 \, t\; \log^4\frac{t}{\delta}}\,\Big) \label{eq:dom-bnd} 
\end{align}
where 
\begin{align*}
    \beta_*\big(\|P_*\|,\bar{\rho}\big) =\begin{cases}
        \big(\|P_*\| +\|P_*\|^{(6-\bar{\rho})/2}\big)^{\frac{(6-\bar{\rho})}{4}} & \text{if } \bar{\rho}<6, \\
        2 & 6\leq\bar{\rho}<8.
    \end{cases} 
\end{align*}

For $\bar{\rho}_* = 2$, the regret bound achieves its optimum:
\begin{align}
   R_{\text{ARSLO}^+(\bar{\rho}_*)}(t) 
   \leq & \mathcal{O}\Big(\sqrt{\, n^2 (n+m)^3 \; \|P_*\|^{12}\; t \;\log^4\frac{t}{\delta} }\,\Big). \label{eq:optmizeReg}
\end{align}
\end{theorem}

A detailed expression of the regret upper bound is provided in Appendix \ref{sec:regasloplusap}. In general, for a given choice of $\bar{\rho}\in [0,8)$, the dominant decomposed regret term may be either $R_4(t)$ or $R_6(t)$, which leads to the overall regret bound taking the form in (\ref{eq:dom-bnd}). Balancing these two terms is achieved at $\bar{\rho} = \bar{\rho}* = 2$. We denote the corresponding algorithm as ARSLO$^+(\bar{\rho}_*)$, with its bound given in (\ref{eq:optmizeReg}).
Comparing this algorithm with ARSLO, we observe that the regret dependence on $\|P_*\|$ improves from a power of $8$ to $6$, thanks to removing the strong sequential-stability requirement and incorporating a dwell-time–inspired approach for tuning the exploration–exploitation trade-off. A tighter analysis could further reduce this dependence to a power of approximately $4.5$, but at the cost of an exponential dependence on $\|P_*\|$ in the runtime. Setting this aside, the $\mathcal{O}(\|P_*\|^6)$ dependence of ARSLO$^+(\bar{\rho}_*)$, compared to the $\mathcal{O}(\|P_*\|^{5.5})$ dependence in \cite{simchowitz2020naive}, is slightly higher; however, it is worth emphasizing that ARSLO$^+(\bar{\rho}_*)$ provides anytime regret guarantees along with precise, explicit, high-probability guarantees on the state-norm trajectory, which may be valuable in many other control domains.

The parameter $\bar{\rho}$ is also useful for considering the trade-off between the state-norm bound in (\ref{eq:almoststateBound}) and the regret bound in (\ref{eq:dom-bnd}) when designing the algorithm. For example, in the regret-optimal algorithm ARSLO$^+(\bar{\rho}_*)$, while the dependence of the regret on $\|P_*\|$ is optimal, the state-norm dependence is $\mathcal{O}(\|P_*\|^{2.5})$. In contrast, for the sequential strong stability case with $\bar{\rho}=8$, which is less efficient in terms of regret, this dependence is of order $\mathcal{O}(\|P_*\|^{1.5})$.

\section{Warm-up Phase} \label{sec:warmup}
Recall that ARSLO and ARSLO$^{+}(\bar{\rho})$ require the initial estimate $\Theta_0$ to lie within a prescribed neighborhood of $\Theta_*$, i.e.,
\begin{align*}
    \|\Theta_0 - \Theta_*\|_F \le \varepsilon.
\end{align*}
For ARSLO, it has been shown that $\varepsilon = \bar{\epsilon}(\bar{\kappa}_1)$, where $\bar{\epsilon}(\bar{\kappa}_1)$ is given in~\eqref{def:bareps}, and
\begin{align*}
    \bar{\kappa}_1 = \sqrt{\kappa_1^2 + \frac{1}{2\kappa_1^2}},
\end{align*}
as defined in~\eqref{eq:barkapseq_Def}.

For ARSLO$^{+}(\bar{\rho})$, on the other hand, the requirement becomes $\varepsilon = \underline{\epsilon}(\bar{\kappa}_1, \bar{\rho})$, where $\underline{\epsilon}(\bar{\kappa}_1, \bar{\rho})$ is given in~\eqref{eq:epsstst0}, and 
$\bar{\kappa}_1$ is defined in (\ref{eq:kappabar_no_seq_Def}) as
\begin{align*}
    \bar{\kappa}_1 = \sqrt{\kappa_1^2 + 4\kappa_1^{6 - \bar{\rho}}}.
\end{align*}

\RestyleAlgo{ruled}
\begin{algorithm} \label{alg:warmUp}
\caption{Warm-Up Phase}

\textbf{Inputs:} $\vartheta$, $K_0$, and $\bar{\rho} \in [0,\,8)$ (when running ARSLO$^{+}(\bar{\rho})$).

Set $\varepsilon=c$ \quad \tcp{$c$ is chosen sufficiently small and is used only for initialization}

\While{$\epsilon_w(t) > \varepsilon$}{
    
    Receive state $x_t$

    Play input $u_t = K_0 x_t + \nu_t$ where $\nu_t \sim \mathcal{N}(0, 2\sigma_{\omega}^2\kappa_0^2 I)$

    Observe $x_{t+1}$

    Compute $\bar{V}_t$ and then $\hat{\Theta}_t^{w}$ using \eqref{eq:LSE_Sopp}, followed by $r_t^w$ using \eqref{eq:rwwarm}  

    Compute $\epsilon_w(t)$ using \eqref{eq:radiuswarmup} and construct the confidence ball 
    $\mathcal{C}_t^{w}$ 

    Solve the dual SDP for $P(\mathcal{C}_t^{w})$

    Compute $\kappa_1 = \sqrt{{2\|P(\mathcal{C}_t^{w})\|}/{\alpha_0}}$
    
    \tcp{Compute initial-radius thresholds for ARSLO or ARSLO$^+(\bar{\rho})$}
    For ARSLO Compute $\bar{\kappa}_1$ and $\bar{\epsilon}(\bar{\kappa}_1)$  
    using~\eqref{eq:barkapseq_Def} and~\eqref{def:bareps} then set $\varepsilon=\bar{\epsilon}(\bar{\kappa}_1)$

  For ARSLO$^{+}(\bar{\rho})$ compute $\bar{\kappa}_1$ and $\underline{\epsilon}(\bar{\kappa}_1, \bar{\rho})$ 
    using~\eqref{eq:kappabar_no_seq_Def} 
    and~\eqref{eq:epsstst0} then set $\varepsilon=\underline{\epsilon}(\bar{\kappa}_1, \bar{\rho})$

    Set $\Theta_0 = \hat{\Theta}_t^{w}$
}

\textbf{Output:} $\Theta_0$, $\bar{\kappa}_1$, and $\bar{\epsilon}(\bar{\kappa}_1)$ (or $\underline{\epsilon}(\bar{\kappa}_1, \bar{\rho})$)
\end{algorithm}

For both cases, $\kappa_1 = \sqrt{{2\|P(\mathcal{C}_1)\|}/{\alpha_0}}$
where $P(\mathcal{C}_1)$ is the solution of the relaxed dual SDP program associated with the initial ellipsoid (ball) $\mathcal{C}_1$. This ellipsoid is precisely the desired initial estimate, $\|\Theta - \Theta_0\|_F \leq \varepsilon$. Therefore, the warm-up phase must perform sufficient system identification to produce an initial estimate $\Theta_0$ and verify that it lies inside the required initial neighborhood. 

  Algorithm \ref{alg:warmUp}, employs an initial $(\kappa_0,\gamma_0)$-strongly stabilizing policy $K_0$ and applies the input $u_t=K_0 x_t+\nu_t$, where $\nu_t \sim \mathcal{N}(0, 2\sigma_{\omega}^2\kappa_0^2 I)$. Using the collected data pairs $\{(z_t,x_t)\}$, the algorithm constructs a confidence ellipsoid of the form
\begin{align*}
\operatorname{Tr}\!\left((\Theta - \hat{\Theta}^w_t)^\top \bar{V}_t (\Theta - \hat{\Theta}^w_t)\right) \le r_t^w,
\end{align*}
which is analogous to the ellipsoid described in Section~\ref{sec:prelim}, except that the least-squares estimator is regularized with respect to $\|\Theta\|_F^2$ using a regularization parameter $\rho$. The center of the confidence ellipsoid is computed by
\begin{align}
\hat{\Theta}_t^w &= \bar{V}_t^{-1} Z_t^\top X_t, \label{eq:LSE_Sopp} \\
\bar{V}_t &= \rho I + Z_t^\top Z_t,
\end{align}
where $Z_t$ and $X_t$ are constructed as described in Section~\ref{sec:prelim}. The radius parameter $r_t^w$ is defined as
\begin{align}
r_t^w =
\left(
\sigma_{\omega} \sqrt{2n \log \frac{n \det(\bar{V}_t)}{\delta \det(\rho I)}} 
+ \sqrt{\rho}\,\psi_{\Theta_*}
\right)^2,
\label{eq:rwwarm}
\end{align}
where in construction we applied $\|\Theta_*\|_F \le \psi_{\Theta_*}$. In particular, one may set $\psi_{\Theta_*} = \sqrt{n}\,\vartheta$, in which case the regularization parameter is chosen as $\rho = \sigma_{\omega}^2 \vartheta^{-2}$.

The largest confidence ball contained in this ellipsoid is then defined as
\begin{align}
\mathcal{C}_t^w(\delta)
:= \left\{ \Theta \,\middle|\, \|\Theta - \hat{\Theta}_t^w\|_F \le \epsilon_w(t) \right\},
\end{align}
where $\Theta_*$ belongs to it with probability at least \(1-2\delta\), where 
\begin{align}
\epsilon_w(t) = \sqrt{\frac{r_t^w}{\lambda_{\min}(\bar{V}_t)}}.
\label{eq:radiuswarmup}
\end{align}
and the probability is obtained via a union bound over the failure events corresponding to the state-norm bound and the lower bound on the covariance matrix, each occurring with probability at most \(\delta\).

The warm-up algorithm monitors when the confidence ball $\mathcal{C}_t^w(\delta)$ (or $\mathcal{C}_t^w$) is contained within a desired neighborhood $\mathcal{C}_1=\{\Theta \mid \|\Theta - \Theta_*\|_F \le \varepsilon\}$. A sufficient condition for this is $\epsilon_w(t) \le \varepsilon$. Specifically, at each iteration of the ``while'' loop, the algorithm uses the confidence ball $\mathcal{C}_t^w$ to solve the associated dual SDP for $P(\mathcal{C}_t^w)$. The solution is then used to compute $\kappa_1 = \sqrt{{2\|P(\mathcal{C}_t^w)\|}/{\alpha_0}}$ and subsequently $\bar{\kappa}_1$ either via \eqref{eq:barkapseq_Def} or \eqref{eq:kappabar_no_seq_Def}. Using $\bar{\kappa}_1$, the algorithm computes $\varepsilon$ (either $\bar{\epsilon}(\bar{\kappa}_1)$ or $\underline{\epsilon}(\bar{\kappa}_1, \bar{\rho})$) and evaluates the termination condition $\epsilon_w(t) \le \varepsilon$. This process continues until a time $t_w$ at which the condition is satisfied. The algorithm then outputs $\Theta_0 = \hat{\Theta}(t_w)$
 which lies within the desired neighborhood for initializing the ARSLO (or ARSLO$^+(\bar{\rho})$) algorithm. This guarantees that the subsequent learning phase satisfies the promised stability and regret guarantees. The following remark characterizes the scaling of the warm-up duration $t_w$.

\begin{remark} \label{rem:psvd}
From the preceding analysis in Theorem \ref{thm:reg_warmup}, we have 
$\lambda_{\min}(\bar{V}_t) = \mathcal{O}(t)$ and 
$r_t^w(t) = \mathcal{O}(\log t)$. 
Hence, there exists a finite time $t_w$ such that condition~\eqref{eq:radiuswarmup} is satisfied. 
Moreover, from \eqref{eq:radiuswarmup}, it follows directly that the warm-up phase duration satisfies 
\( t_w = \mathcal{O}(1/\varepsilon^2) \).
\end{remark}

The regret‐optimal algorithm, ARSLO\(^{+}(\bar{\rho}_*)\), which is attained when \(\bar{\rho} = \bar{\rho}_* = 2\) by Corollary~\ref{cl:strong}, requires an initial estimate in \(\mathcal{O}(1/\kappa_*^{24})\). This is significantly tighter than the requirement for the ARSLO algorithm, which by Corollary~\ref{cl:strongseq} is \(\mathcal{O}(1/\kappa_*^{14})\). Referring to Remark \ref{rem:psvd}, the duration of the warm-up phase is proportional to $1/\varepsilon^2$. Consequently, the warm-up phase for the ARSLO algorithm may take shorter than that of ARSLO\(^{+}(\bar{\rho}_*)\). However, it is worth noting that the regret accumulated during the warm-up phase contributes only as a constant term in the overall regret upper bound. What truly matters is the regret of the main algorithm, and a direct comparison of the regret bounds shows that ARSLO\(^{+}(\bar{\rho}_*)\) outperforms the ARSLO algorithm. The following theorem summarizes the regret bound of the warm-up phase.

\begin{theorem}\label{thm:reg_warmup}
Let $\delta\in(0,1/2)$. Then, probability at least $1-2\delta$, the following statements hold.

The regret incurred during the warm-up phase of the ARSLO algorithm satisfies
  \begin{align*}
    R_{\text{warm-up}}(t_w)
    \le \mathcal{O}\big(
        \frac{\kappa_0^6\, \alpha_1\, \bar{\vartheta}_{B_*}^8}{\gamma_0^2}
        (n+m)^5 \|P_*\|^{14}
    \big),
\end{align*}
The regret incurred during the warm-up phase of the ARSLO$^{+}(\bar{\rho})$ algorithm satisfies
 \begin{align*}
    R_{\text{warm-up}}(t_w, \bar{\rho})
    \le \mathcal{O}\big(
        \frac{\kappa_0^6\, \alpha_1\, \bar{\vartheta}_{B_*}^8}{\gamma_0^2}
        (n+m)^6 \|P_*\|^{h}
    \big)
\end{align*}
where $h$ is defined in (\ref{eq:Defhforre}).
\end{theorem}

 \section{Discussion and Conclusion}\label{philosophy}

In most existing algorithms, due to their specific regret decomposition, the state norm does not directly affect the dominant terms of the regret and usually appears only in $\mathcal{O}(\log T)$ terms. As a result, these works typically do not aim to derive tight high-probability bounds on the state norm trajectory. Their analyses show that stability in the Lyapunov sense is sufficient to control the regret bound, and they often provide only crude bounds on the state norm by arguing that as the length of epochs increases, the effect of policy updates diminishes in the long run. In contrast, we found that carefully studying the boundedness of the state norm is necessary, as it provides strong control over the state trajectory and makes the algorithm, with its anytime regret guarantee, transferable to a wider range of control settings. In our work, we observed that while imposing a strong sequentiality notion can help to explicitly talk about the state norm bound, the resulting regret bound might not be optimal. This is because strong sequentiality requires additional constraints on the algorithm to keep subsequently generated policies close to each other. We proposed an alternative idea: let the algorithm generate only stabilizing policies, and handle concerns about state explosion due to policy updates using a dwell-time--inspired approach. We noticed that if the focus is purely on minimizing regret, this extreme approach is also inefficient. The optimal regret bound is achieved when the algorithm ensures that two subsequently generated policies are \emph{neither too far apart nor too close}. This reveals a fundamental trade-off between state norm control and regret upper bounds. Our algorithm lies at the intersection of two broad classes of methods: self-tuning regulator--based algorithms, such as CE-based approaches \cite{mania2019certainty, simchowitz2020naive, jedra2022minimal}, and OFU-based algorithms, such as \cite{abeille2020efficient, cohen2019learning}. While our method is in the spirit of OFU, it also incorporates elements of input perturbation methods. This combination allowed us to overcome one of the remaining challenges of computationally efficient OFU-based algorithms: the requirement for an a priori bound on the optimal average expected cost, i.e., $\|P_*\| \le D$. Such a requirement is unnecessary for self-tuning regulator--based algorithms, and by bridging these two approaches, we were able to remove this assumption for OFU-based methods. Finally, although the techniques developed in this work are derived for LQ control with regret guarantees, they may be extended to more general learning-based nonlinear control schemes. In such settings, where policy updates occur over time, the resulting closed-loop dynamics can naturally be interpreted as switched systems. This perspective suggests a principled pathway for relating closed-loop stability, performance, and the exploration–exploitation trade-off through tools from switched-systems theory, such as dwell-time.

\acks{
 This work was supported in part by a C3.AI DTI Research Grant. 

}

\newpage


% \newpage

\startcontents[mainsections]
\printcontents[mainsections]{l}{1}{\section*{Content of Appendices}\setcounter{tocdepth}{2}}

% \section{A section}
% \lipsum

% \section{Another section}
% \lipsum

% \stopcontents[mainsections]
% \newpage
\appendix
% \startcontents[mainsections]
% \printcontents[mainsections]{l}{1}{\section*{Content of Appendices}\setcounter{tocdepth}{2}}
\newpage
\section{Discussion of Related Works} \label{secMoreDiscu}
One useful and practical categorization of regret guaranteed algorithms for LQ control, also adopted by \cite{jedra2022minimal}, is to group existing methods under two main umbrellas: (i) self-tuning regulators and (ii) OFU-based approaches.

Self-tuning regulator–based algorithms operate as follows. At each step, they
estimate the unknown matrices $(A_*,B_*)$, treat these estimates as the true
parameters, solve the DARE, and design the control accordingly. To ensure
sufficient excitation of the system—and thus enable learning of $(A_*,B_*)$—the
control inputs are typically perturbed with appropriately tuned noise. Many
existing works, including \cite{mania2019certainty}, \cite{faradonbeh2020input},
\cite{simchowitz2020naive}, and \cite{jedra2022minimal}, fall within this
category, and these methods are computationally efficient.

The second class consists of OFU-based algorithms, in which a high-probability
confidence set for $(A_*,B_*)$ is constructed, and a control is designed by
playing optimistically with respect to this set. Compared to robust control
schemes, which follow a \textit{min--max} paradigm, these methods can be viewed
as \textit{min--min} algorithms. OFU-based approaches were initially and
predominantly developed in the context of bandits
\cite{auer2002finite,abbasi2011improved}. The first application of such ideas to
the LQ control setting appears in \cite{campi1998adaptive}, which provides only
asymptotic guarantees. Along this line of research, \cite{abbasi2011online}
proposed the first regret-guaranteed algorithm for LQ control.
Later works aimed to relax certain assumptions or burdens of this algorithm. For
example, \cite{faradonbeh2017finite} extends the approach to stabilizable systems
and relaxes the Gaussian process noise assumption to a class of heavy-tailed noise
with arbitrary correlation structures. \cite{lale2022reinforcement} reduces the
dimensional dependence of the regret upper bound from exponential to polynomial.
However, neither of these works addresses the most fundamental drawback of
\cite{abbasi2011online}, namely its computational inefficiency.
 In response to this, \cite{cohen2019learning} and \cite{abeille2020efficient} proposed
computationally efficient algorithms. The former is based on a relaxation of the
SDP-based formulation of LQ control using the constructed confidence ellipsoid,
while the latter relies on extended value iteration (EVI) followed by a Lagrangian
relaxation. To the best of our knowledge, these two works constitute the only
computationally efficient OFU-based algorithms for LQ control.

Our work sits precisely at their intersection of these categories. On one hand, it embodies the spirit of OFU-based algorithms, building a confidence ellipsoid and playing optimistically with respect to this set to design the controller. On the other hand, our proposed algorithm applies carefully tuned input perturbation noise, which enables guaranteed regret bounds and stability in real time.

We are primarily interested in conducting the comparison within the category of computationally efficient algorithms. Our contributions can be summarized under three main points:
(i) establishing an anytime regret guarantee; (ii) providing a high-probability bound on the state trajectory and discussing the trade-offs that arise when relaxing strong sequential stability; and (iii) removing the need for any a priori bound on the optimal average expected cost $J_*$, or equivalently, on $\|P_*\|$.

One important aspect of our algorithm is its ability to provide an anytime regret guarantee. Most existing methods are designed for a fixed horizon, and their extension to the infinite-horizon setting typically relies on the doubling trick. However, this approach does not fully resolve the issue, since each epoch still operates with a fixed horizon. As discussed earlier, the foundational line of regret-guaranteed algorithms begins with\cite{abbasi2011regret}, which offers an anytime guarantee but suffers from high computational complexity and an exponential dependence on the system dimension. Although more recent methods have addressed the computational challenges of that approach, they remain horizon-dependent, whereas \cite{abbasi2011regret} inherently provides an anytime guarantee. In contrast, our method is both computationally efficient and anytime. Among existing Self-tuning regulator–based works, the only method that has anytime-style guarantee is the one by \cite{jedra2022minimal}, but its guarantees apply only to the expected regret. In the line of OFU-based algorithms, to the best of our knowledge, the only two computationally efficient existing methods are those of \cite{cohen2019learning} and \cite{abeille2020efficient}; however, both are horizon-dependent. This makes our approach the first computationally efficient OFU-style algorithm that also provides an anytime regret guarantee.

The next important guarantee concerns whether an algorithm can provide a fully system-theoretic high-probability upper bound on the state norm trajectory. While overall stability is ensured by all the existing algorithms which is appeared to be enough to guarantee the promised regret bound, however, an explicit high probability bound on state norm is still need attention. In CE-based methods \cite{jedra2022minimal} and \cite{simchowitz2020naive} the overall stability is established using a common Lyapunov function. However, the former work does not provide an explicit high probability upper-bound on state norm, while the latter introduces a crude bound which is not precise. We refer the readers to Lemma 5.3 of \cite{simchowitz2020naive} and its proof, where the crude bound is obtained by apply $\max_{k \ge 1} \; k(1-\rho)^k \;\lesssim\; \frac{1}{\rho}$ for some $\rho<1$ but we know that $\lesssim$ absorbs a constant that might be dependent on $\rho$. This crude bound does not affect the regret bound as it appears in $\mathcal{O}(\log T)$ terms. Even when an algorithm generates stabilizing policies, updating from one policy to another can cause temporary spikes, since the next generated policy is not necessarily better than the current one unless the algorithm is explicitly designed to prevent this. Such an enforcement has been proposed by \cite{faradonbeh2017finite} and \cite{abeille2020efficient}. In these approaches, an initial system identification phase is used to obtain a sufficiently accurate estimate, followed by the fulfillment of explicit accuracy conditions to guarantee sequential stability. While the former approach does not provide an explicit state norm bound, the latter provides a state norm bound that relies on an a priori bound on $J_*$, making it not fully system-theoretic and therefore less reliable. Similarly, \cite{cohen2019learning} introduced the notion of strong sequential stability. In this framework, the algorithm is designed to generate stabilizing policies while ensuring that any two consecutively generated policies remain sufficiently close, thereby preventing spikes in the state norm during policy updates. This technique requires an appropriate initialization and the fulfillment of explicit conditions during the implementation of the main algorithm. Overall, the OFU-based algorithms in \cite{faradonbeh2017finite}, \cite{abeille2020efficient}, and \cite{cohen2019learning} share a common strategy: obtaining a sufficiently accurate initial estimate, followed by the enforcement of explicit conditions thereafter. While we provide an anytime and fully system-theoretic version of the guarantee of \cite{cohen2019learning}, we also show that this strategy is not efficient in terms of the regret bound. By incorporating dwell-time techniques, we demonstrate how to improve the regret bound while maintaining explicit control over the state trajectory.

We also compare algorithms in terms of the a priori information or bounds required for their implementation. In particular, we focus on requirements such as an upper bound on $\|P_*\|$ and the horizon, which are explicitly indicated in a separate column of the table specifying whether the algorithm is anytime or not. Notably, in our proposed algorithm, the parameter $\delta$ is fully independent of the implementation time and can be set to any arbitrary value, making it effectively non-restrictive. Almost all computationally efficient algorithms require an initial stabilizing policy $K_0$. In this work, we eliminate the need for prior access to the optimal average cost $J_*$, a requirement commonly present in existing computationally efficient OFU-based algorithms. While this issue has been addressed in certainty-equivalence methods, as studied in \cite{simchowitz2020naive, jedra2022minimal}, it remains unresolved in current computationally efficient OFU-based algorithms, including \cite{abeille2020efficient, cohen2019learning}. Such a requirement makes the performance of these algorithms dependent on an external bound that may be loose, which is reflected in their regret upper bounds. We address this limitation by introducing a warm-up phase that iteratively solves the dual SDP problem to detect when the parameter estimate enters the desired ball, whose radius depends on $\|P_*\|$.

Table \ref{tbl:comp} compares our algorithm with several notable algorithms from the literature. We specifically selected these works because they either provide an explicit regret upper bound or their estimated regret upper bound has been reported by other studies.

\begin{table}[t]
\centering
\scriptsize
\setlength{\tabcolsep}{3pt}

\begin{threeparttable}
\caption{Comparison of existing methods.}
\label{tbl:comp}

\begin{tabular}{l l c c c c}
\toprule
Paper & Method & Regret upper-bound w.p. $1-\delta$ & Any-time & State norm w.p.
& \begin{tabular}{c}Required\\ information\end{tabular} \\
\midrule

\cite{simchowitz2020naive}$^\star$ & CE 
  & $\widetilde{\mathcal{O}}\!\left(\sqrt{nm^{2}\|P_{*}\|^{11}T\log(1/\delta)}\right)$ 
  & $\times$ & $\times$ & $K_0$ \\

\cite{jedra2022minimal}$^\diamondsuit$ & CE  
  & $\widetilde{\mathcal{O}}\!\left(\sqrt{nm^{2}\|P_{*}\|^{10.5}T\log^2 T}\right)$ 
  & $\checkmark$ & $\times$ & $K_0$ \\

\cite{cohen2019learning} & OFU  
  & $\widetilde{\mathcal{O}}\!\left(\sqrt{n^{4}(n+m)^{3}D^{10}T\log^{4}(T/\delta)}\right)$ 
  & $\times$ & $\checkmark$ & $K_0, Tr(P_*)\leq D$ \\

\cite{abeille2020efficient}$^\star$ & OFU 
  & $\widetilde{\mathcal{O}}\!\left(\sqrt{n(n+m)^2D^{3}\|P_{*}\|^{4}T\log^{3}(T/\delta)}\right)$ 
  & $\times$ & $\times$ & $K_0, Tr(P_*)\leq D$ \\

ARSLO & OFU  
  & $\widetilde{\mathcal{O}}\!\left(\sqrt{n^3(n+m)^2 \|P_*\|^{16}T\log^4(T/\delta)}\right)$ 
  & $\checkmark$ & $\checkmark$ & $K_0$ \\

ARSLO$^{+}(\bar{\rho}_{*})$ & OFU  
  & $\widetilde{\mathcal{O}}\!\left(\sqrt{n^3(n+m)^2 \|P_*\|^{12}T\log^4(T/\delta)}\right)$ 
  & $\checkmark$ & $\checkmark$ & $K_0$ \\
\bottomrule
\end{tabular}

\begin{tablenotes}
\footnotesize
\item[$^\diamondsuit$] This regret bound holds in expectation.
\item[$^\star$]These works provide only a crude high-probability bound on the state norm.
\end{tablenotes}

\end{threeparttable}
\end{table}

\section{Preliminaries}\label{sec:suppmat}

% Note: in this sample, the section number is hard-coded in. Following
% proper LaTeX conventions, it should properly be coded as a reference:

%In this appendix we prove the following theorem from
%Section~\ref{sec:textree-generalization}:

% \subsection{Doubling technique}
%  \label{sec:doubling}
% The doubling trick is an adaptation technique used to maintain the same level of regret as when the algorithm is aware of the time horizon. Here is how it operates: Suppose an algorithm achieves regret of $\mathcal{O}(t^{\alpha})$ for some $0< \alpha\leq 1$ when it knows the horizon $t$. The doubling technique dictates that if we consecutively run the algorithm for durations $T_i=C_0 2^i$ with $i=0,1,...$, then the algorithm attains $\mathcal{O}(T^{\alpha})$ regret, where $T=C_0\sum_{i=0}^{n} 2^i$ and $n$ is unknown, meaning the algorithm remains oblivious to $T$. An example of such non-oblivious strategies is the OSLO algorithm, which is equipped with a warm-up algorithm proposed by \cite{cohen2019learning}, achieving $\mathcal{O}(\sqrt{T})$ regret when the horizon $T$ is predetermined. While applying this trick can conserve the order of regret, in terms of the constants, this technique might not be efficient. Specifically, the upper bound of one of the decomposed terms of regret for the OSLO algorithm is logarithmic, and it is shown that the doubling technique cannot conserve logarithmic terms (see \cite{besson2018doubling}).

\subsection{Primal and Dual SDP Formulation of LQR}
 \label{sec:primaldualSDP}

Consider the standard LQR problem of (\ref{eq:dyn_atttt}) and (\ref{eq:costrep}). Assuming controllability of $(A_*, B_*)$, it is known that the optimal sequence minimizing average expected cost (\ref{eq:avExpC}) is $u^*=\{u^*_t\}_{t=1}^{T-1}$, where $u^*_t=K_*(\Theta*)x_t$ and
\begin{align}
K_*(\Theta_*)= -(B_*^\top P(\Theta_*)B_*+R)^{-1}B_*^\top P(\Theta_*)A_*. \label{OptFeed}
\end{align}
In (\ref{OptFeed}), $P(\Theta_*)$ denotes the unique solution for the Discrete Algebraic Riccati Equation (DARE), computed by
\begin{align}
P(\Theta_*) = Q + A_*^\top P(\Theta_*)A_* - A_*^\top P(\Theta_*)B_*(B_*^\top P(\Theta_*)B_*+R)^{-1}B_*^\top P(\Theta_*)A_*. \label{eq:DARE}
\end{align}
It is known that the optimal average expected cost is $J_{*} = P(\Theta_*) \bullet W$.

This problem can be transformed into the Semi-Definite Programming (SDP) formulation in the following manner.

Observing that a steady-state joint state and input distribution $(x_{\infty},u_{\infty})$ exists for any stabilizing policy $\pi$, we denote the covariance matrix of this joint distribution as $\mathcal{E}(\pi)$, which is defined as follows:
\begin{align*}
\mathcal{E}(\pi)= \mathbb{E}\begin{pmatrix}
x_{\infty}x_{\infty}^\top & x_{\infty}u_{\infty}^\top \\
u_{\infty}x_{\infty}^\top & u_{\infty}u_{\infty}^\top
\end{pmatrix}
\end{align*}
where $u_{\infty}=\pi(x_{\infty})$ and the average expected cost of the policy is given as follows

\begin{align*}
J(\pi)= \begin{pmatrix}
Q & 0 \\
0 & R
\end{pmatrix}\bullet \mathcal{E}(\pi).
\end{align*}

Now, for a linear stabilizing policy $K$ that maps $u_{\infty}=Kx_{\infty}$, the covariance matrix of the joint distribution takes the following form:

\begin{align}
&\mathcal{E}(K)=\begin{pmatrix}
X & XK^\top \\
KX & KXK^\top
\end{pmatrix}. \label{eq:good123er}
\end{align}
where $X=\mathbb{E}[x_{\infty}x_{\infty}^\top]$. Then the average expected cost is computed as follows 
\begin{align*}
J(K)= \begin{pmatrix}
Q & 0 \\
0 & R
\end{pmatrix}\bullet \mathcal{E}(K)=(Q+K^\top R K)\bullet X.
\end{align*}

Now given $\Theta_*=(A_*\; B_*)^\top$, the LQR control problem can be formulated in a SDP form, as follows:
\begin{align}
\nonumber\textrm{minimize}_{\Sigma}\; \; \; \begin{pmatrix}
Q & 0 \\
0 & R
\end{pmatrix}\bullet \Sigma\\
\nonumber\textrm{S.t.}\; \; \; \Sigma_{xx}={\Theta}_*^\top\Sigma\Theta_*+W\\
\Sigma\succ 0\label{eq:SDPKhali}
\end{align}
where any feasible solution $\Sigma$ can be represented as follows:
\begin{align}
  \Sigma=\begin{pmatrix}
\Sigma_{xx} & \Sigma_{xu} \\
\Sigma_{ux} & \Sigma_{uu}
\end{pmatrix}  
\end{align}
in which $\Sigma_{xx} \in  \mathbb{R}^{n \times n}$, $\Sigma_{uu} \in  \mathbb{R}^{m \times m}$, and $\Sigma_{ux}=\Sigma_{xu} \in  \mathbb{R}^{m \times n}$. The optimal value of (\ref{eq:SDPKhali}) coincides with the average expected cost $J_*$, and for $W\succ 0$ (ensuring $\Sigma_{xx}\succ 0$), the optimal policy of the system, denoted as $K_*$ and associated with the optimum solution $\Sigma^*$, can be derived from $K_*=\Sigma^*_{{ux}}\Sigma^{*^{-1}}_{{xx}}$ considering (\ref{eq:good123er}).

The dual SDP problem is written as follows
\begin{align}
\begin{array}{rrclcl}
\displaystyle \max & \multicolumn{1}{l}{P\bullet W}\\
\textrm{s.t.} & \begin{pmatrix}
Q-P & 0 \\
0 & R
\end{pmatrix}+\Theta_* P {\Theta^\top_*}=0\\
&P\succeq 0 
\end{array} \label{eq:equal_dual}
\end{align}
Let $P_*$ denote the optimal solution of (\ref{eq:equal_dual}) and let the optimal input be $u_t^*=K_* x_t$. Then define
\begin{align*}
    z_t=\begin{pmatrix}
I \\
K_*
\end{pmatrix} x_t.
\end{align*}
Multiplying both sides of the optimal equality  (\ref{eq:equal_dual}) by $z_t^\top$ and $z_t$ yields
\begin{align}
   P_*=Q+K_*^\top R K_*+(A_*+B_*K_*)^\top P_* (A_*+B_*K_*) \label{eq:promRic}
\end{align}
which is the prominent optimal Riccati equation. Therefore, the dual solution $P_*$ is indeed the solution of the optimal Riccati equation.

\subsection{$(\kappa, \gamma)-$Stability}

\label{sec:kapgamdef}

To gain a deeper understanding of $(\kappa, \gamma)$-stability, we present illustrative proofs from Lemma B.1 of \cite{cohen2018online}. The claim to be shown is that for any linear time-invariant system defined by the pair $(A_*, B_*)$, a stabilizing policy $K$ ensures $(\kappa, \gamma)$-strong stability in the sense of Definition \ref{def:sequentially}. The closed-loop system achieves stability if $\rho(A_*+B_*K)= 1-\gamma$, for some $0<\gamma<1$. Now defining matrix $Q=(1-\gamma)^{-1}(A_*+B_*K)$, the stability of $Q$ is guaranteed if a positive definite matrix $P$ exists such that
\begin{align}
Q^\top P Q\preceq P. \label{komak}
\end{align}

By definition, when $Q$ is stable, then $(A_*+B_*K)$ is also stable. Inequality (\ref{komak}) can be further rewritten as:
\begin{align}
(A_*+B_*K)^\top P(A_*+B_*K)\preceq (1-\gamma)^2 P. \label{eq:komak2}
\end{align}

Pre and post multiplying both sides of (\ref{eq:komak2}) by $P^{-\frac{1}{2}}$ and $P^{\frac{1}{2}}$ yields 
\begin{align}
L^\top L\preceq (1-\gamma)^2 . \label{eq:komak3}
\end{align}
where $L=P^{\frac{1}{2}}(A_*+B_*K)P^{-\frac{1}{2}}$. Then defining $H=P^{-\frac{1}{2}}$, one can write
\begin{align*}
HLH^{-1}=A+BK.
\end{align*}

Setting the condition number of $P^{-\frac{1}{2}}$ to $\kappa$, i.e., $\|H\|\|H^{-1}\|\leq \kappa$ and having $\|L\|\leq 1-\gamma$ from  (\ref{eq:komak3}) complete proof. 

% \subsection{Proof of Proposition \ref{prop1}}

% \begin{proof}
% It is straightforward that $J_*\leq J(K_0)$. The average expected cost of a $(\kappa_0, \gamma_0)$- strongly stabilizing policy is $J(K_0)=\sigma_{\omega}^2\operatorname{Tr}(P_{K_0})$ (see \cite{bertsekas2012dynamic}) where
% \begin{align*}
%     P_{K_0}=Q+K_0^\top R K_0+(A_*+B_*K_0)^\top P_{K_0}(A_*+B_*K_0).
% \end{align*}

% Now, by applying Lemma \ref{lem:komaki}, one can write:
% \begin{align*}
%     P_{K_0} \preceq \frac{\kappa_0^2}{\gamma_0} \|Q+K_0^\top R K_0\| I\preceq \frac{\kappa_0^2}{\gamma_0}\alpha_1 (1+\kappa_0^2) I
% \end{align*}
% using which, it is shown that $J(K_0)\leq \sigma_{\omega}^2 \textcolor{blue}{n}\frac{\kappa_0^2}{\gamma_0}\alpha_1 (1+\kappa_0^2)$. Therefore, $\nu=\sigma_{\omega}^2 n\frac{\kappa_0^2}{\gamma_0}\alpha_1 (1+\kappa_0^2)$ is an upper-bound for the optimal average expected cost.
% \end{proof}

\section {Helpful lemmas}

\begin{lemma} (Lemma 25 of \cite{cohen2019learning}) \label{lem:komaki}
    Let $X$ and $Z$ denote matrices of equal size and let $Y$ denote a $(\kappa, \gamma)$ stable matrix such that $X\preceq Y^\top X Y+Z$, then $X\preceq \frac{\kappa^2}{\gamma}\|Z\| I$.
\end{lemma}

The following lemma, adapted from \cite{abbasi2011online} gives a self-normalized bound for scalar-valued martingales. 
 
\begin{lemma} \label{Self_normalized_Bound_1}
	 Let $F_k$ be a filtration, $z_k$ be a stochastic process adapted to $F_k$ and $\omega^i_k$ (where $\omega^i_k$ is the $i-$th element of noise vector $\omega_k$) be a real-valued martingale difference, again adapted to filtration $F_k$ which satisfies the conditionally sub-Gaussianity assumption (Assumption \ref{Assumption 1}) with known constant $\sigma_{\omega}$. 
	 Consider the martingale and co-variance matrices:
	\begin{align*}
S^i_t:=\sum _{k=1}^{t} z_{k-1}\omega^i_k, \quad V_{t}=\lambda I+\sum _{s=1}^{t-1}z_{t} z_{t}^\top
	\end{align*}
	
then with probability of at least $1-\delta$, $0<\delta<1$ we have,
\begin{align*}
\left\lVert S^i_t\right\rVert^2_{V_{t}^{-1}} \leq 2 \sigma_{\omega}^2\log \bigg(\frac{\det(V_t) }{\delta \det(\lambda I)}\bigg)
\end{align*}	
\end{lemma}

\begin{proof} 
The proof follows from Theorem 3 and Corollary 1 in \cite{abbasi2011regret} which gives.
\begin{align*}
\left\lVert S^i_t\right\rVert^2_{V_{t}^{-1}} \leq 2 \sigma_{\omega}^2 \log\bigg(\frac{\det(V_t)^{1/2} \det(\lambda I)^{-1/2}}{\delta}\bigg)
\end{align*}	
Then, we have 
\begin{align*}
\log\bigg(\frac{\det(V_t)}{ \det(\lambda I)}\bigg)\geq \log\bigg(\frac{\det(V_t)^{1/2}}{ \det(\lambda I)^{1/2}}\bigg),
\end{align*}	
true which completes proof.
\end{proof}

The following corollary generalizes the Lemma \ref{Self_normalized_Bound} for a vector-valued martingale which will be later used in Theorem \ref{thm:Conficence_SetCons} to derive the confidence ellipsoid of estimates.

\begin{corollary} \label{Self_normalized_Bound}
	Under the same assumptions as in Lemma \ref{Self_normalized_Bound} and defining
\begin{align*}
S_t=Z^\top_tW_t=\sum _{k=1}^{t} z_{k-1}\omega^\top_k
\end{align*}
with probability at least $1-\delta$,
\begin{align}
\operatorname{Tr} (S^\top_tV^{-1}_tS_t)&\leq 2\sigma_{\omega}^2 n\log \bigg(\frac{n \det(V_t) }{\delta\det(\lambda I)}\bigg)\label{eq:mart_vec}
\end{align}
\end{corollary}

\begin{proof}
Applying Lemma \ref{Self_normalized_Bound} for $i=1,...,n$ yields
\begin{align*}
\left\lVert S^i_t\right\rVert^2_{V_{t}^{-1}} \leq 2 \sigma_{\omega}^2 \log \bigg(\frac{n\det(V_t) }{\delta \det(\lambda I)}\bigg)
\end{align*}
with probability at least $1-\delta/n$. Furthermore, we have
\begin{align*}
\operatorname{Tr} (S^\top_tV^{-1}_tS_t)=\sum_{i=1}^{n}{S^i_t}^\top V^{-1}_tS^{i}_t
\end{align*}
applying union bound, (\ref{eq:mart_vec}) holds with probability at least $1-\delta$.
\end{proof}

\begin{lemma}(\cite{cohen2019learning}) \label{lem:login}
   For some $M\succ 0$ and vector $z$ the following statement holds true.
    \begin{align*}
        \min \{z^\top M^{-1}z , 1\}\leq  2\log \frac{\det(zz^\top +M)}{\det(M)}
    \end{align*}
\end{lemma}

% \begin{lemma}(Hanson-Wright inequality \cite{cohen2019learning}) \label{lem:HansonWright}
%    Let $x\sim \mathcal{N}(0, \Sigma)$, where $\Sigma\in \mathbb{R}^{n\times n}$, be a Gaussian random vector. Furthermore, let $A\in \mathbb{R}^{m\times n}$. Then, for all $z\geq 1$,
% \begin{align}
% \operatorname{Prob}\left(\|Ax\|^2-\mathbb{E}[\|Ax\|^2]> 4z \|\Sigma\|\|A\|\|A\|_*\right)<e^{-z} \label{eq:hans}
% \end{align}
% Moreover, in equation (\ref{eq:hans}), by the property of a Gaussian random vector, we have $\mathbb{E}[\|Ax\|^2]=\operatorname{Tr}(A \Sigma A)$.
% \end{lemma}

\begin{lemma} \label{lem:Usefulalgebricin}
Let $a \geq 1$ and $b\geq e$. Then the following inequality
\begin{align*}
    \frac{a \log(b x)}{\sqrt{x + c}} \leq 1
\end{align*}
holds for any $x \geq 1$, provided that
\begin{align*}
c= a^2 \Big( 2\log(2 a^2b) + 1 \Big)^2
\end{align*}
\end{lemma}

\begin{proof}
Equivalently we look to specify $c$ such that
\begin{align*}
   a^2 \log^2 bx \leq x + c, \quad \forall x \geq 1. 
\end{align*}
Let $y=\log bx$, and define the function
\begin{align*}
    g(y):= a^2 y^2 - \frac{e^y}{b}, \quad y \geq \log b.
\end{align*}
Then the problem reduces to finding a $c$ such that $g(y)\leq c$ for all $y\geq \log b$.

The critical points of $g(y)$ must satisfy $ g^{\prime}(y)=0$ which equivalently is
\begin{align*}
 \frac{e^y}{2 a^2 b} = y.
\end{align*}

Now suppose $y_*$ be any solution of this equation in $[\log \beta, \; \infty)$. If there was no such $y_*$ then the natural choice for $c$ is $\log b$.
Now we aim to upper-bound $y_*$ if it exists. We let

\begin{align*}
    D:=\log (2a^2b)
\end{align*}
and we define the candidate $Y$ such that $y_*\leq Y$ as follows
\begin{align*}
    Y:=D+\log D+1
\end{align*}

Now we define the following function which is continuous for any $y>0$
\begin{align*}
    H(y) := \frac{e^y}{2 a^2 b y}.
\end{align*}
Then we have
\begin{align*}
    H^{\prime}(y) = \frac{e^y(y-1)}{2 a^2 by^2} \geq 0
\end{align*}
which shows that it is non-decreasing for any $y\in [1,\;\infty)$. Now lets us evaluate $H(y)$ in $y=Y$.
\begin{align}
    H(Y) = \frac{e^Y}{2 a^2 \beta Y} = \frac{e \; e^D D}{2 a^2 b (D + \log D + 1)} = \frac{e D}{D + \log D + 1}\label{eq:usepul}
\end{align}
Now notice that 
\begin{align*}
    \frac{(e-1)D}{\log D +1}\geq 1, \quad \textit{for}\; D\geq \log 2
\end{align*}
which results in 
\begin{align*}
    \frac{eD}{\log D +D+1}=\frac{(e-1)D+D}{\log D +1+D}\geq \frac{(e-1)D}{\log D +1}\geq 1.
\end{align*}
Applying this result in (\ref{eq:usepul}) means $H(Y)\geq 1$. 

Now, notice that since $H(y_*) = 1$ and $H$ is nondecreasing, it follows that
\begin{align*}
  y_* \leq  Y = D + \log D + 1.  
\end{align*}
Now consider the function $g(y)$ and the fact that it attains its maximum in $y_*$. Then for any admissible $y$ we can write
\begin{align*}
    g(y)\leq g(y_*)=a^2y_*^2-\frac{e^{y_*}}{b}\leq a^2y_*^2\leq a^2Y^2=a^2(D+\log D+1)^2\leq a^2 (2D+1)^2.
\end{align*}
Recalling that a $c$ such that $g(y)\leq c$ indeed guarantees $a^2 \log^2 bx\leq x+c$ for all $x\geq 1$, then we choose such $c$ as follows
\begin{align*}
c= a^2 \Big( 2\log(2 a^2b) + 1 \Big)^2. 
\end{align*}
\end{proof}

\begin{lemma}(Azuma, 1967 — Theorem 22 in \cite{cohen2019learning})\label{lem:Azuma}
Let \(X_1, \dots, X_N\) be a martingale‑difference sequence (with respect to some filtration) such that
\[
  |X_i| \;\le\; c
\quad
\text{for all }i = 1,\dots,N.
\]
Then for all \(t > 0\),
\[
  \Pr\!\Big(\sum_{i=1}^N X_i > t\Big)
  \;\le\; \exp\!\Big(-\frac{t^2}{2 N c^2}\Big).
\]
\end{lemma}

\begin{lemma}(Hanson--Wright inequality, Theorem 21 in \cite{cohen2019learning}) \label{lem:HansonWright}
Let $x \sim \mathcal{N}(0, I_n)$ be a standard Gaussian random vector, and let $A \in \mathbb{R}^{m \times n}$. Then for all $z > 0$,
\begin{align*}
    \Pr\Big( \|Ax\| - \|A\|_F > 2 \|A\| \, \|A\|_F \sqrt{z} + 2 \|A\|^2 z \Big) < e^{-z}.
\end{align*}

In particular, if $x \sim \mathcal{N}(0, \Sigma)$ for some covariance matrix $\Sigma$, then $\mathbb{E}\|Ax\|^2 = \operatorname{Tr}(A \Sigma A^\top)$, and for any $z \ge 1$,
\begin{align*}
    \Pr\Big( \|Ax\|^2 - \mathbb{E}\|Ax\|^2 > 4 z \|\Sigma\| \, \|A\| \, \|A\|_F \Big) < e^{-z}.
\end{align*}
\end{lemma}

% \begin{proof}
%  Let $x \in \mathbb{R}^d$. Consider the quadratic form $x^\top S x$.  
% Rewrite using $V^{1/2}$:  
% \[
% x^\top S x = x^\top V^{-1/2} (V^{1/2} S V^{1/2}) V^{-1/2} x = y^\top (V^{1/2} S V^{1/2}) y,
% \]  
% where $y = V^{-1/2} x$.  

% Bound the quadratic form using the spectral norm:  
% \[
% y^\top (V^{1/2} S V^{1/2}) y \le \|V^{1/2} S V^{1/2}\| \, y^\top y = \|V^{1/2} S V^{1/2}\|_2 \, x^\top V^{-1} x.
% \]  

% Since this holds for all $x$, we obtain the PSD inequality:  
% \[
% S \preceq \|V^{1/2} S V^{1/2}\| \, V^{-1}.
% \].
% \end{proof}

\section{Confidence Ellipsoid Construction} \label{ap:confdcos}

Consider the linear system (\ref{eq:dynam_by_theta}), that evolves up to time $t$. Then, we can express the dynamics as follows:
\begin{align*}
X_{t}&=Z_{t} \Theta_{*}+W_{t} 
\end{align*}
where $W$ is the vertical concatenation of $\omega_{1}^\top,...,\omega_{t}^\top$ and $X_{t}$ and $Z_{t}$ are matrices constructed by rows $x^\top_{1}, ...,x^\top_{t}$ and ${z}^\top_{1}, ...,{z}^\top_{t-1}$ respectively. 

Given an initial estimate $\Theta_0$ of $\Theta_*$ we define $e(\Theta)$ as follows:
\begin{align*}
e(\Theta)&=\lambda \operatorname{Tr} \big((\Theta-\Theta_0)^\top(\Theta-\Theta_0)\big)+\sum _{s=0}^{t-1} \operatorname{Tr} \big((x_{s+1}-{\Theta}^\top z_{s})(x_{s+1}-{\Theta}^\top z_{s})^\top\big) 
 \end{align*}
which has a regularization term $\operatorname{Tr} \big((\Theta-\Theta_0)^\top(\Theta-\Theta_0)\big)$ that penalizes deviation from the estimate $\Theta_0$. Then $l_2-$regularized least square estimation is obtained as follows:
 \begin{align}
 \hat{\Theta}_{t} =\operatorname*{argmin}_{\Theta} e(\Theta)={V}_{t}^{-1}\big(Z_{t}^\top X_{t}+ \lambda \Theta_0\big) \label{eq:LSEp} 
 \end{align}
 where 
  \begin{align*}
V_{t}=\lambda I + \sum_{s=0}^{t-1} z_{s}z_{s}^\top=\lambda I +Z_{t}^\top Z_{t},
 \end{align*}
 
In a different scenario when there is no such an initial estimate $\Theta_0$, the regularization is done with respect to $\operatorname{Tr} \big(\Theta^\top\Theta\big)$ that results in least square estimation 
 \begin{align}
 \hat{\Theta}_{t} ={V}_{t}^{-1} Z_{t}^\top X_{t}. \label{eq:LSEp2} 
 \end{align}

The following theorem provides an upper-bound on estimation error for both of the mentioned scenarios.

\begin{theorem} (\cite{abbasi2011regret},\cite{cohen2019learning})\label{thm:Conficence_SetCons} 
The least square estimation error $(\hat{\Theta}_t-\Theta_*)$ satisfies
 \begin{align}
\operatorname{Tr}\big((\hat{\Theta}_t-\Theta_*)^\top V_{t}(\hat{\Theta}_t-\Theta_*)\big) \leq r_{t}\label{eq:confSet1_tighterghfff}
\end{align}
where 
\begin{enumerate}
    \item when there is an initial estimate $\Theta_0$ such that $\|\Theta_0-\Theta_*\|_*\leq \epsilon$
\begin{align}
    r_{t}=\bigg( \sigma_{\omega} \sqrt{2n \log\frac{n\det(V_{t}) }{\delta \det(\lambda I)}}+\sqrt{\lambda}  \epsilon \bigg)^2 \label{radius_centralEl_realTime}
\end{align}
and $\hat{\Theta}_t$ is computed by (\ref{eq:LSEp}).
\item And when there is no such $\Theta_0$
\begin{align}
    r_{t}=\bigg( \sigma_{\omega} \sqrt{2n \log\frac{n\det(V_{t}) }{\delta \det(\lambda I)}}+\sqrt{\lambda}  \vartheta \bigg)^2. \label{radius_centralEl_realTime2}
\end{align}  
and $\hat{\Theta}_t$ is computed by (\ref{eq:LSEp2}).
\end{enumerate}

\end{theorem}

\section {Relaxed Primal and Dual SDPs}\label{eq:rlxpdfor}
When incorporating the confidence ellipsoid, the primal and dual SDPs (\ref{eq:SDPKhali} and \ref{eq:equal_dual}) can be relaxed to be used for control design and regret bound analysis. In order to incorporate the estimation error in SDP formulation, we first need the following perturbation lemma. This lemma, a generalized version of Lemma 24 of \cite{cohen2019learning}, can be used to incorporate un-normalized confidence ellipsoids into SDP formulation.

\begin{lemma} [Perturbation bound with trace-based confidence ellipsoid I] \label{lem:purturbation}

Let $X$, $\Delta$ denotes matrices of matching size, and let $P \succeq 0$, and $V \succ 0$ be symmetric matrices. Suppose the trace-based confidence ellipsoid satisfies
\begin{align*}
    \operatorname{Tr}(\Delta V \Delta^\top) \leq r
\end{align*}
Then we have the following symmetric positive semidefinite bound:
\[
-\mu \, \|P\| V^{-1} \preceq (X+\Delta)^\top P (X+\Delta) - X^\top P X \preceq \mu \, \|P\| V^{-1},
\]
where
\[
\mu := 2 \, \|V\|^{1/2}\,\|X\|\, \sqrt{r} + r.
\]
\end{lemma}

\begin{proof}
Expanding the quadratic difference gives
\begin{align*}
    (X+\Delta)^\top P (X+\Delta) - X^\top P X = X^\top P \Delta + \Delta^\top P X + \Delta^\top P \Delta.
\end{align*}
Define
\begin{align*}
    M := V^{1/2} (X^\top P \Delta + \Delta^\top P X) V^{1/2}
\end{align*}
which is a positive semi definite matrix.

Using the triangle inequality:
\begin{align*}
    \|M\| \le \|V^{1/2} X^\top P \Delta V^{1/2}\| + \|V^{1/2} \Delta^\top P X V^{1/2}\|.
\end{align*}
Since $\|A^\top\| = \|A\|$ for any matrix $A$, the two terms are equal as P is PSD, giving
\[
\|M\| \le 2 \, \|V^{1/2} X^\top P \Delta V^{1/2}\|.
\]

For any matrices $A, B, C$,
\begin{align*}
    \|A B C\| \le \|A\| \, \|B\| \, \|C\|_F.
\end{align*}
Set $A = V^{1/2} X^\top$, $B = P$, $C = \Delta V^{1/2}$. Then
\begin{align*}
    \|V^{1/2} X^\top P \Delta V^{1/2}\| \le \|V^{1/2} X^\top\| \, \|P\| \, \|\Delta V^{1/2}\|_F.
\end{align*}
Hence,
\begin{align*}
    \|M\| \le 2 \, \|V^{1/2} X^\top\| \, \|P\| \, \|\Delta V^{1/2}\|_F.
\end{align*}
By the definition of the Frobenius norm,
\begin{align*}
    \|\Delta V^{1/2}\|_F^2 = \operatorname{Tr}((\Delta V^{1/2}) (\Delta V^{1/2})^\top) = \operatorname{Tr}(\Delta V \Delta^\top) \le r.
\end{align*}
Thus,
\begin{align*}
    \|M\| \le 2 \, \|V^{1/2}\| \| X\| \, \|P\| \, \sqrt{r}\leq 2\|V\|^{1/2}\|X\|\, \|P\|\, \sqrt{r}.
\end{align*}
in which we applied the fact that $\|V^{1/2}\|=\|V\|^{1/2}$.

Since
\[
X^\top P\Delta + \Delta^\top P X 
= 
V^{-1/2} M V^{-1/2},
\]
and for any symmetric matrix $M$,
\[
V^{-1/2} M V^{-1/2} \preceq \|M\|\, V^{-1},
\]
we obtain
\begin{align}
(X^\top P\Delta+\Delta^\top P X)\preceq  2\|V\|^{1/2}\|X\|\, \|P\|\, \sqrt{r} V^{-1}.\label{eq:CrossTer}
\end{align}

We now upper-bound the second term $\Delta^\top P \Delta$. Define
\begin{align*}
    \bar{M}=(V^{1/2}\Delta^\top) \, P \, (\Delta V^{1/2}).
\end{align*}
It follows that
\begin{align*}
    \|\bar{M}\|\leq \|P\|\, \|\Delta V^{1/2}\|_F^2=\|P\|\, \operatorname{Tr}(\Delta V \Delta^\top)\leq \|P\| \, r.
\end{align*}
Therefore,
\begin{align}
\nonumber \Delta^\top P \Delta
&=
V^{-1/2}
\bar{M}
V^{-1/2}\preceq \|\bar{M}\| V^{-1} \\
&\preceq \|P\| \, r V^{-1}.\label{eq:delPdel}
\end{align}

Combining the bounds (\ref{eq:CrossTer}) and (\ref{eq:delPdel}) gives
\begin{align*}
    (X+\Delta)^\top P (X+\Delta) - X^\top P X \preceq \mu \, \|P\| V^{-1}.
\end{align*}
where $\mu := 2 \, \|V\|^{1/2}\,\|X\|\, \sqrt{r} + r$.

The other side of inequality can be easily shown by following a similar argument.
\end{proof}

We also need the following result for the purpose of showing Lemma \ref{lem:deriveRelaxedSDPs}.

\begin{lemma} [Perturbation bound with trace-based confidence ellipsoid II] \label{lem:perturb2}
Let  $X$ and $\Delta$ be matrices and and let $\Sigma \succeq 0$, and $V \succ 0$. Suppose 
\begin{align*}
    \operatorname{Tr}(\Delta V \Delta^\top) \leq r
\end{align*}  
for some $r>0$. Also, set $\mu := 2 \, \|V\|^{1/2}\,\|X\|\, \sqrt{r} + r$. Then we have
\begin{align*}
    \|(X+\Delta)\Sigma (X+\Delta)^\top -X\Sigma X^\top\|\leq
\mu\,\big(\Sigma\bullet V^{-1}\big).
\end{align*}
\end{lemma}
\begin{proof}
 Let
\begin{align*}
    A := (X+\Delta)\Sigma (X+\Delta)^\top - X\Sigma X^\top.
\end{align*}
Since $\Sigma \succeq 0$, the matrix $A$ is symmetric. Therefore,
\begin{align*}
    \|A\|=\sup_{\|u\|=1} |u^\top A u|.
\end{align*}
Fix any unit vector $u$ and define 
\begin{align*}
    P := uu^\top.
\end{align*}
Then $P \succeq 0$ and $\|P\| =1$.

Using cyclicity of the trace,
\begin{align*}
u^\top A u
&=
u^\top\Big((X+\Delta)\Sigma (X+\Delta)^\top - X\Sigma X^\top\Big)u \\
&=
\operatorname{tr}\!\Big(
\Sigma\big((X+\Delta)^\top uu^\top (X+\Delta)
      - X^\top uu^\top X\big)
\Big) \\
&=
\Sigma \bullet
\Big(
(X+\Delta)^\top P (X+\Delta)
- X^\top P X
\Big).
\end{align*}

Applying Lemma~\ref{lem:purturbation} with $P=uu^\top$ yields
\begin{align*}
    -\mu V^{-1} \preceq (X+\Delta)^\top P (X+\Delta)- X^\top P X\preceq \mu V^{-1},
\end{align*}
since $\|P\| = 1$.

Since $A$ and $V^{-1}$ are symmetric matrices, taking the trace inner product with $\Sigma \succeq 0$ gives
\begin{align*}
    -\mu\,\big(\Sigma\bullet V^{-1}\big)\leq u^\top A u\leq
\mu\,\big(\Sigma\bullet V^{-1}\big).
\end{align*}
Hence,
\begin{align*}
    |u^\top A u| \leq \mu\,\big(\Sigma\bullet V^{-1}\big).
\end{align*}
Since this bound holds for every unit vector $u$,
\begin{align*}
    \|A\|=\sup_{\|u\|=1}|u^\top A u|\leq
\mu\,\big(\Sigma\bullet V^{-1}\big).
\end{align*}
Therefore,
\begin{align*}
    \|(X+\Delta)\Sigma (X+\Delta)^\top -X\Sigma X^\top\|\leq
\mu\,\big(\Sigma\bullet V^{-1}\big).
\end{align*}
\end{proof}

Now having the perturbation lemmas we can write the following lemma to relax the SDPs (\ref{eq:SDPKhali}) and (\ref{eq:equal_dual}).

\begin{lemma} \label{lem:deriveRelaxedSDPs}
Let $\Theta$ be an estimate of $\Theta_*$ such that $\operatorname{Tr}\big((\Theta-\Theta_*)^\top V (\Theta-\Theta_*)\big)\leq r$ for some $V\succ 0$ and $r>0$. Then the the primal and dual SDPs' equalities (\ref{eq:SDPKhali}) and (\ref{eq:equal_dual}) hold if the following inequalities hold respectively.

\begin{align}
 \Sigma_{xx}\succeq \Theta^\top\Sigma\Theta+W-\mu (\Sigma \bullet V^{-1})I   \label{eq:primal_ineq}
\end{align}
and 
\begin{align}
 \begin{pmatrix}
Q-P & 0 \\
0 & R
\end{pmatrix}+\Theta P {\Theta}^\top\succeq \mu \|P\| V^{-1} \label{eq:TSjp}
\end{align}
where $\mu \leq r+ 2 \vartheta \sqrt{r} \|V\|^{1/2}$. 
\end{lemma}

\begin{proof}
To show (\ref{eq:primal_ineq}), consider  (\ref{eq:SDPKhali}) and apply Lemma \ref{lem:perturb2}. This yields
\begin{align}
   \nonumber  \Sigma_{xx}&={\Theta_*}^\top\Sigma\Theta_*+W \\
   &\succeq 
  \Theta^\top \Sigma \Theta +W- \mu (\Sigma \bullet V^{-1})I. \label{eq:prop3}
\end{align}

The proof of (\ref{eq:TSjp}) follows a similar procedure.
\end{proof}

Conclusively, given the parameter estimate in the form of confidence ellipsoid (\ref{eq:confSet1_tighterghfff2}, \ref{radius_centralEl_realTime200}) by this construction the relaxed primal SDP is formulated as follows:
\begin{align}
\begin{array}{rrclcl}
\displaystyle  \operatorname{min} & \multicolumn{1}{l}{\begin{pmatrix}
	Q & 0 \\
	0 & R
	\end{pmatrix}\bullet \Sigma}\\
\textrm{s.t.} & \Sigma_{xx}\succeq {\hat{\Theta}_t}^\top \Sigma\hat{\Theta}_t+W-\mu_t\big(\Sigma\bullet {{V}^{-1}_{t}}\big)I,\\
& \Sigma\succ 0.  &
\end{array}\label{eq:RelaxedSDPapp}
\end{align}
where $\mu_t= r_t+\sqrt{r_t}\vartheta \|V_{t}\|^{1/2}$. 
Denoting optimal solution of program (\ref{eq:RelaxedSDPapp}) by $\Sigma(\mathcal{C}_t)$, the control signal extracted from solving relaxed primal SDP, $u=K(\mathcal{C}_t)x$ is deterministic and linear in state where 
\begin{align*}
K(\mathcal{C}_t)=\Sigma_{ux}(\mathcal{C}_t){\Sigma_{xx}^{-1}(\mathcal{C}_t)}. 
\end{align*} 
Furthermore, the relaxed dual SDP formulation is given as follows:
\begin{align}
\begin{array}{rrclcl}
\displaystyle \max & \multicolumn{1}{l}{P\bullet W}\\
\textrm{s.t.} & \begin{pmatrix}
Q-P & 0 \\
0 & R
\end{pmatrix}+\hat{\Theta}_t P {\hat{\Theta}_t}^\top \succeq \mu_t \|P\|{{V}^{-1}_{t}}\\
&P \succeq 0 
\end{array}.\label{eq:RedSDP_DUAL} 
\end{align}
where we denote its optimal solution by $P(\mathcal{C}_t)$.

\section{Lower-Bound on Minimum Eigenvalue of Covariance Matrix}\label{sec:MinEigVal}

In this section, we aim to provide a lower bound for the minimum eigenvalue of covariance matrix $V_t\succ 0$ during implementation of Algorithm \ref{Alg:ACOLC}, the so-called ARSLO. This bound is highly useful for stability guarantee and regret bound analysis. We build upon the methodology of Theorem 20 in \cite{cohen2019learning}, which is specifically designed for the warm-up phase with a fixed policy $K_0$ and input perturbation with a time-invariant covariance matrix. However, what we obtain is tailored to the ARSLO algorithm, which incorporates input perturbation with a time-varying Gaussian noise covariance and does not rely on a fixed policy.

We begin with the following lemma which is borrowed from \cite{cohen2019learning}.

\begin{lemma}\label{lem:lowboundV}

  Let the system~(\ref{eq:dyn_atttt}) be controlled by an input $u_t = K_t x_t + \eta_t$, where $K_t$ is a $(\kappa,\gamma)$-strong stabilizing policy generated by either ARSLO or ARSLO$^+(\bar{\rho})$ algorithms. Furthermore, let the input perturbation $\eta_t$ be
\begin{align*}
    \eta_t \sim \mathcal{N}\left(0, \frac{\sigma_{\omega}^2\bar{p}_t \left(K_t K_t^\top + \frac{\|P_t\|}{\alpha_0}I\right)}{\sqrt{t + \bar{c}}} \right),
\end{align*}
where $\bar{p}_t=\mathcal{O}(\log t)$ and $P_t$ (with some abuse of notation) is the solution of the relaxed dual SDP (\ref{eq:RedSDP_DUALP}). Then we have $\mathbb{E}[z_t z_t^\top] \succeq \frac{\sigma_{\omega}^2 \bar{p}_t}{2\sqrt{t}} I$ for all $t$'s such that 

\begin{align*}
    \frac{\sqrt{t+\bar{c}}}{ \bar{p}_t }-\frac{1}{2}\geq 0.
\end{align*}
\end{lemma}

\begin{proof}
  Given the dynamics and the assumption on the process noise $\omega_t$ it is straightforward to show that $\mathbb{E}[x_t x_t^\top|\mathcal{F}_{t-1}]\succeq \sigma_{\omega}^2 I$. Furthermore by definition $u_t=K_tx_t+\eta_t$ and considering the fact that input perturbation $\eta_t$ is independent than $\omega_{t+1}$ we can write:
  \begin{align*}
    \mathbb{E}[z_t z_t^\top|\mathcal{F}_{t-1}]= \begin{pmatrix}
I \\
K_t 
\end{pmatrix}\mathbb{E}[x_t x_t^\top|\mathcal{F}_{t-1}]\begin{pmatrix}
I \\
K_t 
\end{pmatrix}^\top +\begin{pmatrix}
0& 0\\
0 & \mathbb{E}[\eta_t \eta_t^\top|\mathcal{F}_{t-1}]
\end{pmatrix}.
  \end{align*}
  
By lower-bounding right hand side we can further write
\begin{align*}
 \nonumber  \mathbb{E}[z_t z_t^\top|\mathcal{F}_{t-1}]\succeq&  \frac{\sigma_{\omega}^2\bar{p}_t}{\sqrt{t+\bar{c}}} \begin{pmatrix}
\frac{\sqrt{t+\bar{c}}}{ \bar{p}_t } I& \frac{\sqrt{t+\bar{c}}}{ \bar{p}_t } K_t\\
\frac{\sqrt{t+\bar{c}}}{ \bar{p}_t } K_t^\top& \frac{\sqrt{t+\bar{c}}}{ \bar{p}_t } K_tK_t^\top + K_tK_t^\top+\frac{\|P_t\|}{\alpha_0}
\end{pmatrix}\\
& \succeq\frac{\sigma_{\omega}^2\bar{p}_t}{\sqrt{t+\bar{c}}}\begin{pmatrix}
(\frac{\sqrt{t+\bar{c}}}{ \bar{p}_t }-\frac{1}{2}+\frac{1}{2}) I& \frac{\sqrt{t+\bar{c}}}{ \bar{p}_t } K_t\\
\frac{\sqrt{t+\bar{c}}}{ \bar{p}_t } K_t^\top& \frac{\sqrt{t+\bar{c}}}{ \bar{p}_t } K_tK_t^\top + K_tK_t^\top +\frac{I}{2}
\end{pmatrix}\\
& \succeq \frac{\sigma_{\omega}^2\bar{p}_t}{\sqrt{t+\bar{c}}}\begin{pmatrix}
 \frac{I}{2}& 0\\
0& \frac{I}{2}
\end{pmatrix}+\frac{\sigma_{\omega}^2\bar{p}_t}{\sqrt{t+\bar{c}}} \underbrace{\begin{pmatrix}
(\overbrace{\frac{\sqrt{t+\bar{c}}}{ \bar{p}_t }-\frac{1}{2}) I}^{\bar{\Gamma}_{11}}& \overbrace{\frac{\sqrt{t+\bar{c}}}{ \bar{p}_t } K_t}^{\bar{\Gamma}_{12}}\\
\underbrace{\frac{\sqrt{t+\bar{c}}}{ \bar{p}_t } K_t^\top}_{\bar{\Gamma}_{21}}& \underbrace{\frac{\sqrt{t+\bar{c}}}{ \bar{p}_t } K_tK_t^\top + K_tK_t^\top}_ {\bar{\Gamma}_{22}}
\end{pmatrix}}_{\bar{\Gamma}}\\
& \succeq \frac{\sigma_{\omega}^2\bar{p}_t}{2\sqrt{t+\bar{c}}} I \quad \quad   \forall t\quad \textit{such that} \quad  \frac{\sqrt{t+\bar{c}}}{ \bar{p}_t }-\frac{1}{2}\geq 0,
\end{align*}
where in the second inequality we used the fact that $\|P_t\|/\alpha_0\geq 1/2$ (see (\ref{eq:inequality_of_P}) in the proof of Lemma \ref{Stability_lemma18}). The last inequality holds because $\bar{\Gamma}\succ 0$ by Schur complement noting that $\bar{\Gamma}_{11}\succ 0$ and $\bar{\Gamma}_{22}-\bar{\Gamma}_{12}\bar{\Gamma}_{11}^{-1}\bar{\Gamma}_{21}\succ 0$. 
\end{proof}

\begin{remark}
Note that $\bar{p}_t = \mathcal{O}(\log t)$, and $\bar{c}$ for either the ARSLO or ARSLO$^+(\bar{\rho})$ algorithms is chosen such that
\begin{align*}
    B_* \frac{\bar{p}_t \, \sigma_{\omega}^2 \Big(K(\mathcal{C}_t) K^\top(\mathcal{C}_t) + \frac{\|P(\mathcal{C}_t)\| I}{\alpha_0}\Big)}{\sqrt{t+\bar{c}}} B_*^\top \preceq  \sigma_{\omega}^2 I.
\end{align*}
This condition is enforced in the proof of Lemma \ref{lem:costatebound} for the ARSLO algorithm, and similarly in the proof of Theorem \ref{prop2} for ARSLO$^+(\bar{\rho})$, by requiring the sufficient condition, for an appropriately chosen $\bar{c}$,
\begin{align*}
    2 \kappa_t^2 \bar{\vartheta}_{B_*}^2 \frac{\bar{p}_t}{\sqrt{t+\bar{c}}} \leq 1,
\end{align*}
which yields
\begin{align}
    \frac{\bar{p}_t}{\sqrt{t+\bar{c}}} \leq 1 \quad \text{as } \kappa_t \geq 1.\label{eq:veryimptn}
\end{align}
This ensures the existence of a small value $s$ such that
\begin{align*}
    \frac{\sqrt{t+\bar{c}}}{\bar{p}_t} - \frac{1}{2} \geq 0 \quad \text{for all } t \ge s,
\end{align*}
and consequently, the minimum time step for which the result of Lemma \ref{lem:lowboundV} holds is small.
\end{remark}

\begin{lemma}(Lemma 35 \cite{cohen2019learning})\label{lem:borrowedBycohen}
Let \( a \in \mathbb{R}^{n+m} \) be a unit vector, and define \( S_t = a^\top z_t \). Consider the random indicator variable \( I_t \) defined by
\begin{align*}
    I_t = 
    \begin{cases}
        1 & \text{if } S_t^2 > \dfrac{\sigma_{\omega}^2 \bar{p}_t}{4\sqrt{t+\bar{c}}}, \\
        0 & \text{otherwise}.
    \end{cases}
\end{align*}
Then, it holds that \( \mathbb{E}[I_t \mid \mathcal{F}_{t-1}] \geq \dfrac{1}{5} \).
\end{lemma}

Now, we are in a position to introduce the key lemma which will be useful for stability analysis purpose. 

\begin{lemma} 
    Let $a\in \mathbb{R}^{n+m}$ be a unit vector, then
    
    \begin{align}
        a^\top \big(\sum_{k=1}^{t} z_k z_k^\top\big) a \geq \frac{\sigma_{\omega}^2 \bar{p}_t}{40}\sqrt{t+\bar{c}}- {\frac{\sigma_{\omega}^2\bar{c}}{40}} \label{eq:kmk092}
    \end{align}
for $t\geq 200 \log 1/\delta$ with probability at least $1-\delta$ where $\delta \in (0,\; 1)$.
\end{lemma}
\begin{proof}
  By defining $\mathcal{M}_t= I_t-\mathbb{E}[I_t|\; \mathcal{F}_{t-1}]$ and noting that it is a martingale difference and the fact that $|\mathcal{M}_t|\leq 1$. For each $s$, apply Azuma with failure probability $\delta/t$. Then
\begin{align}
    \mathbb{P}\!\left(\forall\, s \le t:\;\sum_{k=1}^{s} M_k\ge-\sqrt{2s\log\!\left(\frac{t}{\delta}\right)}
\right)\ge 1-\delta. \label{eq:addedUniformity}
\end{align}
Hence, with probability at least $1-\delta$
\begin{align*}
   \sum_{k=1}^t \mathcal{M}_k\geq-\sqrt{2t\log \frac{t}{\delta}}\geq - \frac{t}{10}
\end{align*}
  holds for $t\geq 200 \log \frac{t}{\delta}$. By definition of $\mathcal{M}_k $ this implies 
\begin{align*}
    \sum_{k=1}^t I_k\geq \sum_{k=1}^t \mathcal{M}_k+\sum_{k=1}^t \mathbb{E}[I_k| \mathcal{F}_{k-1}]\geq -\frac{t}{10}+\sum_{k=1}^t \mathbb{E}[\mathcal{M}_k| \mathcal{F}_{k-1}]\geq -\frac{t}{10}+ \frac{t}{5}\geq \frac{t}{10}
\end{align*}
  where the last inequality holds by Lemma \ref{lem:borrowedBycohen}. We further have $S_k^2\geq  I_k \frac{\sigma_{\omega}^2 \bar{p}_k }{4\sqrt{k+\bar{c}}}$ which results in
\begin{align*}
    a^\top \big(\sum_{k=1}^{t} z_k z_k^\top\big) a&= \sum_{k=1}^t S_k^2 \geq \sum_{k=1}^t I_k \frac{\sigma_{\omega}^2 \bar{p}_k }{4\sqrt{k+\bar{c}}}\geq \frac{\sigma_{\omega}^2 \bar{p}_t}{4\sqrt{t+\bar{c}}}\sum_{k=1}^t I_k\\
    &\geq \frac{\sigma_{\omega}^2 \bar{p}_t}{40\sqrt{t+\bar{c}}}t= \frac{\sigma_{\omega}^2 \bar{p}_t}{40}\sqrt{t+\bar{c}}-\frac{\sigma_{\omega}^2 \bar{p}_t}{40\sqrt{t+\bar{c}}}\bar{c}\\
    &\geq \frac{\sigma_{\omega}^2 \bar{p}_t}{40}\sqrt{t+\bar{c}}-\frac{\sigma_{\omega}^2\bar{c}}{40}
\end{align*}
 where in the last inequality we applied (\ref{eq:veryimptn}). This completes the proof.
\end{proof}

Now we use the result of Lemma 11, to find a lower-bound on the minimum eigenvalue of co-variance matrix $\sum_{k=1}^t z_kz_k^\top$. The following Lemma concludes the result.
\begin{lemma}\label{eq:lowerbndcov}
    Fix $\delta\in (0,\,1)$. Then with probability at least $1-\delta$ we have
    \begin{align}
        \lambda_{min}\big(\sum_{k=1}^t z_kz_k^\top\big)\geq \frac{\sigma_{\omega}^2 \bar{p}_t}{80}\sqrt{t+\bar{c}}- \underbrace{\frac{\sigma_{\omega}^2\bar{c}}{80}}_{=:\bar{C}}\label{eq:resLem10}
    \end{align}
    for $t\geq 400(n+m+\log \frac{t}{\delta})$.
\end{lemma}
\begin{proof}
We establish the result using a standard covering argument over the unit sphere, following the general strategy of \cite{cohen2019learning}, but we include it here for completeness.

Let $\mathcal{N}(1/4)$ be a minimal $1/4$-net of the unit sphere $\mathbb{S}^{n+m-1}$, and define
\begin{align*}
\bar{V}_t := \sum_{k=1}^{t} z_k z_k^\top .
\end{align*}
Consider the set
\begin{align*}
\mathcal{M}
\;=\;
\left\{
\frac{\bar{V}_t^{-1/2}u}{\|\bar{V}_t^{-1/2}u\|}
\;:\;
u \in \mathcal{N}(1/4)
\right\}.
\end{align*}
Assume that $t \;\ge\; 200 \log\!\left({|\mathcal{M}|}/{\delta}\right)$. Using~\eqref{eq:kmk092} and applying a union bound over all $n \in \mathcal{M}$, we obtain that with probability at least $1-\delta$,
\begin{align*}
n^\top \bar{V}_t n
\;\ge\;
\frac{\sigma_{\omega}^2 \bar{p}_t}{40}\sqrt{t+\bar{c}}
- \frac{\sigma_{\omega}^2\bar{c}}{40}
\;=:\;
\mathcal{A}_t .
\end{align*}

By construction, this implies that for every $u \in \mathcal{N}(1/4)$,
\begin{align}
u^\top \bar{V}_t^{-1} u
\;\le\;
\frac{1}{\mathcal{A}_t}.
\label{eq:netbound}
\end{align}

Let $z$ be a unit eigenvector corresponding to the smallest eigenvalue of $\bar{V}_t$. By the definition of the $1/4$-net, there exists $u_z \in \mathcal{N}(1/4)$ such that $\|z-u_z\| \le 1/4$. We then have
\begin{align*}
\|\bar{V}_t^{-1}\|
= z^\top \bar{V}_t^{-1} z  &\le u_z^\top \bar{V}_t^{-1} u_z
   + (z - u_z)^\top \bar{V}_t^{-1} (z + u_z) \\
&\le u_z^\top \bar{V}_t^{-1} u_z
   + \|z - u_z\| \, \|\bar{V}_t^{-1}\| \, \|z + u_z\| \\
&\le \frac{1}{\mathcal{A}_t}
   + \frac{2}{4} \|\bar{V}_t^{-1}\|,
\end{align*}
where we used~\eqref{eq:netbound} and the fact that $\|z\|=\|u_z\|=1$.

Rearranging yields
\begin{align*}
\|\bar{V}_t^{-1}\| \;\le\; \frac{2}{\mathcal{A}_t}.
\end{align*}

Finally, note that $|\mathcal{M}| = |\mathcal{N}(1/4)|$, and standard bounds give $|\mathcal{N}(1/4)| \le 12^{n+m}$. Hence,
\begin{align*}
t \;\ge\; 400 \bigl(n+m+\log(1/\delta)\bigr)
\end{align*}
ensures the required condition on $t$. Since $\|\bar{V}_t^{-1}\| = 1/\lambda_{\min}(\bar{V}_t)$, the claim follows.
\end{proof}
\begin{remark}
 Note that (\ref{eq:addedUniformity}) ensures that the derived lower bound holds uniformly in time with probability at least $1-\delta$.   
\end{remark}

\section{Stability Analysis of ARSLO Algorithm}\label{ARSLoStabAn}
 To establish stability, we first need the following lemma, which is equivalent to Lemma 16 of \cite{cohen2019learning}, with the only difference being that we do not require the selection of a regularization parameter $\lambda$ of $\mathcal{O}(\sqrt{T})$. This distinction is indeed the core objective of our study.
\begin{lemma} \label{lem:consoldualprim}
    Suppose $\Sigma(\mathcal{C}_t)$ and $P(\mathcal{C}_t)$ be the solution of the  relaxed primal and dual SDPs (\ref{eq:RelaxedSDP}) and (\ref{eq:RedSDP_DUAL}) respectively. Having $\mu_t\geq r_t+ 2 \vartheta \sqrt{r_t}\|V_t\|^{1/2}$ then $\Sigma_{xx}(\mathcal{C}_t)$ is invertible and for $K(\mathcal{C}_t)=\Sigma_{ux}(\mathcal{C}_t)\Sigma_{xx}(\mathcal{C}_t)^{-1}$ we have
\begin{align}
    \nonumber P(\mathcal{C}_t)= & \; Q+K^\top(\mathcal{C}_t) R K(\mathcal{C}_t)+ (\hat{A}_t+\hat{B}_tK(\mathcal{C}_t))^\top P(\mathcal{C}_t) (\hat{A}_t+\hat{B}_tK(\mathcal{C}_t))\\
    &-\mu_t \| P(\mathcal{C}_t)\|_*\begin{pmatrix}
I \\
K(\mathcal{C}_t)
\end{pmatrix}^\top V^{-1}_t \begin{pmatrix}
I \\
K(\mathcal{C}_t) 
\end{pmatrix}  \label{eq:lem9}
\end{align}   
\end{lemma}
\begin{proof}
   The proof follows the same steps as for Lemma 16 of (\cite{cohen2019learning}) with only difference that we need to show that
\begin{align}
   & \mu_t \big(\Sigma (\mathcal{C}_t)\bullet V_t^{-1}\big) I\preceq \sigma_{\omega}^2 I \label{eq:cl1}\\
   & \mu_t \|P(\mathcal{C}_t)\|V^{-1}_t \preceq \begin{pmatrix}
Q & 0 \\
0 & R 
\end{pmatrix} \label{eq:cl2}
\end{align}
using which all the claims follows. By appropriately tuning the parameters $\lambda$, $\bar{p}_t$, and $\bar{\epsilon}$ according to Theorem \ref{Stability_thm17} it is shown that 
\begin{align}
 \mu_t \|P(\mathcal{C}_t)\|V_t^{-1}\preceq \frac{\alpha_0}{4} I, \label{eq:komaki}   
\end{align}

 which together with the fact that $Q, R\preceq \alpha_0$ 
 (\ref{eq:cl2}) holds true. On the  other hand, noting that $\|\Sigma (\mathcal{C}_t)\|_*\preceq \nu/ \alpha_0$ and (\ref{eq:komaki})
 \begin{align*}
     \mu_t \big(\Sigma (\mathcal{C}_t)\bullet V_t^{-1}\big) I&=  \operatorname{Tr}\big(\frac{1}{ \|P(\mathcal{C}_t)\|}\Sigma (\mathcal{C}_t)^\top \mu_t \|P(\mathcal{C}_t)\| V_t^{-1}\big) I\\
     &\preceq\operatorname{Tr}\big(\frac{1}{ \|P(\mathcal{C}_t)\|}\Sigma (\mathcal{C}_t)^\top \begin{pmatrix}
Q & 0 \\
0 & R 
\end{pmatrix}\big) I\\
&= \frac{1}{ \|P(\mathcal{C}_t)\|} \Sigma (\mathcal{C}_t)\bullet\begin{pmatrix}
Q & 0 \\
0 & R 
\end{pmatrix}\\
&= \frac{1}{\|P(\mathcal{C}_t)\|} P(\mathcal{C}_t)\bullet W \\
&\preceq \frac{1}{\|P(\mathcal{C}_t)\|} \|P(\mathcal{C}_t)\| \sigma_{\omega}^2 I=\sigma_{\omega}^2 I
 \end{align*}
 that completes the proof for (\ref{eq:cl1}).   
\end{proof}

Now we are in a position to provide the proof for Theorem \ref{Stability_thm17}.

\subsection{Proof of Theorem \ref{Stability_thm17}}

The proof of the theorem consists of two steps. First, we establish strong stability using Lemma \ref{Stability_lemma18}. Then, we prove the sequential property of the generated policies in Lemma \ref{lem:sequentiality}.

\begin{lemma} \label{Stability_lemma18}
Any feedback gain $K(\mathcal{C}_t)$ designed by the ARSLO algorithm is $(\kappa_*, \gamma_*)$-strongly stabilizing, where $\kappa_* = \sqrt{\frac{2\|P_*\|}{\alpha_0}}$ and $\gamma_* = \kappa_*^{-2}/2$.
\end{lemma}

\begin{proof}
  One can write  
  \begin{align}
     \nonumber  (\hat{A}_t+\hat{B}_tK(\mathcal{C}_t))^\top P(\mathcal{C}_t) (\hat{A}_t+\hat{B}_tK(\mathcal{C}_t))&= \begin{pmatrix}
I \\
K(\mathcal{C}_t)
\end{pmatrix}^\top \hat{\Theta}_t P(\mathcal{C}_t) \hat{\Theta}_t^\top \begin{pmatrix}
I \\
K(\mathcal{C}_t)
\end{pmatrix}\\
 &\preceq   \begin{pmatrix}
I \\
K(\mathcal{C}_t)
\end{pmatrix}^\top {\Theta}_* P(\mathcal{C}_t) {\Theta}_*^\top \begin{pmatrix}
I \\
K(\mathcal{C}_t)
\end{pmatrix}^\top \label{eq:appetl}\\
\nonumber &-\mu_t \|P(\mathcal{C}_t)\| \begin{pmatrix}
I \\
K(\mathcal{C}_t)
\end{pmatrix}^\top V_t^{-1}\begin{pmatrix}
I \\
K(\mathcal{C}_t)
\end{pmatrix}\\
 \nonumber  &=({A}_*+{B}_*K(\mathcal{C}_t))^\top P(\mathcal{C}_t) ({A}_*+{B}_*K(\mathcal{C}_t))\\
  &-\mu_t \|P(\mathcal{C}_t)\| \begin{pmatrix}
I \\
K(\mathcal{C}_t)
\end{pmatrix}^\top V_t^{-1}\begin{pmatrix}
I \\
K(\mathcal{C}_t)
\end{pmatrix} \label{eq:komaki3}
  \end{align}   
where $\hat{\Theta}_t=(\hat{A}_t,\, \hat{B}_t)^\top$ and in the inequality (\ref{eq:appetl}) we applied the perturbation lemma, Lemma \ref{lem:purturbation}. Substituting (\ref{eq:komaki3}) into (\ref{eq:lem9}), yields 
\begin{align}
    \nonumber P(\mathcal{C}_t)\succeq & \; Q+K^\top(\mathcal{C}_t)) R K(\mathcal{C}_t))+ ({A}_*+{B}_*K(\mathcal{C}_t))^\top P(\mathcal{C}_t) ({A}_*+{B}_*K(\mathcal{C}_t))\\
    &-2\mu_t \| P(\mathcal{C}_t)\|\begin{pmatrix}
I \\
K(\mathcal{C}_t)
\end{pmatrix}^\top V^{-1}_t \begin{pmatrix}
I \\
K(\mathcal{C}_t) 
\end{pmatrix}.  \label{eq:ForStabilityp}
\end{align}

Now we aim to tune parameters $\lambda$ and $\bar{p}_t$ and specify $\bar{\epsilon}$, the error upper-bound of $\Theta_0$ to be provided by the warm-up algorithm, such that 
\begin{align}
    \mu_t  \| P(\mathcal{C}_t)\| V_t^{-1}\preceq \frac{\alpha_0}{32 \kappa_t^8}I. \label{eq: verygood-}
\end{align}
where
\begin{align}
    \kappa_t=\sqrt{\frac{2\|P(\mathcal{C}_t)\|}{\alpha_0}}.\label{eq:defkappa}
\end{align}
Note that the subscript $t$ in $\kappa_t$ corresponds to the subscript $t$ in $\mathcal{C}_t$, which denotes the confidence ellipsoid constructed using data up to time $t$.

Applying (\ref{eq:defkappa}), the condition (\ref{eq: verygood-}) can be rewritten as follows
\begin{align}
      \mu_t  V_t^{-1}\preceq \frac{1}{16 \kappa_t^{10}}I. \label{eq: verygood}
\end{align}

% \begin{align}
%     \mu_t  \| P(\mathcal{C}_t)\|_* V_t^{-1}\preceq \frac{\alpha_0}{4}I \label{eq: verygood22}
% \end{align}

% Given that $\| P(\mathcal{C}_t)\|_*\leq \|P_*\|_*$ by \textcolor{blue}{Lemma}, the sufficient condition for (\ref{eq: verygood}) to hold is
% \begin{align}
%     \mu_t  V_t^{-1}\preceq \frac{\alpha_0}{32 \kappa^8 \|P_*\|_*}I \label{eq: verygood2}.
% \end{align}

By definition $\mu_t= r_t+ 2\vartheta \sqrt{r_t} \|V_t\|^{1/2}$ a sufficient condition for (\ref{eq: verygood}) is that
 \begin{align}
  & r_tV_t^{-1}\preceq \frac{1}{32 \kappa_t^{10}}I \label{eq:Term1}\\
   & 2 \vartheta \sqrt{r_t} \|V_t\|^{1/2}V_t^{-1}\preceq \frac{1}{32 \kappa_t^{10}}I. \label{eq:Term2}
 \end{align}
provided that both inequalities hold simultaneously.

 By applying the result of Lemma \ref{eq:lowerbndcov}, namely the inequality (\ref{eq:resLem10}) for $t\geq 400(n+m+\log \frac{t}{\delta})$, one can write:
\begin{align*}
 2 \vartheta \sqrt{r_t} \|V_t\|^{1/2}V_t^{-1}&\preceq \frac{2\vartheta \sqrt{r_t}\|V_t\|^{\frac{1}{2}}}{\frac{\sigma_{\omega}^2 \bar{p}_t}{80}\sqrt{t+\bar{c}}-\bar{C}+\lambda} I
\end{align*}
where 
\begin{align*}
    \bar{C}:=\frac{\sigma_{\omega}^2 \bar{c}}{80}.
\end{align*}

Then a sufficient condition for (\ref{eq:Term2}) is 
\begin{align}
   \frac{2\vartheta \sqrt{r_t}\|V_t\|^{\frac{1}{2}}}{\frac{\sigma_{\omega}^2 \bar{p}_t}{80}\sqrt{t+\bar{c}}-\bar{C}+\lambda} &\leq \frac{1}{32 \kappa_t^{10}} \label{eq:pcsdg0}
\end{align}
and similarly for (\ref{eq:Term1}) is

\begin{align}
     \frac{r_t}{\frac{\sigma_{
\omega}^2 \bar{p}_t}{80}\sqrt{t+\bar{c}}-\bar{C}+\lambda}\leq \frac{1}{32 \kappa_t^{10}} \label{eq:pcsdg+}
\end{align}
Let further set

\begin{align}
    \lambda\geq \bar{C} \label{eq:addlambda}
\end{align}
then the condition for (\ref{eq:pcsdg0}) can be written as follows

\begin{align*}
   \frac{2\vartheta \sqrt{r_t}\|V_t\|^{\frac{1}{2}}}{\frac{\sigma_{\omega}^2 \bar{p}_t}{80}\sqrt{t+\bar{c}}-\bar{C}+\lambda} &\leq \frac{2\vartheta \sqrt{r_t} (\lambda+\|\sum_{k=1}^tz_kz_k^\top\|)^{1/2}}{\frac{\sigma_{\omega}^2 \bar{p}_t}{80}\sqrt{t+\bar{c}}}\leq \frac{1}{32 \kappa_t^{10}} 
\end{align*}
To satisfy the inequality above, it suffices to require

\begin{align*}
    \bar{p}_t\geq \frac{5120\, \kappa_t^{10}\vartheta \sqrt{r_t}(\lambda+\|\sum_{k=1}^t z_kz_k^\top\|)^{1/2}}{\sigma_{\omega}^2 \sqrt{t+\bar{c}}}.
\end{align*}
Similarly, a sufficient condition for (\ref{eq:pcsdg+}) is given by

\begin{align*}
    \bar{p}_t\geq \frac{2560\, r_t\,\kappa_t^{10}}{\sigma_{\omega}^2 \sqrt{t+\bar{c}}}.
\end{align*}
We then choose $\bar{p}_t$ as

\begin{align}
    \bar{p}_t\geq\max\bigg\{\frac{5120\, \kappa_t^{10}\vartheta \sqrt{r_t}(\lambda+\|\sum_{k=1}^t z_kz_k^\top\|)^{1/2}}{\sigma_{\omega}^2 \sqrt{t+\bar{c}}},\, \frac{2560\, r_t\,\kappa_t^{10}}{\sigma_{\omega}^2 \sqrt{t+\bar{c}}}\bigg\}.\label{eq:barpdef}
\end{align}
Note that, one can show that
$\bar{p}_t$ is determined
by the first term for all $t$.

Next, we proceed to specify $\bar{c}$ such that
\begin{align}    B_*\frac{\bar{p}_t\sigma_{\omega}^2\big(K(\mathcal{C}_t) K^\top(\mathcal{C}_t)+\frac{\|P(\mathcal{C}_t)\|I}{\alpha_0}\big)}{\sqrt{t+\bar{c}}} B_*^\top\preceq \sigma_{\omega}^2I.\label{eq:noiseNbigs}
\end{align}
This condition ensures that the effect of input perturbation noise on the closed-loop system does not exceed the magnitude of the process noise. 

A sufficient condition for \eqref{eq:noiseNbigs} can be written as
 \begin{align}
     \frac{2  \, \kappa_t^2 \, \bar{\vartheta}_{B_*}^2 \bar{p}_t}{\sqrt{t+\bar{c}}}\leq 1 \label{eq:goodcbarfinds}
 \end{align}
where we define $\bar{\vartheta}_{B_*} = \max\{1, \vartheta_{B_*}\}$ and $\|B_*\| \leq \vartheta_{B_*}$.

To derive a suitable upper bound for $\bar{p}_t$, we first establish the following inequalities.

We have
\begin{align}
    \log \frac{n\det (V_t)}{\delta\det(\lambda I)}
    &\le (n+m)\log\frac{n}{\delta}\Big(1 + t\log\frac{t}{\delta}\Big) \label{eq:seqps}\\
    &\le 2(n+m)\log\frac{2nt}{\delta}. \label{eq:seqpf1}
\end{align}
where \eqref{eq:seqps} holds under the choice
\begin{align}
    \lambda\geq \frac{\sum_{k=1}^z \|z_k z_k^\top\|^2}{t \log \frac{2t}{\delta}}.\label{eq:Lam_1}
\end{align}

By choosing $\epsilon$ appropriately such that
\begin{align}
    \sqrt{\lambda} \epsilon \leq \sigma_{\omega}\sqrt{2n(n+m)} \label{eq:epslambs}
\end{align}
it then follows from \eqref{eq:seqpf1} and \eqref{eq:epslambs} that
\begin{align}
    r_t\leq 16\sigma_{\omega}^2n(n+m)\log \frac{2nt}{\delta} \label{eq:barrDefb}.
\end{align}
Furthermore, from {Lemma} \ref{lem:costatebound}, we have
\begin{align*}
    \|z_t\|^2\leq  c^2_z(\kappa_*)\log \frac{t}{\delta}.
\end{align*}

To obtain a computable value for $\bar{c}$, we require an upper bound on $\kappa_*$, which can be derived directly from Lemma~\ref{lem:closeness}:
\begin{align*}
\kappa_t^2 \leq \kappa_*^2 \leq \kappa_t^2 + \frac{1}{2\kappa_t^2} := \bar{\kappa}_t^2, \quad \forall t.
\end{align*}
From this, we can define \begin{align*}
    \kappa_t\leq \kappa_*\leq  \bar{\kappa}_1:=\sqrt{\kappa_1^2+\frac{1}{\kappa_1^2}}.
\end{align*}
Combining these results, we can now derive and upper-bound for the left hand side of (\ref{eq:goodcbarfinds}) as follows
\begin{align}
    \frac{40960\, \bar{\vartheta}^2_{B_*}\kappa_t^{10}\vartheta \sqrt{r_t}}{\sigma_{\omega}^2 ({t+\bar{c}})} \big(\lambda+t\, c_z^2(\kappa_*)\log \frac{t}{\delta}\big)^{1/2}\leq & \frac{40960\, \bar{\vartheta}^2_{B_*}\bar{\kappa}_1^{10}\vartheta \sqrt{\sigma_{\omega}^2n(n+m)\log \frac{t}{\delta}}}{\sigma_{\omega}^2 ({t+\bar{c})}}\sqrt{\lambda}\label{eq:tm11}\\
    & +\frac{40960\, \bar{\vartheta}^2_{B_*}\bar{\kappa}_1^{10}\vartheta \sqrt{\sigma_{\omega}^2n(n+m)t}}{\sigma_{\omega}^2 ({t+\bar{c}})}c_z(\bar{\kappa}_1)\log \frac{2nt}{\delta}\label{eq:tm12}
\end{align}
We enforce the following condition:
\begin{align}
    \frac{40960\, \bar{\vartheta}^2_{B_*}\bar{\kappa}_1^{10}\vartheta \sqrt{\sigma_{\omega}^2n(n+m)}}{\sigma_{\omega}^2 \sqrt{t+\bar{c}}}c_z(\bar{\kappa}_1)\log \frac{2nt}{\delta}\leq \frac{1}{2}.\label{eq:kmkna}
\end{align}

Applying Lemma \ref{lem:Usefulalgebricin}, an appropriated choice of $\bar{c}$ that satisfies this condition is
 \begin{align}
    \bar{c}= a^2 \Big( 2 \log\Big(\frac{4 a^2 n}{\delta}\Big) + 1 \Big)^2. \label{eq:magncbars}
\end{align}
where
\begin{align*}
    a = \frac{40960\, \bar{\vartheta}^2_{B_*}\bar{\kappa}_1^{10}\vartheta \sqrt{\sigma_{\omega}^2n(n+m)}c_z(\bar{\kappa}_1)}{\sigma_{\omega}^2 }.
\end{align*}
With this choice of $\bar{c}$, condition (\ref{eq:kmkna}) ensures that the term (\ref{eq:tm12}) is less than $1/2$. Moreover, by choosing $\lambda=\bar{C}=\frac{\sigma_\omega^2 \bar{c}}{80}$ we also satisfy (\ref{eq:tm11}) to be less than $1/2$.

Recalling (\ref{eq:Lam_1}), which provides an additional condition for tuning $\lambda$ based on the upper
bound of the co-state, we require
\begin{align*}
    \lambda\geq c_z^2(\bar{\kappa}_1).
\end{align*}
Therefore, we select
\begin{align}
    \lambda=\max \bigg\{\frac{\sigma_{\omega}^2\bar{c}}{80},\, c_z^2(\bar{\kappa}_1)\bigg\}=\frac{\sigma_{\omega}^2\bar{c}}{80}\label{eq:lamchosens}
\end{align}
where the equality holds due to the magnitude of $\bar{c}$ in \eqref{eq:magncbars}.

With $\bar{c}$ specified, $\lambda$ is determined, then, from \eqref{eq:barpdef}, $\bar{p}_t$ is obtained as
\begin{align}
    \bar{p}_t:=\frac{5120\, \kappa_t^{10}\vartheta \sqrt{r_t}(\lambda+\|\sum_{k=1}^t z_kz_k^\top\|)^{1/2}}{\sigma_{\omega}^2 \sqrt{t+\bar{c}}}.\label{eq:candbarptps}
\end{align}

Recalling \eqref{eq:epslambs}, the appropriate choice of $\bar{\epsilon}(\bar{\kappa}_1)$ is given by
\begin{align}
     \bar{\epsilon}(\bar{\kappa}_1) := \frac{\sigma_\omega \sqrt{2 n (n+m)}}{\sqrt{\lambda}}.\label{eq:epscrudebnd0102s}
\end{align}

With $\bar{\epsilon}(\bar{\kappa}_1)$ given by (\ref{eq:epscrudebnd0102s}) and $\lambda$ and $\bar{p}_t$ tuned according to (\ref{eq:lamchosens}) and (\ref{eq:candbarptps}), respectively, (\ref{eq: verygood}) holds true, resulting in the satisfaction of
\begin{align}
    \mu_t  \| P(\mathcal{C}_t)\|_* V_t^{-1}\preceq \frac{\alpha_0}{4}I. \label{eq: verygood22}
\end{align}
Applying this into (\ref{eq:ForStabilityp}) yields
\begin{align}
 P(\mathcal{C}_t)& \succeq \frac{\alpha_0}{2} I+\frac{\alpha_0}{2}K^\top (\mathcal{C}_t)K(\mathcal{C}_t)+(A_*+B_*K(\mathcal{C}_t))^\top P(\mathcal{C}_t)(A_*+B_*K(\mathcal{C}_t)) \label{eq:inequality_of_P}
\end{align} 
where we applied $Q,R \succeq \alpha_0 I$. Furthermore, (\ref{eq:inequality_of_P}) implies 
\begin{align}
   (A_*+B_*K(\mathcal{C}_t))^\top P(\mathcal{C}_t)(A_*+B_*K(\mathcal{C}_t)) \preceq P(\mathcal{C}_t)-\frac{\alpha_0}{2}I. \label{eq:whatneeded}
\end{align}
By pre and post multiplying both sides of (\ref{eq:whatneeded}) with $P^{-1/2}(\mathcal{C}_t)$ we can further write:	
\begin{align}
\nonumber {P}^{-1/2}(\mathcal{C}_t)(A_*+B_*K(\mathcal{C}_t))^\top P(\mathcal{C}_t)(A_*+B_*K(\mathcal{C}_t)){P}^{-1/2}(\mathcal{C}_t) &\preceq I-\frac{\alpha_0}{2} {P}^{-1}(\mathcal{C}_t)&\\
 &\preceq\quad (1-\frac{\alpha_0 }{2 \|P(\mathcal{C}_t)\|})I \label{eq:uppBound_L-}
\end{align}
where in (\ref{eq:uppBound_L-}) we applied the fact that $P(\mathcal{C}_t)\preceq \|P(\hat{\Theta}_t\| I$.

Let 
\begin{align} 
    L_t=&{P}^{-\frac{1}{2}}(\mathcal{C}_t)(A_*+B_*K(\mathcal{C}_t)){P}^{\frac{1}{2}}(\mathcal{C}_t) \label{eq:strongstab}
\end{align}
noting that $P^{-1/2}(\mathcal{C}_t)$ exists by $P(\mathcal{C}_t)\succ 0$. Then by (\ref{eq:uppBound_L-}) we have 
 
 \begin{align}
     \|L_t\|\leq\sqrt{1-\frac{\alpha_0 }{2 \|P(\mathcal{C}_t)\|}}&\leq 1-\frac{1}{2}(\frac{\alpha_0 }{2 \|P(\mathcal{C}_t)\|})=:1-\gamma_t \label{eq:bndLp+}\\
     &\leq 1-\frac{1}{2}(\frac{\alpha_0 }{2 \|P_*\|})=:1-\gamma_* \label{eq:bndL}
 \end{align}
where inequality (\ref{eq:bndLp+}) follows from the fact that $\alpha_0 < 2 \|P(\mathcal{C}_t)\|$, and inequality (\ref{eq:bndL}) follows from $P(\mathcal{C}_t)\succeq P*$, as established in Lemma \ref{lem:closeness}.

Furthermore, by
 \begin{align*}
    \frac{\alpha_0}{2}{K}^\top (\mathcal{C}_t)K(\mathcal{C}_t) \preceq  P(\mathcal{C}_t) 
 \end{align*}

which follows from (\ref{eq:inequality_of_P}), we directly deduce that 
\begin{align}
    \|K(\mathcal{C}_t)\|&\leq \sqrt{\frac{2\|P(\mathcal{C}_t)\|}{\alpha_0}}:={\kappa}_t \label{eq:bndK}\\
    &\nonumber \leq \sqrt{\frac{2\|P_*\|}{\alpha_0}}:=\kappa_*.
\end{align}

By (\ref{eq:strongstab}) we can write
\begin{align*} 
  A_*+B_*K(\mathcal{C}_t)=H_t L_tH_t^{-1}
\end{align*}
where $H_t={P}^{1/2}(\mathcal{C}_t)$. We have already demonstrated upper bounds on $\|L_t\|$ and  $\|K(\mathcal{C}_t)\|$ by (\ref{eq:bndLp+}) and (\ref{eq:bndK}) respectively. By (\ref{eq:inequality_of_P}) we know that $P(\mathcal{C}_t)\succeq \frac{\alpha_0}{2}I$ which yields $\|H_t^{-1}\|\leq \sqrt{2/\alpha_0}:=1/b_0$. Furthermore, since $\|H_t\|=\|P (\mathcal{C}_t)\|^{1/2}:=B_0$ we obtain
\begin{align*}
  \|H_t^{-1}\|\|H_t\|&\leq \frac{1}{b_0} B_0=\sqrt{\frac{2\|P(\mathcal{C}_t)\|}{\alpha_0}} \\
 &\leq \sqrt{\frac{2\|P_*\|}{\alpha_0}}.
\end{align*}

Conclusively, recalling parts (1) and (2) of Definition \ref{def:sequentially}, we can conclude that any generated policies are $(\kappa, \gamma)-$ strong stabilizing (also $(\kappa_*, \gamma_*)-$ strong stabilizing) where $\gamma=\kappa^{-2}/2 $ and $\gamma_*=\kappa_*^{-2}/2$.

 Now we conclude that the statement of the theorem holds with probability at least $1 - 3\delta$. For this purpose, we note that the entire construction of the parameters $\lambda$ and $\bar{c}$ relies on the sequential strong stability of the generated policies, which is precisely the property we aim to establish.

To close the argument, we employ a bootstrap (self-consistency) technique. In particular, the choice of $\lambda$ and $\bar{c}$ is made under the satisfaction of (\ref{eq: verygood}), which in turn guarantees that the generated policies are $(\kappa_*, \gamma_*)$-strongly sequentially stabilizing. This establishes a self-consistent set of conditions under which the stability property holds, thereby completing the proof.
% Let us define
% \begin{align*}
%     T_t:=16\mu_t V_{t}^{-1}\kappa_t^{10}
% \end{align*}
Define the following events

\begin{align*}
     \mathcal{E}_t^1&:=\left\{\forall s = 1, \ldots, t,\quad \Theta_* \in \mathcal{C}_s(\delta)\right\}\\
      \mathcal{E}_t^2&:=\left\{\forall s = 1, \ldots, t,\quad \lambda_{min}(V_s)\geq \lambda+\frac{\sigma_{\omega}^2\bar{p}_t}{80}\sqrt{s+\bar{c}}-\bar{C} \right\}\\
      \mathcal{E}_t^3&:=\{\forall s = 1, \ldots,t,\quad \max_{1\leq k\leq s} \| B_*\eta_k+\omega_{k+1}\|\leq \sqrt{20 \sigma_{\omega}^2n\log \frac{t}{\delta}}\}\\
      \mathcal{E}_t^4&:=\{\forall s = 1, \ldots, t,\quad \bar{c}\geq T_s\}\\
       \mathcal{E}_t^5&:=\{\forall s = 1, \ldots, t,\quad \lambda\geq M_s\}
\end{align*}
where 
\begin{align*}
    T_s :=\frac{10240\, \kappa_s^{12}\bar{\vartheta}_{B_*}^2 \vartheta \sqrt{r_s}(\lambda+\|\sum_{k=1}^s z_kz_k^\top\|)^{1/2}}{\sigma_{\omega}^2}-s.
\end{align*}
and 
\begin{align*}
    M_s:= \max \bigg\{\bar{C},\, \frac{\sum_{k=1}^s \|z_k z_k^\top\|^2}{s \log \frac{2s}{\delta}}\bigg\}.
\end{align*}

From the uniform self-normalized martingale bound (Theorem 2 of \cite{abbasi2011improved}), we have
\begin{align}
    \mathbb{P}(\mathcal{E}_t^1) \ge 1 - \delta. \label{eq:helpprob1}
\end{align}
for any $\epsilon$. Here we constrained $\epsilon$ to satisfy (\ref{eq:epslambs}) after an appropriate choice of $\lambda$ that satisfies the event $\mathcal{E}_t^5$.

The event $\mathcal{E}_t^2$ holds with probability at least $1-\delta$ for any choice of $\bar{c}$ satisfying the event $\mathcal{E}_t^4$. Furthermore, the event $\mathcal{E}_t^3$ also holds with probability at least $1-\delta$ provided that $\bar{c}$ satisfies $\mathcal{E}_t^4$. In addition, for a given $\bar{c}$, the parameter $\lambda$ is chosen such that $\mathcal{E}_t^5$ holds.

Together, these events ensure that the sufficient conditions for $(\kappa_*, \gamma_*)$-sequential  strong stability are satisfied, provided that the initial estimate $\Theta_0$ is obtained from the warm-up phase.

Definite 
\begin{align}
    \mathcal{H}_t:=\mathcal{E}_t^1\cap \mathcal{E}_t^2\cap \mathcal{E}_t^3.\label{eq:mathcH0}
\end{align}

We have already shown that if there exist parameters $\bar{c}$ and $\lambda$ such that the events $\mathcal{E}_t^4$ and $\mathcal{E}_t^5$ hold and the event $\mathcal{H}_t$ is satisfied, then the generated policies are $(\kappa_*,\gamma_*)$-strongly sequentially stabilizing up to time $t$.

Furthermore, referring to the previous analysis, whenever the generated policies are $(\kappa_*,\gamma_*)$-strongly sequentially stabilizing, we have
\begin{align}
    T_s \leq C_1(\bar{\kappa}_1), \quad \text{and} \quad M_s \leq C_2(\bar{\kappa}_1), \quad \forall s = 1, \ldots, t,
    \label{eq:enerycon}
\end{align}
where $C_1(\bar{\kappa}_1)$ and $C_2(\bar{\kappa}_1)$ are deterministic constants depending on $\bar{\kappa}_1$, which is the output of the warm-up algorithm and is fixed prior to running the ARSLO algorithm.

Choose
\begin{align}
    \bar{c}\geq C_1(\bar{\kappa}_1) \quad \& \quad \lambda\geq C_2(\bar{\kappa}_1) \label{eq:modem223}
\end{align}
and then $\epsilon$ such that (\ref{eq:epslambs}) is fulfilled and call it $\bar{\epsilon} (\bar{\kappa}_1)$.

We claim that on the event $\mathcal{H}_t$, the generated policies are $(\kappa_*,\gamma_*)$-strongly sequentially stabilizing.

 To establish the claim, define the first failure time
\begin{align*}
    \tau
    :=
    \inf\Big\{
    s \ge 1 :
    \text{the generated policies are not }
    (\kappa_*,\gamma_*)\text{-strongly sequentially  stabilizing up to time }s
    \Big\}.
\end{align*}
Suppose, toward a contradiction, that $\tau \le t$. 
By the definition of $\tau$, the generated policies are
$(\kappa_*,\gamma_*)$-strongly sequentially stable up to time $\tau-1$.

By ({\ref{eq:enerycon}}),
\begin{align}
    T_s \le C_1(\bar{\kappa}_1),
    \qquad
    M_s \le C_2(\bar{\kappa}_1),
    \qquad
    \forall s \leq \tau .
\end{align}

Combining the above inequalities with
(\ref{eq:modem223}) yields
\begin{align}
    T_s \le \bar{c},
    \qquad
    M_s \le \lambda,
    \qquad
    \forall s \leq \tau .
\end{align}
Hence, $\mathcal{E}_\tau^4\cap\mathcal{E}_\tau^5$ hold. Since $\mathcal{H}_t$ holds for $\forall t$, we also have $\mathcal{H}_\tau$. Consequently, $\mathcal{H}_\tau
\cap
\mathcal{E}_\tau^4
\cap
\mathcal{E}_\tau^5$
holds. This implies that the generated policies are
$(\kappa_*,\gamma_*)$-strongly sequentially stable up to time $\tau$. This contradicts the definition of $\tau$.
Therefore, $\tau > t$, and hence the generated policies are $(\kappa_*,\gamma_*)$-strongly sequentially stable throughout the horizon. 

Recall that $\mathcal{E}_0(\epsilon)$, defined in (\ref{eq:condinitialtheta}), denotes the event that the initial estimate $\Theta_0$ satisfies the required accuracy level with $\epsilon=\bar{\epsilon}(\bar{\kappa}_1)$. Since
\begin{align}
    \mathbb{P}(\mathcal{H}_t \mid \mathcal{E}_0\big(\bar{\epsilon}(\bar{\kappa}_1))\big) \ge 1-3\delta,\label{eq:probHceps0}
\end{align}
we conclude that, conditioned on the event $\mathcal{E}_0$, the sequence of policies generated by ARSLO is $(\kappa_*,\gamma_*)$-strongly sequentially stabilizing with probability at least $1-3\delta$.

We carry out the proof relying on the uniform-in-time lower bound on the minimum eigenvalue of the covariance matrix given in Lemma \ref{eq:lowerbndcov}, which holds for
\[
t \ge 400\bigl(n+m+\log(t/\delta)\bigr),
\]
or, equivalently, via an elementary but slightly conservative algebraic manipulation for
\begin{align*}
t \;\ge\; 800b + 400\log(800b):=t_f,
\end{align*}
where \(b = n+m+\log \frac{1}{\delta} > 2\).

In this step, we show that for \(t \leq t_f\), the chosen parameters remain sufficient to guarantee the claim of the lemma. In particular, we verify that the selected \(\lambda\) ensures that the sufficient conditions \eqref{eq:Term1} and \eqref{eq:Term2} hold simultaneously, i.e.,
\begin{align*}
   & r_t\leq \frac{\lambda}{32 \kappa_t^{10}}\quad \&  \quad 2 \vartheta \sqrt{r_t}\,\|V_t\|^{1/2} \leq \frac{\lambda}{32 \kappa_t^{10}}.
\end{align*}

This follows directly from the definition of \(\bar{c}\), which is chosen to be of order \(\bar{\kappa}_1^{20}\), together with the fact that \(\bar{\kappa}_1 > \kappa_t\) for all \(t\).

\end{proof}

In the next step, we prove that by tuning $\lambda$ and $\bar{p}_t$, and specifying $\bar{\epsilon}$ according to Lemma \ref{Stability_lemma18}, the sequential property of the policies generated by the ARSLO algorithm (part (3) of Definition \ref{def:sequentially}) is fulfilled. Beforehand, we require the following lemma, which is helpful.

\begin{lemma} \label{lem:closeness}
   Given an initial estimate $\Theta_0$ satisfying (\ref{def:bareps}), and under the event $\mathcal{H}_t$, the following statement holds:
    \begin{align*}
        P(\mathcal{C}_t)\preceq P_* \preceq P(\mathcal{C}_{t+1})+ \frac{\alpha_0 \gamma_{t+1}}{2}.
    \end{align*}
  where $P_*$ is the optimal solution of dual program (\ref{eq:equal_dual}) and $\gamma_{t+1}=\kappa_{t+1}^{-2}/2$.
\end{lemma}

\begin{proof}
   The proof follows steps similar to those in Lemma 19 of \cite{cohen2019learning}, with only minor differences. For the sake of completeness, we provide it here.  
  
First we show that $ P(\mathcal{C}_t)\preceq P_*$.

By reapplying the perturbation lemma, Lemma \ref{lem:purturbation}, to the inequality of the relaxed dual SDP (\ref{eq:RedSDP_DUAL}), we obtain:
\begin{align}
    \begin{pmatrix}
Q-P(\mathcal{C}_t) & 0 \\
0 & R
\end{pmatrix}+{\Theta}_* P(\mathcal{C}_t) {\Theta}_*^\top \succeq 0 \label{eq:khoshe}
\end{align}
By post multiplying (\ref{eq:khoshe}) by $    \begin{pmatrix}
I\\
K_*
\end{pmatrix}$
and pre-multiplying by its transpose, we get:

\begin{align}
Q-P(\mathcal{C}_t)+K_*^\top R K_*+(A_*+B_*{K_*})^\top P(\mathcal{C}_t)(A_*+B_*K_*)\succeq 0 .\label{eq:joon01}
\end{align}

Combining (\ref{eq:joon01}) with (\ref{eq:promRic}), the optimal Riccati solution, yields:
\begin{align}
    P(\mathcal{C}_t)-P_*\preceq (A_*+B_*K_*)^\top (P(\mathcal{C}_t)-P_*)(A_*+B_*K_*) \label{eq:almos}
\end{align}

Utilizing Lemma \ref{lem:komaki} on (\ref{eq:almos}), we conclude $P(\mathcal{C}_t)\preceq P_*$.

Now we show the second inequality. From \cite{bertsekas2012dynamic} we have:
\begin{align*}
    P_*\preceq Q+K^\top (\mathcal{C}_t)R K(\mathcal{C}_t)+(A_*+B_*K(\mathcal{C}_t))^\top P_*(A_*+B_*K(\mathcal{C}_t))
\end{align*}
Combining this with (\ref{eq:ForStabilityp}) and applying Lemma \ref{lem:komaki}, along with the fact that the generated policies are $(\kappa, \gamma)$-strongly stabilizing, implies

\begin{align}
   \nonumber \nonumber  P_*- P(\mathcal{C}_t)&\preceq \frac{16\| P(\mathcal{C}_t)\|^2\mu_t}{\alpha_0^2} \|P(\mathcal{C}_t)\| \|\begin{pmatrix}
I \\
K(\mathcal{C}_t)
\end{pmatrix}^\top V^{-1}_t \begin{pmatrix}
I \\
K(\mathcal{C}_t) 
\end{pmatrix} \| I\\
&= 4\kappa_t^4 \mu_t\|P(\mathcal{C}_t)\| \|\begin{pmatrix}
I \\
K(\mathcal{C}_t)
\end{pmatrix}^\top V^{-1}_t \begin{pmatrix}
I \\
K(\mathcal{C}_t) 
\end{pmatrix} \| I \label{eq:badnis}. 
\end{align}

By (\ref{eq: verygood}) of Lemma \ref{Stability_lemma18} which holds on the event $\mathcal{H}_t\cap \mathcal{E}_0(\bar{\epsilon}(\bar{\kappa}))$ (which holds with probability at least $1-5\delta$), we have

\begin{align}
    \mu_t  \| P(\mathcal{C}_t)\| \begin{pmatrix}
I \\
K(\mathcal{C}_t)
\end{pmatrix}^\top V^{-1}_t \begin{pmatrix}
I \\
K(\mathcal{C}_t) 
\end{pmatrix}&\preceq \frac{\alpha_0}{16 {\kappa_t}^6}I \label{eq: verygood00}
\end{align}
in which we applied the fact that
\begin{align*}
    \|\begin{pmatrix}
I \\
K(\mathcal{C}_t)
\end{pmatrix}^\top\begin{pmatrix}
I \\
K(\mathcal{C}_t) 
\end{pmatrix} \|\leq 1+\frac{2\|P(\mathcal{C}_t)\|}{\alpha_0}\leq \frac{4\|P(\mathcal{C}_t)\|}{\alpha_0}=2 \kappa_t^2
\end{align*}
as $2\|P(\mathcal{C}_t)\|>\alpha_0$.

Substituting (\ref{eq: verygood00}) into (\ref{eq:badnis}) we can write

  \begin{align}
     \nonumber  P_*- P(\mathcal{C}_t)&\preceq 4 \kappa^4 \mu_t \|P(\mathcal{C}_t)\| \|\begin{pmatrix}
I \\
K(\mathcal{C}_t)
\end{pmatrix}^\top V^{-1}_t \begin{pmatrix}
I \\
K(\mathcal{C}_t) 
\end{pmatrix} \| I\\
&\preceq 4\kappa_t^4 \frac{\alpha_0}{16 \kappa_t^6} I = \frac{\alpha_0 {\gamma_t}}{2} I. \label{eq:lemprfbound}
  \end{align}
 % where in the last inequality we used the fact that $\bar{\kappa}\geq \sqrt{2\|P_*\|/\alpha_0}$ which holds because $\|P_*\|\leq \nu/ \sigma_{\omega}^2$.
\end{proof}

Now, we provide the following lemma, which indeed demonstrates that the control designed by Algorithm \ref{Alg:ACOLC} satisfies property (3) in Definition \ref{def:sequentially}."
\begin{lemma} \label{lem:sequentiality} 
Let $H^2_t=P(\mathcal{C}_t)$ be the dual relaxed SDP solution. Then $\|H_{t+1}H_t\|\leq 1+\gamma_{t+1}/2$ under the event $\mathcal{H}_t\cap \mathcal{E}_0(\bar{\epsilon}(\bar{\kappa}_1))$.      
\end{lemma}
\begin{proof}
    The proof directly follows \cite{cohen2019learning}, but we have chosen to redo it, as we will require a slightly modified version of this lemma for the proof of Theorem \ref{prop2}.
Recalling definition of $H_t=P^{1/2}(\mathcal{C}_{t})$ we can write:
 \begin{align}
    \nonumber  \|H_{t+1}^{-1}H_t\|^2&=  \|P^{-\frac{1}{2}}(\mathcal{C}_{t+1})P^{\frac{1}{2}}(\mathcal{C}_{t})\|^2\\
     \nonumber & = \|P^{-\frac{1}{2}}(\mathcal{C}_{t+1})P(\mathcal{C}_{t})P^{-\frac{1}{2}}(\mathcal{C}_{t+1})\|\\
   &\leq \|P^{-\frac{1}{2}}(\mathcal{C}_{t+1})(P(\mathcal{C}_{t+1})+\frac{{\gamma_{t+1}} \alpha_0}{2} I)P^{-\frac{1}{2}}(\mathcal{C}_{t+1})\| \label{eq:komak12}\\
   \nonumber &= \|I+\frac{{\gamma_{t+1}}\alpha_0}{2}P^{-1}(\mathcal{C}_{t+1}) \| \\
&\nonumber  \leq 1+\frac{{\gamma_{t+1}}\alpha_0}{2}\|P^{-1}(\mathcal{C}_{t+1})\| \\
    & \nonumber \leq 1+{\gamma_{t+1}}
 \end{align}
 where (\ref{eq:komak12}) follows by Lemma \ref{lem:closeness} and the last inequality holds because $\|P^{-1}(\mathcal{C}_{t+1})\|\leq 2/\alpha_0$. This concludes that $\|H_{t+1}^{-1}H_t\|\leq \sqrt{1+{\gamma_{t+1}}}\leq 1+{\gamma_{t+1}}/2$.
\end{proof}

% \subsubsection{Proof of Lemma \ref{lemma:upperBound}}

\subsection{Proof of Proposition \ref{lemma:upperBound}}

\begin{proof}
  Proof follows by \cite{cohen2019learning} with an only difference that now we have an input perturbation as well. Let $M=A_*+B_*K_t$ where $K_t'$s are strong sequential stabilizing policies generated by ARSLO algorithm. The closed loop dynamics then is written as follows:
    \begin{align*}
        x_t=M_1x_1+\sum_{s=1}^{t-1}M_s(B_*\nu_{s}+\omega_{s+1}) 
    \end{align*}
   where 
   \begin{align*}
       M_t=I,\;\; M_{s}=M_{s+1}(A_*+B_*K_s)=\prod_{j=s}^{t-1}(A_*+B_*K_j), \; \forall 1\leq s\leq t-1. 
   \end{align*}
   Recalling Definition \ref{def:sequentially}, we can set $A_*+B_*K_j= H_j L_j H_j^{-1}$, then for all $1\leq s\leq t$ where $\|L_j\|\leq 1-\gamma_j$ and $\|H^{-1}_{j+1}H_{j}\|\leq 1+\gamma_{j+1}$ we can write
  \begin{align*}
     \nonumber  \|M_s\|\leq &\lVert (H_{t-1}L_{t-1}H^{-1}_{t-1})...(H_{s+1}L_{s+1}H^{-1}_{s+1}) (H_{s}L_{s}H^{-1}_{s})\rVert\\
     \leq &\|H_{t-1}\| \; \|(L_{t-1}H^{-1}_{t-1}H_{t-2})...(L_{s+2}H^{-1}_{s+2}H_{s+1})(L_{s+1}H^{-1}_{s+1}H_{s})\|\:\|L_s\|\:\|H_s\|^{-1}\\
    \nonumber  \leq & \|H_{t-1}\| \bigg(\prod_{j=s+1}^{t-1}\|L_j\|\; \|H_{j}^{-1}H_{j-1}\|\bigg) \|H_s^{-1}\|\\
    \leq &\sqrt{P(\mathcal{C}_{t-1})} \bigg(\prod_{j=s+1}^{t-1}(1-\gamma_j)\; (1+\gamma_j/2)\bigg) (1-\gamma_s)\sqrt{\frac{2}{\alpha_0}}\\
\leq  &\kappa_{{t-1}}\bigg(\prod_{j=s+1}^{t-1}(1-\gamma_j/2)\; \bigg) (1-\gamma_s/2)\\
   \leq &{\kappa_*} (1-{\gamma_*}/2)^{t-s} 
  \end{align*}
where $\|H_{j}^{-1} H_{j-1}\| \leq (1 + \gamma_j / 2)$ holds for all $j$ by the sequentiality property of the generated policies. 
We also applied the fact that $H_t = P^{1/2}(\mathcal{C}_t) \succeq \sqrt{\alpha_0 / 2}\, I$. 
Finally, we used the fact that $\kappa_t \leq \kappa_*$, as well as the inequality $(1 - \gamma_j / 2) \leq (1 - \gamma_* / 2)$, where the latter holds by definition of $\gamma_* := \kappa_*^{-2} / 2$.

Upper-bounding state norm yields
    \begin{align*}
        \nonumber \|x_t\|&\leq \|M_1\|\|x_1\|+\sum_{s=1}^{t-1}\|M_{s+1}\|\|B_*\eta_s+\omega_{s+1}\|\\
       \nonumber &\leq \kappa_* (1-\gamma_*/2)^{t-1} \|x_1\|+\kappa_*\sum_{s=1}^{t-1}(1-{\gamma_*}/2)^{t-s-1} \|B_*\eta_s+\omega_{s+1}\|  \\
       \nonumber &\leq  \kappa_* e^{-{\gamma_*} (t-1)/2}\|x_1\|+\frac{2{\kappa_*}}{{\gamma_*}} \max_{1\leq s\leq t}\|B_*\eta_s+\omega_{s+1}\|. 
    \end{align*}
    This bound holds on the event $\mathcal{H}_t$, which occurs with probability at least $1-3\delta$, provided that the event $\mathcal{E}_0(\bar{\epsilon}(\bar{\kappa}_1))$ holds.
This completes the proof of Proposition \ref{lemma:upperBound}.
\end{proof}

\begin{lemma} \label{lem:costatebound}
Provided with an initial estimate $\Theta_0$ satisfying (\ref{def:bareps}), under the event $\mathcal{H}_t$, the following bounds hold for all \(t \ge 1\).

First,
  \begin{align}
    \|x_t\|^2\leq 4\frac{\kappa_*^2}{\gamma_*^2}\big(\frac{\|x_1\|}{\sqrt{\log \frac{1}{\delta}}}+\sqrt{20n \sigma_{\omega}^2}\big)^2\log \frac{t}{\delta}:=X^2_t (\kappa_*).\label{eq:simpleBoundX}
 \end{align}
Recalling that
\begin{align}
   z_t:= \begin{pmatrix}
x_t \\
K(\mathcal{C}_t) x_t+\eta_t
\end{pmatrix} \label{eq:defztn}
\end{align}
we also have
    \begin{align}
        \|z_t\|^2\leq {64 \kappa_*^6(1+\kappa_*^2)}\bigg(\frac{\|x_1\|}{\log \frac{1}{\delta}}+\sqrt{20n \sigma_{\omega}^2}\bigg)^2\log \frac{t}{\delta}:=c_z^2(\kappa_*)\log \frac{t}{\delta}=:\zeta_t^2(\kappa_*).\label{eq:zetaValll}
    \end{align}
\end{lemma}

\begin{proof}

Recall the definition of the input perturbation noise. We can write
\begin{align*}
    B_*\eta_s+\omega_{s+1}\sim\mathcal{N}\bigg(0, B_*\frac{\bar{p}_t\sigma_{\omega}^2\big(K(\mathcal{C}_t) K^\top(\mathcal{C}_t)+\frac{\|P(\mathcal{C}_t)\|}{\alpha_0}\big)}{\sqrt{t+\bar{c}}} B_*^\top+\sigma_{\omega}^2I\bigg).  
\end{align*}
Now, with the choice of $\bar{c}$ in~(\ref{eq:barcARSLO}), which guarantees
\begin{align*}   B_*\frac{\bar{p}_t\sigma_{\omega}^2\big(K(\mathcal{C}_t) K^\top(\mathcal{C}_t)+\frac{\|P(\mathcal{C}_t)\|I}{\alpha_0}\big)}{\sqrt{t+\bar{c}}} B_*^\top\preceq \sigma_{\omega}^2I
\end{align*}
applying the Hanson-Wright concentration inequality yields
\begin{align*}
   \max_{1\leq s\leq t} \| B_*\eta_s+\omega_{s+1}\|\leq \sqrt{20 \sigma_{\omega}^2n\log \frac{t}{\delta}} 
\end{align*}
under the event $\mathcal{E}_t^3$. 
Consequently, by (\ref{eq:neatBun}), we have
\begin{align*}
    \|x_t\|&\leq \kappa_* e^{-{\gamma_*} (t-1)/2}\|x_1\|+\frac{2\kappa_*}{\gamma_*}\sqrt{20 \sigma_{\omega}^2n\log \frac{t}{\delta}}.
\end{align*}
Since $\mathcal{H}_t \subseteq \mathcal{E}_t^3$, the above bound holds on the event $\mathcal{H}_t$, provided that the initial estimate $\Theta_0$ satisfies \eqref{def:bareps}.
Furthermore,
\begin{align*}
    \|x_t\|^2\leq 4\frac{\kappa_*^2}{\gamma_*^2}\big(\frac{\|x_1\|}{\sqrt{\log \frac{1}{\delta}}}+\sqrt{20n \sigma_{\omega}^2}\big)^2\log \frac{t}{\delta}:=X_t .
\end{align*}

Considering (\ref{eq:defztn}),
we obtain
  \begin{align*}
        \|z_t\|^2\leq & 2(1+\kappa_*^2)\|x_t\|^2+2 \|\eta_t\|^2.
    \end{align*}
Furthermore, by the definition of $\bar{c}$ in~(\ref{eq:magncbars}), it holds that 
\begin{align*}
    \|\eta_t\|\leq \sqrt{10n \sigma_{\omega}^2 \log \frac{t}{\delta}}
\end{align*}
which in turn yields
\begin{align*}
        \|z_t\|^2\leq  \frac{16 \kappa^2(1+\kappa_*^2)}{\gamma_*^2}\bigg(\frac{\|x_1\|}{\log \frac{1}{\delta}}+\sqrt{20n \sigma_{\omega}^2}\bigg)^2\log \frac{t}{\delta}
\end{align*}
under the same event $\mathcal{H}_t\cap \mathcal{E}_0(\bar{\epsilon}(\bar{\kappa}_1))$.
\end{proof}

\subsection{Proof of Corollary \ref{cl:strongseq}}
\begin{proof}
The proof follows from the definition of $\lambda$ in (\ref{eq:defaARSLoth}), which has the same order as $a^2$, namely $\mathcal{O}(\bar{\kappa}_1^{28})$. Moreover, by the definition of $\bar{\kappa}_1$ in (\ref{eq:barkapseq_Def}), it is of the same order as $\kappa_1$, which is upper-bounded by $\kappa_*$. Finally, since $\bar{\epsilon}(\bar{\kappa}_1)$ defined in (\ref{def:bareps}) is of order $\mathcal{O}(1/\sqrt{\lambda})$, the claim follows.
\end{proof}

% \section{Proof of Theorem \ref{thm.parameter Estimate error}}
% \begin{proof}
%     The proof directly follows from the constructed confidence ellipsoid given by (\ref{eq:confSet1_tighterghfff2}) and (\ref{radius_centralEl_realTime20}), and the lower bound on the minimum eigenvalue of the covariance matrix $V_{\tau}$
%  given by Lemma \ref{eq:lowerbndcov}.
% \end{proof}

\section {Regret Bound Analysis of ARSLO Algorithm} \label{Sec:regretARSLOan}

\subsection{Proof of Lemma \ref{lem:epseventprob}}

\begin{proof}
Recall
\begin{align*}
\mathcal{E}_t^1 &= \left\{ \forall s = 1, \ldots, t,\; \Theta_* \in \mathcal{C}_s(\delta) \right\}
\end{align*}
and define
\begin{align*}
\mathcal{G}_t &= \left\{ \forall s = 1, \ldots, t,\; \|z_s\|^2 \le \zeta_t^2(\kappa_*) \right\}.
\end{align*}

Then the event $\mathcal{E}_t$ can be written as
\[
\mathcal{E}_t = \mathcal{E}_t^1 \cap \mathcal{G}_t \cap \mathcal{E}_0(\bar{\epsilon}).
\]

where for sake of brevity we use $\bar{\epsilon}$ instead of $\bar{\epsilon}(\bar{\kappa}_1)$.

Referring to Lemma \ref{lem:costatebound}, the event $\mathcal{G}_t$ holds whenever the event $\mathcal{H}_t \cap \mathcal{E}_0(\bar{\epsilon})$ holds, where $\mathcal{H}_t$ is defined in (\ref{eq:mathcH0}). Furthermore, since $\mathcal{H}_t\subseteq \mathcal{E}_t^1$, it follows that the event $\mathcal{E}_t$ holds whenever $\mathcal{H}_t \cap \mathcal{E}_0(\bar{\epsilon})$ holds.

We already know from (\ref{eq:probHceps0}) that
\begin{align}
    \mathbb{P}\big(\mathcal{H}_t \mid \mathcal{E}_0(\bar{\epsilon})\big) \ge 1-3\delta.
\end{align}

Hence,
\begin{align}
\mathbb{P}\big(\mathcal{E}_t \mid \mathcal{E}_0(\bar{\epsilon})\big) \ge 1 - 3\delta.
\end{align}

Finally, since $\mathbb{P}\big(\mathcal{E}_0(\bar{\epsilon})\big) \ge 1-2\delta$, we conclude that
\begin{align*}
\mathbb{P}(\mathcal{E}_t)
&= \mathbb{P}\big(\mathcal{E}_t \mid \mathcal{E}_0(\bar{\epsilon})\big)\mathbb{P}\big(\mathcal{E}_0(\bar{\epsilon})\big) \\
&\ge (1-3\delta)(1-2\delta) \\
&\ge 1-5\delta.
\end{align*}

This completes the proof.
\end{proof}

\subsection{Regret Decomposition} \label{eq:subsecDecomp}

We begin by decomposing the regret bound. The structure of the decomposition follows the approach of \cite{cohen2019learning}, with modifications necessitated by the presence of input perturbation. These changes slightly alter the form of the regret decomposition and introduce additional terms contributing to the final bound.

For the instantaneous regret we have:
\begin{align*}
    r_k 
= x_k^\top Q x_k + u_k^\top R u_k - J^\star
\leq& x_k^\top \big(Q +K^{\top}(\mathcal{C}_k) R K(\mathcal{C}_k)\big)x_k+2\eta_k^\top R K(\mathcal{C}_k) x_k+\eta^\top_k R \eta_k - \sigma_{\omega}^2 \|P_k\|_*.
\end{align*}
Here, we used the following fact about the average expected cost:
\begin{align*}
    J_*\geq \sigma_{\omega}^2\|P(\mathcal{C}_t)\|_*
\end{align*}
which follows from  $J_*\geq \sigma_{\omega}^2\|P_*\|_*$ together with
$\|P_*\|_*\geq \|P(\mathcal{C}_t)\|_*$. 

By Lemma \ref{lem:consoldualprim} we can write
\begin{align}
    \nonumber x_k^\top\big(Q+K^\top(\mathcal{C}_k) R K(\mathcal{C}_k)\big)x_k=&x_k^\top P(\mathcal{C}_k)x_k -x_k^\top(\hat{A}_k+\hat{B}_kK(\mathcal{C}_k))^\top  P(\mathcal{C}_k) (\hat{A}_k+\hat{B}_kK(\mathcal{C}_k))x_k
    \\
    &+\mu_{\tau(k)} \| P(\mathcal{C}_k)\|_*\bar{z}_k^\top V^{-1}_{\tau(k)} \bar{z}_k \label{eq:komskzkhsn}
\end{align} 
where
\begin{align*}
    \bar{z}_k:=\begin{pmatrix}
x_k \\
K(\mathcal{C}_k)x_k
\end{pmatrix}.
\end{align*}
Under the ARSLO algorithm, the policy update condition $\det(V_k) > (1+\beta) \det(V_{\tau(k)})$, where $\tau(k)$ denotes the last policy update time prior to $k$, implies that for all \(k\) in the same epoch, $A_k=A_{\tau(k)}$, $B_k=B_{\tau(k)}$, and $K(\mathcal{C}_k)=K(\mathcal{C}_{\tau(k)})$.

Using (\ref{eq:komskzkhsn}), we obtain
\begin{align}
  \nonumber  r_k\leq& x_k^\top P(\mathcal{C}_k)x_k -x_{k+1}^\top P(\mathcal{C}_k)x_{k+1}\\
  \nonumber  & +x_k^\top(A_*+B_*K(\mathcal{C}_k))^\top  P(\mathcal{C}_k) (A_*+B_*K(\mathcal{C}_k))x_k\\
   &-x_k^\top(\hat{A}_k+\hat{B}_kK(\mathcal{C}_k))^\top  P(\mathcal{C}_k) (\hat{A}_k+\hat{B}_kK(\mathcal{C}_k))x_k \label{eq:intrst}\\
   \nonumber &+2\omega_k^\top P(\mathcal{C}_k) (A_*+B_*K(\mathcal{C}_k))x_k+\omega_k^\top  P(\mathcal{C}_k)\omega_k-\sigma_{\omega}^2\|P(\mathcal{C}_t)\|_*+\mu_{\tau(k)} \| P(\mathcal{C}_k)\|\bar{z}_k^\top V^{-1}_{\tau(k)} \bar{z}_k\\
   &+2\eta_k^\top R K(\mathcal{C}_k) x_k+\eta^\top_k R \eta_k.\label{eq:additnterms}
\end{align}
The terms in \eqref{eq:additnterms} are explicit additional contributions arising from the input perturbation. Also,
in deriving \eqref{eq:intrst}, we simply added and subtracted
\begin{align*}
    x_{k+1}^\top P(\mathcal{C}_k)x_{k+1}=&x_k^\top(A_*+B_*K(\mathcal{C}_k))^\top  P(\mathcal{C}_k) (A_*+B_*K(\mathcal{C}_k))x_k\\
    &+2\omega_k^\top P(\mathcal{C}_k) (A_*+B_*K(\mathcal{C}_k))x_k+\omega_k^\top  P(\mathcal{C}_k)\omega_k.
\end{align*}
Next, applying Lemma~\ref{lem:purturbation} yields
 \begin{align*}
    \bar{z}_k^\top \bigg(\Theta_*P(\mathcal{C}_k)\Theta_*^\top-\hat{\Theta}_k P(\mathcal{C}_k)\hat{\Theta}^\top_k \bigg)\bar{z}_k \leq \mu_{\tau(k)} \| P(\mathcal{C}_k)\|\;\bar{z}_k^\top V^{-1}_{\tau(k)}\bar{z}_k.
 \end{align*}
 which implies
\begin{align*}
    &x_k^\top(A_*+B_*K(\mathcal{C}_k))^\top  P(\mathcal{C}_k) (A_*+B_*K(\mathcal{C}_k))x_k\\
   &-x_k^\top(\hat{A}_k+\hat{B}_kK(\mathcal{C}_k))^\top  P(\mathcal{C}_k) (\hat{A}_k+\hat{B}_kK(\mathcal{C}_k))x_k\leq \mu_{\tau(k)} \| P(\mathcal{C}_k)\|\;\bar{z}_k^\top V^{-1}_{\tau(k)}\bar{z}_k 
\end{align*}
providing an upper bound for \eqref{eq:intrst}.

Since \(V_{\tau(k)}\) is constructed using \(z_k\) and not \(\bar{z}_k\), and  
\begin{align*}
   \bar{z}_k=z_k-\begin{pmatrix}
0 \\
\eta_k
\end{pmatrix}.
\end{align*}
we obtain the structural bound
\begin{align*}
\bar{z}_k^\top V^{-1}_{\tau(k)}\bar{z}_k\leq 2z_k^\top V^{-1}_{\tau(k)}z_k+2 \eta_k^\top V^{-1}_{\tau(k)} \eta_k.
\end{align*}
Combining these results gives
\begin{align}
\nonumber r_k\leq& x_k^\top P(\mathcal{C}_k)x_k -x_{k+1}^\top P(\mathcal{C}_k)x_{k+1}\\
  \nonumber &+2\omega_k^\top P(\mathcal{C}_k) (A_*+B_*K(\mathcal{C}_k))x_k\\
 \nonumber  &+\omega_k^\top  P(\mathcal{C}_k)\omega_k-\sigma_{\omega}^2\|P(\mathcal{C}_t)\|_*\\
 & +4\mu_{\tau(k)} \| P(\mathcal{C}_k)\| z_k^\top V^{-1}_{\tau(k)}z_k\label{eq:ImptrmR}\\
\nonumber  &+2\eta_k^\top R K(\mathcal{C}_k) x_k\\
 \nonumber &+\eta^\top_k \big(R+4\mu_{\tau(k)} \| P(\mathcal{C}_k)\|V^{-1}_{\tau(k)}\big) \eta_k.
\end{align}
To further bound term \eqref{eq:ImptrmR}, note that \(\tau(k)\) is the last policy update time, and the algorithm ensures $\det(V_{k})\leq (1+\beta)\det(V_{\tau(k)})$. Consequently, $z_k^\top V^{-1}_{\tau(k)}z_k\leq (1+\beta)z_k^\top V^{-1}_kz_k$ and hence
\begin{align*}
   4\mu_{\tau(k)} \| P(\mathcal{C}_k)\| z_k^\top V^{-1}_{\tau(k)}z_k\leq 4(1+\beta)\mu_{\tau(k)} \| P(\mathcal{C}_k)\| z_k^\top V^{-1}_k z_k. 
\end{align*}

Finally, with definition of event $\mathcal{E}_k$ given by (\ref{eq:Godevent1}) the accumulated regret decomposes as
\begin{align*}
\tilde{R}_{\textit{ARSLO}}(t) \leq \sum_{j=1}^6 R_j(t)
\end{align*}
where:
 \begin{subequations}
    \allowdisplaybreaks
	\begin{align}
	&R_1(t)=\sum _{k=1}^{t} \big(x_{k}^\top P(\mathcal{C}_k)x_{k}-x_{k+1}^\top P(\mathcal{C}_k)x_{k+1}\big)1_{\mathcal{E}_k},\label{eq:R1} \\
  &R_2(t)=\sum _{k=1}^{t} w^\top_kP(\mathcal{C}_k)\big(A_*+B_*K(\mathcal{C}_k)\big)x_k 1_{\mathcal{E}_k},\label{eq:R2}\\
 & R_3(t) =\sum _{k=1}^{t} \big( w^\top_kP(\mathcal{C}_k)w_k-\sigma_{\omega}^2\|P(\mathcal{C}_k)\|_*\big)1_{\mathcal{E}_k},\label{eq:R3}\\
 &R_4(t) =\sum _{k=1}^{t} 4(1+\beta)\|P(\mathcal{C}_k)\|\mu_{\tau(k)}\big(z^\top_kV_{k}^{-1}z_k\big)1_{\mathcal{E}_k`},\label{eq:R4}  \\
	&  R_5(t) =\sum _{k=1}^{t} 2\eta_k^\top R K(\mathcal{C}_k) x_k  1_{\mathcal{E}_k},\label{eq:R5}  \\
    & R_6(t) =\sum _{k=1}^{t} \eta^\top_k \big(R+4\mu_{\tau(k)} \| P(\mathcal{C}_k)\|V^{-1}_{\tau(k)}\big) \eta_k 1_{\mathcal{E}_k}.\label{eq:R6}  
	\end{align}
\end{subequations}

\subsection{Proof of Theorem \ref{thm:RegretBound}} \label{Sec:ProofRegARSLO}

\begin{proof}

To derive the regret bound, we upper-bound each term in (\ref{eq:R1}-\ref{eq:R6}) separately, setting $\beta = 1$ to remain consistent with the ARSLO algorithm (Algorithm~\ref{Alg:ACOLC}).
Throughout the regret analysis, we make use of the norm upper bounds established in Lemma~\ref{lem:costatebound}.

% \textcolor{blue}{\begin{lemma} \label{lem:StatebondReg}
% With probability at least $1-\delta$ we have
%    \begin{align}
%         \|z_t\|^2\leq  \frac{16 \kappa_*^2(1+\kappa_*^2)}{\gamma_*^2}\bigg(\frac{\|x_1\|}{\log \frac{1}{\delta}}+\sqrt{20n \sigma_{\omega}^2}\bigg)^2\log \frac{t}{\delta}:=\zeta_t^{*^2} \label{eq:zetaStae}
%     \end{align}
%   Furthermore,  
%     \begin{align}
%     \|x_t\|^2\leq 4\frac{\kappa_*^2}{\gamma_*^2}\big(\frac{\|x_1\|}{\sqrt{\log \frac{1}{\delta}}}+\sqrt{20n \sigma_{\omega}^2}\big)^2\log \frac{t}{\delta}:=X_t^{*^2} \label{eq:Xsttae}.
% \end{align} 
% \end{lemma}}

% \begin{proof}
%    The upper bounds follow directly from (\ref{eq:zetaValll}), (\ref{eq:simpleBoundX}), and the fact that $\kappa_* \geq \kappa$.
% \end{proof}

\begin{lemma} \label{Thm.BoundR1}
The following upper bound holds for $|R_1(t)|$:
	\begin{align*}
	|R_1(t)|\lesssim \|P_*\|^4 n(n+m)\log_2^2 \frac{nt}{\delta}.
	\end{align*}	
\end{lemma}

\begin{proof}
We begin by rewriting \(R_1(t)\) as
	\begin{align}
	\nonumber R_1(t)=&\sum_{k=1}^{t}(x^\top_kP(\mathcal{C}_k)x_k-x^\top_{k+1}P(\mathcal{C}_k)x_{k+1})1_{\mathcal{E}_k}\\
	\nonumber=&\sum_{k=2}^{t-1}(x^\top_k(P(\mathcal{C}_k)-P(\mathcal{C}_{k-1})x_k) 1_{\mathcal{E}_k}\\
   \nonumber  &+x_1^\top P(\hat{\Theta}_1)x_1-x_{t+1}^\top P(\mathcal{C}_t)x_{t+1}.
	\end{align}
The term \(x_k^\top (P(\mathcal{C}_k)-P(\mathcal{C}_{k-1})) x_k\) is nonzero only when the policy is updated. Thus, to upper-bound \(R_1\), it suffices to bound the total number of policy updates. Let \(N(t)\) denote the number of updates up to time \(t\).

 Recalling the policy-update rule with \(\beta = 1\), we have
	\begin{align*}
		\det(V_{\tau(t)})\geq 2^{N(t)}\det(\lambda I) 
	\end{align*}
where \(\tau(t)\) is the last update time prior to \(t\).

This yields
\begin{align*}
    N(t)\leq\log_2 (\frac{\det(V_{\tau(t)})}{\det(\lambda)})\leq \log_2 (\frac{\det(V_{t})}{\det(\lambda)}).
\end{align*}
On the event $\mathcal{E}_t$ and from the definition of $\lambda$ in (\ref{eq:lamchosens}) and similar to what we carried out in (\ref{eq:seqpf1}) we have
\begin{align*}
   \log_2 (\frac{\det(V_{t})}{\det(\lambda)})\leq  2(n+m)\log_2 \frac{2t}{\delta}. 
\end{align*}
This yields
\begin{align*}
   N(t) \leq & 2(n+m)\log_2 \frac{2t}{\delta}.
\end{align*}
Using \(\|P(\mathcal{C}_k)\|\le \|P_*\|\) and \(X_t(\kappa_*) := \max_{k\le t}\|x_k\|\), which holds on the event  $\mathcal{E}_t$, we obtain
\begin{align*}
	|R_1|\leq \|P_*\|X_t^{2}(\kappa_*)+ 2\|P_*\|(n+m) X_t^{2}(\kappa_*) \log_2\frac{2t}{\delta}
	\end{align*}	
Finally, incorporating the dependence of \(X_t^2(\kappa_*)\) on \(\|P_*\|\) and \(n\) yields the stated bound, completing the proof.
\end{proof}

\begin{lemma} \label{R2}
For any $\delta\in(0,1)$, with probability at least $1 - \delta/4$, the following bound holds:
	\begin{align*}
	|R_2(t)|\lesssim \sqrt{\|P_*\|^5 n\, t\,\log^3 \frac{t}{\delta}}. 
	\end{align*}	
\end{lemma}
\begin{proof}
Applying Lemma~30 from \cite{cohen2019learning}, we obtain
\begin{align*}
    R_2(t)=\sum_{k=1}^t \underbrace{x^\top_k\big(A_*+B_*K(\mathcal{C}_k)\big)^\top P(\mathcal{C}_k)1_{\mathcal{E}_k}}_{\bar{v}^\top_k}w_k\leq 2\sigma_{\omega}\bar{D}\sqrt{\frac{8}{\delta}},
\end{align*}
which holds with probability at least $1-\delta/4$, where $\sum_{k=1}^t \bar{v}_k^\top \bar{v}_k \leq \bar{D}^2$. 

We now derive an upper bound on this quantity. Observe that
\begin{align*}
    \sum_{k=1}^t \bar{v}_k^\top \bar{v}_k&\leq \|P_*\|^2 \sum_{k=1}^t x_k^\top\big(A_*+B_*K(\mathcal{C}_k)\big)^\top \big(A_*+B_*K(\mathcal{C}_k)\big) x_k\,1_{\mathcal{E}_k}\\
    & \leq \|P_*\|^2(1-\gamma_*)^2\sum_{k=1}^t x_k^\top x_k\,1_{\mathcal{E}_k}\\
    &\leq \|P_*\|^2\, X_t^2(\kappa_*)\,t:=\bar{D}^2
\end{align*}
where under the events $\mathcal{E}_k$ $\forall k$ we used the bounds $\|P(\mathcal{C}_t)\|\leq \|P_*\|$ and $\|A_*+B_*K(\mathcal{C}_t)\|\leq 1-\gamma_*< 1$. Hence,

	\begin{align*}
	|R_2(t)|\leq 2\sigma_{\omega}\|P_*\|X_t(\kappa_*)\sqrt{t\frac{2}{\delta}}. 
	\end{align*}

Finally, substituting the expression for $X_t(\kappa_*)$ from (\ref{eq:simpleBoundX}), under the event $\mathcal{E}_t$ yields the desired result, completing the proof.
\end{proof}

\begin{lemma}\label{thm:boundR3}
	For any $\delta\in(0,1)$, the term $R_3(t)$ is bounded as follows: 
	\begin{align*}
	\nonumber R_3(3)\lesssim \|P_*\|_*\sqrt{t \log^3 \frac{t}{\delta}}
	\end{align*}
    with probability at least $1-\delta/4$.
\end{lemma}

\begin{proof}
The proof follows directly from \cite{cohen2019learning}; for completeness, we provide it here.

Let
\begin{align*}
    X_k := \omega_k^\top P(\mathcal{C}_k) \omega_k - \sigma_{\omega}^2 \|P(\mathcal{C}_k)\|_*, \quad k = 1, \dots, t.
\end{align*}

Since $P(\mathcal{C}_k)$ is $\mathcal F_{k-1}$-measurable and $\omega_k$ is independent of
$\mathcal F_{k-1}$ with $\mathbb E[\omega_k \omega_k^\top] = \sigma_{\omega}^2 I$, we have
\begin{align*}
    \mathbb E[X_k \mid \mathcal F_{k-1}] = 0.
\end{align*}
So $X_k$'s forms a martingale difference sequence. Moreover, since $P(\mathcal{C}_k) \succeq 0$ and under the event $\mathcal{E}_k$, $\|P(\mathcal{C}_k)\|_* \le \|P_*\|_*$ we have
\begin{align*}
    X_k \geq -\sigma^2 \|P_*\|_* \quad \text{a.s.}
\end{align*}
However, $X_k$ is not almost surely bounded from above, which prevents a direct application of Azuma's inequality. To address this, on the event $\mathcal{E}_k$, we define the truncated random variable
\begin{align*}
    \widetilde X_k := X_k \cdot \mathbf 1\{ X_k \le \Gamma \}, \qquad
\Gamma := 5 \|P_*\|_* \sigma_{\omega}^2 \log \frac{8t}{\delta}.
\end{align*}
Then $-\sigma_{\omega}^2 \|P_*\|_* \leq \widetilde X_k \le \Gamma$, allowing Azuma's inequality to apply.

Let $\widetilde Y_k := \widetilde X_k - \mathbb E[\widetilde X_k \mid \mathcal F_{k-1}]$.
Then $|\widetilde Y_k| \le \Gamma$ a.s. By Azuma’s inequality, Lemma \ref{lem:Azuma} , we, with probability at least $1-\delta/8$ , can write
\begin{align*}
    \sum_{k=1}^t\widetilde Y_k\leq \Gamma \sqrt{2t \log \frac{8}{\delta}}
\end{align*}
 which implies 
\begin{align*}
    \sum_{k=1}^t \widetilde X_k\leq  \sum_{k=1}^t \mathbb E[\widetilde X_k]1_{\mathcal{E}_k}+8\sigma_{\omega}^2\|P_*\|_* \sqrt{t \log \frac{8}{\delta}\log^2 \frac{8t}{\delta}}.
\end{align*}
Next, we control the truncation probability using the Hanson--Wright inequality, Lemma \ref{lem:HansonWright}.

For each $k$, it implies that for $\delta' \in (0,1/e)$:
\begin{align*}
    \Pr\big( w_k^\top P(\mathcal{C}_k) w_k \geq \sigma_{\omega}^2 \|P(\mathcal{C}_k)\|_* + 4 \sigma^2 \|P^{1/2}(\mathcal{C}_k)\| \,\|P^{1/2}(\mathcal{C}_k)\|_F \log(1/\delta') \big) \le \delta'.
\end{align*}
Using the facts that $\|P^{1/2}(\mathcal{C}_k)\|_F \|P^{1/2}(\mathcal{C}_k)\| \leq \|P(\mathcal{C}_k)\|_*$ and $\|P(\mathcal{C}_k)\|_*\leq \|P_*\|_*$ we get
\begin{align*}
    w_k^\top P(\mathcal{C}_k) w_k \leq \sigma_{\omega}^2 \|P(\mathcal{C}_k)\|_* + 4 \sigma_{\omega}^2 \|P(\mathcal{C}_k)\|_* \log(1/\delta')\leq 5 \sigma_{\omega}^2 \|P_*\|_* \log(1/\delta').
\end{align*}
Choosing $\delta' = \delta/(8t)$, we obtain
\begin{align*}
    \Pr(X_k > \Gamma) \leq \frac{\delta}{8t}.
\end{align*}
By the union bound, with probability at least $1-\delta/8$, $X_k \le \Gamma$ for all $k$, i.e., $\widetilde X_k = X_k$.
 
Azuma inequality holds for $\widetilde X_k$ with probability $1-\delta/8$, and
no-truncation event holds with probability $1-\delta/8$. By union bound, both hold
simultaneously with probability at least $1-\delta/4$. On this event:
\begin{align*}
    \sum_{k=1}^t X_k = \sum_{k=1}^t \widetilde X_k \, \leq 8\sigma_{\omega}^2\|P_*\|_* \sqrt{t \log \frac{8}{\delta}\log^2 \frac{8t}{\delta}}.
\end{align*}
This implies
\begin{align*}
    \sum_{k=1}^t \big(\omega_k^\top P_k \omega_k - \sigma_{\omega}^2 \|P_k\|_*\big)1_{\mathcal{E}_k}\leq \|P_*\|_*\sqrt{t \log \frac{1}{\delta}\log^2 \frac{4t}{\delta}}
\end{align*}
with probability at least $1-\delta/4$ which completes the proof.
\end{proof}

\begin{lemma}\label{thm:boundR4}
	The term $R_4(t)$ has the following upper bound: 

	\begin{align*}
   R_4(t)\lesssim \vartheta \sqrt{n^2(n+m)^3\|P_*\|^6\, t\, \log^4\frac{nt}{\delta}}.
\end{align*}
% \begin{align*}
% \bar{r}:= 8\sigma_{\omega}^2n \bigg(&2\log\frac{n}{\delta}+(1+2\bar{\Upsilon})\log T+\\
% &(m-1)\log(1+2\bar{\Upsilon})T\bigg)+2\epsilon^2 \lambda 
%  	\end{align*}	
\end{lemma}

\begin{proof} 
Recall that $\tau(t)$ denotes the most recent policy update time prior to time $t$.  
Using the definition  $\mu_t = r_t + 2\vartheta \sqrt{r_t} |V_t|^{1/2}$ and setting $\beta = 1$ to align with the ARSLO algorithm, we can write
	\begin{align}
	\nonumber R_4(t):=&4\sum _{k=1}^{t} (1+\beta)\|P(\mathcal{C}_k)\|\mu_{\tau(k)}\big({z}^\top_kV_{k}^{-1}{z}_k\big)1_{\mathcal{E}_k`}\\
    \nonumber \leq & 8 \mu_{t} \|P_*\|\sum _{k=1}^{t} \big({{z}_k}^\top{V_k}^{-1}{{z}_k}\big)1_{\mathcal{E}_k} 
	\end{align}
where the inequality uses $\mu_{\tau(t)} \leq \mu_t$ and $\|P(\mathcal{C}_k)\|\leq \|P_*\|$ for all $k\geq 1$ on the event $\mathcal{E}_k$. 

We decompose $\mu_t$ as follows:
\begin{align*}
\mu_{t}=(r_t+ 2 \vartheta \sqrt{r_t}\|V_t\|^{1/2})\leq& \bigg(r_t+ 2 \vartheta\, \sqrt{r_t}\sqrt{\lambda+t\zeta_t^2(\kappa_*)}\bigg)\\
 \leq &r_t+2\vartheta \sqrt{r_t\,\lambda}+2\vartheta \sqrt{r_t\, \zeta_t^2(\kappa_*)\,t}.
\end{align*}
Moreover, by using (\ref{eq:barrDefb}) as an upper bound for $r_t$ and applying (\ref{eq:zetaValll}), we obtain the following upper bounds for the aforementioned terms on the $\mathcal{E}_t$:
\begin{align*}
    \mu_t\leq& 16\sigma_{\omega}^2n(n+m)\log \frac{2nt}{\delta}\\
    &+ 8\vartheta \sigma_{\omega}^2\sqrt{\frac{\bar{c}n(n+m)}{10}\log \frac{2nt}{\delta}}\\
    &+8\vartheta c_z(\kappa_*)\sigma_{\omega}\sqrt{n(n+m)\,t\, \log \frac{t}{\delta}\, \log \frac{2nt}{\delta}}\\
   \lesssim &\vartheta \sqrt{n^2(n+m)\|P_*\|^4 t \log^2\frac{nt}{\delta}}.
\end{align*}

We observe that 
\begin{align}
    {z_k}^\top{V_k}^{-1}{z_k}\,1_{\mathcal{E}_k}\leq \frac{z_k^\top z_k}{\lambda+\frac{\bar{p}_t\sigma_{\omega}^2 \sqrt{k+\bar{c}}}{40}-\bar{C}}\leq \frac{\zeta_k^2(\kappa_*)}{\frac{\bar{p}_k\sigma_{\omega}^2 \sqrt{k+\bar{c}}}{40}}<1. \label{eq:reqCon}
\end{align}
The second inequality follows from the specific choice of $\bar{c}$ that satisfies (\ref{eq:goodcbarfinds}).

Since the condition (\ref{eq:reqCon}) holds, we may apply Lemma \ref{lem:login}. Consider an epoch within the interval  $[\tau_i,\; \tau_{i+1})$. Over this epoch
\begin{align*}
    \sum_{t=\tau_i}^{\tau_{i+1}}z_t^\top V_t^{-1}z_t\, 1_{\mathcal{E}_k}\leq 2 \log \frac{\det(V_{\tau_{i+1}})}{\det(V_{\tau_i})}.\label{eq:epochlemap}
\end{align*} 

Applying this bound across all epochs starting at $\tau_1,..., \tau_{N(t)}$, and letting $N(t)$ denote the total number of policy updates prior to time $t$, we can write
\begin{align}
   \nonumber  \sum_{t=1}^{t-1}z_t^\top V_t^{-1}z_t\, 1_{\mathcal{E}_k}&\leq \sum_{i=1}^{N(t)-1} 2 \log \frac{\det(V_{\tau_{i+1}})}{\det(V_{\tau_i})}+  2 \log \frac{\det(V_{t})}{\det(V_{\tau_{N(t)}})}\\
 &\leq 2 \log \frac{\det(V_{t})}{\det(\lambda I)}
\end{align}
where the last inequality follows by telescoping the logarithmic terms. 

On the event $\mathcal{E}_t$, and by the definition of $\lambda$, we also have 
\begin{align*}
   \log \frac{\det(V_{t})}{\det(\lambda I)}\leq 2(n+m)\log \frac{2nt}{\delta}\lesssim (n+m)\log \frac{nt}{\delta}
\end{align*}
as shown in (\ref{eq:seqpf1}).
Combining the above bounds yields
\begin{align*}
    R_4(t)\lesssim \vartheta \sqrt{n^2(n+m)^3\|P_*\|^6 t \log^4\frac{nt}{\delta}}.
\end{align*}
\end{proof}	
\begin{lemma}\label{thm:boundR5}
	Fix $\delta\in (0,1)$. Then the term $R_5$ has the following upper bound
  \begin{align*}
    R_5(t)\lesssim \alpha_1 \big(n^4(n+m)\big)^{1/4} \, t^{1/4}\,\sqrt{\vartheta \|P_*\|_* \|P_*\|^{11} \log^3 \frac{nt}{\delta}} 
\end{align*}
with probability at least $1-\delta/4$. 
\end{lemma}
\begin{proof}
  Define
\begin{align*}
    \bar{v}_k^\top = 2x_k^\top K^\top(\mathcal{C}_t) R^\top. 
\end{align*}
Then 
\begin{align*}
    R_5(t)=\sum _{k=1}^{t}\bar{v}_k^\top\eta_k 1_{\mathcal{E}_k}.
\end{align*}
Let $Y_k := \bar{v}_k^\top \eta_k$. Conditioned on all measurements up to time $k$, $Y_k$  is a zero-mean Gaussian random variable. We may write $Y_k = m_k\, f_k$, where $f_k \sim \mathcal{N}(0,1)$ and $\textit{Var}(Y_k)=m_k^2$. 

Applying Theorem 23 of \cite{cohen2019learning} with $V=\beta$ yields
\begin{align}
    \frac{(\sum_{k=1}^t f_k m_k)^2}{\beta + \sum_{k=1}^t m_k^2} \leq 2 \log \frac{1 + \beta^{-1} \sum_{k=1}^t m_k^2}{\delta} .\label{eq:useful342}
\end{align}
We now upper-bound the quadratic term on the events $\mathcal{E}_k$,
\begin{align}
   \nonumber \sum_{k=1}^tm_k^2&\leq \sum_{k=1}^t \|\tilde{v}_k\|^2 \|\Gamma_k\|^2\leq \sum_{k=1}^t \|\bar{v}_k\|^2\frac{2\sigma_{\omega}^2\bar{p}_k \kappa_*^2}{\sqrt{k+\bar{c}}}\\
  \nonumber &\leq 2\big(\max_{1\leq k\leq t}\|\bar{v}_k\|^2\big)\bar{p}_t\sigma_{\omega}^2\kappa_*^2 \sum_{k=1}^t \frac{1}{\sqrt{k+\bar{c}}} \\
   &\leq 8\alpha^2_1 \bar{p}_t\sigma_{\omega}^2\kappa_*^4 X_{t}^2(\kappa_*) \sqrt{t+\bar{c}}=:D. \label{eq:obboundtr}
\end{align}

Substituting the bound \eqref{eq:obboundtr} into \eqref{eq:useful342} and choosing $\beta=D$, we obtain, with probability at least $1-\delta/4$

\begin{align*}
    \underbrace{(\sum_{k=1}\bar{v}_k^\top \eta_k\, 1_{\mathcal{E}_k})^2}_{R_5^2(t)}=(\sum_{k=1}^t f_k m_k)^2\leq 4 D \log \frac{8}{\delta} .
\end{align*}
Equivalently,
\begin{align}
    R_5\leq 4\sqrt{2\bar{p}_t \log \frac{8}{\delta}} \alpha_1\sigma_{\omega}\kappa_*^2X_t(\kappa_*) (t+\bar{c})^{1/4} \label{eq:arslobndq1}.
\end{align}
From (\ref{eq:candbarptps}), on the event $\mathcal{E}_t$, $\bar{p}_t$ satisfies the upper-bound
\begin{align}
    \bar{p}_t\lesssim \vartheta \|P_*\|^{7}\sqrt{n^2(n+m)}\log \frac{nt}{\delta} \label{eq:barpbnd}
\end{align}
and for $X^2_t(\kappa_*)$ we have
\begin{align*}
   X^2_t(\kappa_*)\lesssim \|P_*\|^3 n \log \frac{t}{\delta}.
\end{align*}
Substituting these expressions into \eqref{eq:arslobndq1} gives
\begin{align*}
    R_5(t)\lesssim \alpha_1 (n^4(n+m))^{1/4} \sqrt{\vartheta  \|P_*\|^{12} \log^3 \frac{nt}{\delta}} (t+\bar{c})^{1/4}.
\end{align*}
which holds with probability at least $1-\delta/4$.
\end{proof}

\begin{lemma}\label{thm:boundR6}
	Fix $\delta\in(0,1)$. Then the term $R_6(t)$ has the following upper bound.	
\begin{align*}
     R_6(t)\lesssim \alpha_1 \vartheta \sqrt{n^2(n+m)m^2 \|P_*\|^{16} t \log^4 \frac{nt}{\delta}}
\end{align*}
with probability at least $1-\delta/4$.
\end{lemma}
\begin{proof}
The following inequality holds
\begin{align}
   \nonumber R_6(t) =&\sum _{k=1}^{t} \eta^\top_k \big(R+4\mu_{\tau(k)} \| P(\mathcal{C}_k)\|V^{-1}_{\tau(k)}\big) \eta_k 1_{\mathcal{E}_k}\\
   \nonumber \leq& \sum _{k=1}^{t} \eta^\top_k \big(R+\alpha_0I\big) \eta_k 1_{\mathcal{E}_k}\\
    \leq& 2\alpha_1 \sum _{k=1}^{t} \eta^\top_k\eta_k 1_{\mathcal{E}_k}. \label{eq:etaetaG}
\end{align}
The second inequality follows from (\ref{eq: verygood22}), namely
\begin{align*}
    \mu_t  \| P(\mathcal{C}_t)\| V_t^{-1}\preceq \frac{\alpha_0}{4}I \quad \forall t 
\end{align*}
which is ensured on the events $\mathcal{E}_t$ through the appropriate tuning of $\lambda$, $\bar{p}_t$, and $\bar{\epsilon}$, as followed by the proof of Lemma~\ref{Stability_lemma18}.

Then problem is reduced to upper-bound (\ref{eq:etaetaG}).

We can re-express
\begin{align*}
   \eta_k=\Gamma_k^{1/2}z_k \quad z_k\sim \mathcal{N}(0, I_m).
\end{align*}
Then 
\begin{align*}
    \sum_{k=1}^t \eta_k^\top \eta_k=\sum_{k=1}^t z_k^\top \Gamma_kz_k.
\end{align*}

Now by applying Hanson--Wright inequality (Lemma \ref{lem:HansonWright}), it yields
\begin{align*}
    z_k^\top \Gamma_k z_k\leq &\|\Gamma^{1/2}_k\|^2_F+2\|\Gamma^{1/2}\|_F\|\Gamma_k^{1/2}\|\sqrt{\log\frac{4k}{\delta}}+2\|\Gamma_k^{1/2}\|^2\log\frac{4k}{\delta}\\
    \leq& 5\|\Gamma^{1/2}_k\|^2_F\log\frac{4k}{\delta}:=5\|\Gamma_k\|_*\log\frac{4k}{\delta}
\end{align*}
with probability $1-\delta/4k$.

Hence on $\mathcal{E}_t$ we have
\begin{align*}
    \sum_{k=1}^t  z_k^\top \Gamma_k z_k\,1_{\mathcal{E}_k}&\leq 5\sum_{k=1}^t\|\Gamma_k\|_*\log\frac{k}{\delta}\leq 10n \bar{p}_t \sigma_{\omega}^2 \log \frac{t}{\delta}\kappa_*^2\sum_{k=1}^t \frac{1}{\sqrt{k+\bar{c}}}\\
    &\leq 20\,m\,\bar{p}_t\, \sigma_{\omega}^2\kappa_*^2\,\log \frac{t}{\delta} \big(\sqrt{t+\bar{c}}-\sqrt{\bar{c}}\big)
\end{align*}
in which we applied the fact that $\|\Gamma_k\|_*\leq m\|\Gamma_k\|$.

Finally applying (\ref{eq:barpbnd}) and definition of $\kappa_*$ yields
\begin{align*}
    R_6(t)\lesssim \alpha_1 \vartheta \sqrt{n^2(n+m)m^2 \|P_*\|^{16} t \log^4 \frac{nt}{\delta}}.
\end{align*}
with probability at least $1-\delta/4$.
\end{proof}

Considering all bounded terms, we observe that the dominant contributions to the regret are $R_4(t)$ and $R_6(t)$. Among these, $R_6(t)$ exhibits the worst dependence on $\|P_*\|$, while $R_4(t)$ yields the largest dimensional factor, namely $\mathcal{O}\!\left(\sqrt{n^2(n+m)^3}\right)$ compared to $\mathcal{O}\!\left(\sqrt{n^2(n+m)m^2}\right)$ for $R_6(t)$. Taking the more conservative dependence in both quantities, we obtain
\begin{align*}
    \tilde{R}_{\mathrm{ARSLO}}(t)
    \leq
    \mathcal{O}\!\left(
    \sqrt{
        n^2 (n+m)^3
        \|P_*\|^{16}
        \, t
        \log^4\!\frac{t}{\delta}
    }
    \right)
\end{align*}
with probability at least $1-\delta$.

Finally, the above bound is established on the event $\mathcal{E}_t$, which holds with probability at least $1-5\delta$. Therefore, applying a union bound over the failure events associated with $\mathcal{E}_t$ and the regret analysis yields
\begin{align*}
    R_{\mathrm{ARSLO}}(t)
    \leq
    \mathcal{O}\!\left(
    \sqrt{
        n^2 (n+m)^3
        \|P_*\|^{16}
        \, t
        \log^4\!\frac{t}{\delta}
    }
    \right)
\end{align*}
with probability at least $1-6\delta$.

\end{proof}

% \section{Proof of Theorem \ref{thm.parameter Estimate error}}
% \begin{proof}
%     The proof directly follows from the constructed confidence ellipsoid given by (\ref{eq:confSet1_tighterghfff2}) and (\ref{radius_centralEl_realTime20}), and the lower bound on the minimum eigenvalue of the covariance matrix $V_{\tau}$
%  given by Lemma \ref{eq:lowerbndcov}.
% \end{proof}

\newpage
\section{ARSLO$^+(\bar{\rho})$}

In this section, we introduce the ARSLO$^+(\bar{\rho})$ algorithm ( Algorithm \ref{Alg:ACOLC+}), which, by eliminating the sequentiality requirement for generated policies, is proven to outperform the previously proposed ARSLO algorithm. While the overall structure of ARSLO$^+(\bar{\rho})$ is similar to that of ARSLO, the main differences lie in the policy update criterion, which is carefully designed to ensure closed-loop system stability and to achieve an improved regret upper bound. Additionally, the adjustments of $\bar{p}_t$, $\lambda$, and the initial estimate $\Theta_0$ differ from those in ARSLO. The exploration--exploitation trade-off parameter $\beta(\bar{\rho})$, defined in~(\ref{eq:choicebeta}), can be rewritten as~(\ref{eq:redefbetarho}) using the definition of $\bar{\kappa}_1$.

\RestyleAlgo{ruled} 
\begin{algorithm}[H]
\caption{Any-time Regret SDP-based LQ cOntrol-Plus (ARSLO$^+(\bar{\rho})$)}\label{Alg:ACOLC+}
\textbf{Inputs:} {$\Theta_0$, $\bar{\kappa}_1$, and $\underline{\epsilon}(\bar{\kappa}_1, \bar{\rho})$ for some $\bar{\rho}\in[0,\,8)$ (provided by Algorithm \ref{alg:warmUp}), $\delta$,  $\vartheta$}

Compute $\bar{c}$ using~(\ref{eq:cfARSLOplus}).

Set $\lambda = \frac{\sigma_{\omega}^2 \bar{c}}{40} (n+m) \log \frac{1}{\delta}$

Initialize $V_1 = \lambda I$, $r_1 = \lambda \underline{\epsilon}^2(\bar{\kappa}_1, \bar{\rho})$, $\hat{\Theta}_1 = \Theta_0$, and construct $\mathcal{C}_1$.

Set \begin{align}
    \beta(\bar{\rho}) =\begin{cases}
        2\bar{\kappa}_1^{\frac{(6-\bar{\rho})}{2}} & \text{if } \bar{\rho}<6, \\
        2 & 6\leq\bar{\rho}<8
    \end{cases} . \label{eq:redefbetarho}
\end{align}

\For{$t = 1, 2, ... $}{
  \eIf{$\det (V_t)> (1+\beta(\bar{\rho}))\det (V_{\tau})$ or $t=1$}{
  Calculate  $\mu_t= r_t+\sqrt{r_t}\vartheta \|V_{t}\|^{1/2}$

   Solve the relaxed primal SDP problem (\ref{eq:RelaxedSDP}) and compute $K(\mathcal{C}_t)$ by (\ref{eq:obtPol})

 Compute $P(\mathcal{C}_t)$ through solving the relaxed dual SDP problem (\ref{eq:RedSDP_DUALP})

 Let $V_{\tau}=V_t$}{Set $K(\mathcal{C}_t)=K(\mathcal{C}_{t-1})$
  and $P(\mathcal{C}_t)=P(\mathcal{C}_{t-1})$}
  Calculate $\Gamma_t$ by (\ref{eq:Gamarslop})-(\ref{eq:pylop}) and sample $\eta_t\sim\mathcal{N}\big(0, \Gamma_{t}\big)$
  
  Play $u_t=K(\mathcal{C}_t)x_t+\eta_t$
  
  Observe $x_{t+1}$, save $(x_{t+1}, z_t)$ to the data set

  Construct $X_t$, $Z_t$ and update $V_t$

  Calculate $\hat{\Theta}_t$ by (\ref{eq:LSE_Sol123}) and $r_t$ by (\ref{radius_centralEl_realTime200}) and build confidence ellipsoid $\mathcal{C}_t$
}
\end{algorithm}

\section {Stability Analysis of ARSLO$^+(\bar{\rho})$\label{se:STARSLO+}}
\subsection{Proof of Theorem \ref{Stability_thm200}}
\begin{proof}
The purpose of this proof is to tune the parameters $\epsilon$, $\lambda$, and $\bar{p}_t$ so that the ARSLO$^{+}(\bar{\rho})$ algorithm generates policies that are individually $(\kappa_*,\gamma_*)$-strongly stabilizing while preserving the desired regret guarantees.

From the proof of the previous lemma, we concluded that ensuring a policy is $(\kappa_*,\gamma_*)$-strongly stabilizing requires
\begin{align}
    \mu_t \,\|P(\mathcal{C}_t)\| \,V_t^{-1} \preceq \frac{\alpha_0}{4}I. 
    \label{eq:juststab0}
\end{align}
 
To simultaneously control stability and regret, we impose a slightly stronger requirement:
\begin{align}
    \mu_t \,\|P(\mathcal{C}_t)\| \,V_t^{-1} \preceq 
    \frac{\alpha_0}{4 \kappa_t^{\bar{\rho}}}I, 
    \label{eq:verygood01-}
\end{align}
where $\kappa_t=\sqrt{\frac{2\|P(\mathcal{C}_t)\|}{\alpha_0}}>1$ and $\bar{\rho}\ge 0$.  This condition immediately implies \eqref{eq:juststab0}. By definition, \eqref{eq:verygood01-} is equivalent to
\begin{align}
    \mu_t \,\,V_t^{-1} \preceq 
    \frac{1}{2 \kappa_t^{\bar{\rho}+2}}I. 
    \label{eq:verygood01}
\end{align}
Our goal is to select $\lambda$, $\bar{p}_t$, and to specify $\bar{\epsilon}$—the warm-up estimator error bound—so that \eqref{eq:verygood01} holds.

Recall that $\mu_t= r_t+ 2\vartheta \sqrt{r_t} \|V_t\|^{1/2}$,  so sufficient conditions for \eqref{eq:verygood01} are
 \begin{align}
  & r_tV_t^{-1}\preceq \frac{1}{4 \kappa_t^{\bar{\rho}+2} }I \label{eq:Term101}\\
   & 2 \vartheta \sqrt{r_t} \|V_t\|^{1/2}V_t^{-1}\preceq \frac{1}{4 \kappa_t^{\bar{\rho}+2} }I. \label{eq:Term201}
 \end{align}

 By applying the result of Lemma~\ref{eq:lowerbndcov}, namely inequality~(\ref{eq:resLem10}), for $t\geq 400(n+m+\log \frac{t}{\delta})$, one can write: one can write
\begin{align*}
 2 \vartheta \sqrt{r_t} \|V_t\|^{1/2}V_t^{-1}&\preceq \frac{2\vartheta \sqrt{r_t}\|V_t\|^{\frac{1}{2}}}{\frac{\sigma_{\omega}^2 \bar{p}_t}{80}\sqrt{t+\bar{c}}-\bar{C}+\lambda} I. 
   % &\preceq
   %  \frac{1}{32\kappa^{10}} 
\end{align*}
where 
\begin{align}
    \bar{C}:=\frac{\sigma_{\omega}^2 \bar{c}}{80}\label{eq:C_tDefns}.
\end{align}

Then a sufficient condition for (\ref{eq:Term201}) is
\begin{align}
   \frac{2\vartheta \sqrt{r_t}\|V_t\|^{\frac{1}{2}}}{\frac{\sigma_{\omega}^2 \bar{p}_t}{80}\sqrt{t+\bar{c}}-\bar{C}+\lambda} &\leq \frac{1}{4 \kappa_t^{\bar{\rho}+2}} \label{eq:pcsdg0ns}
\end{align}
and similarly, a sufficient condition for (\ref{eq:Term101}) is
\begin{align}
     \frac{r_t}{\frac{\sigma_{\omega}^2 \bar{p}_t}{80}\sqrt{t+\bar{c}}-\bar{C}+\lambda}\leq \frac{1}{4 \kappa_t^{\bar{\rho}+2}}. \label{eq:pcsdg+ns}
\end{align}

Let us choose $\lambda$ such that
\begin{align*}
    \lambda\geq \bar{C}. 
\end{align*}
Under this choice, a sufficient condition for \eqref{eq:pcsdg0ns} can be derived as follows:
\begin{align*}
   \frac{2\vartheta \sqrt{r_t}\|V_t\|^{\frac{1}{2}}}{\frac{\sigma_{\omega}^2 \bar{p}_t}{80}\sqrt{t+\bar{c}}-\bar{C}+\lambda} \leq &\frac{2\vartheta \sqrt{r_t} (\lambda+\|\sum_{k=1}^tz_kz_k^\top\|)^{1/2}}{\frac{\sigma_{\omega}^2 \bar{p}_t}{80}\sqrt{t+\bar{c}}}\\
   \leq&\frac{1}{4 \kappa_t^{\bar{\rho}+2} }.
\end{align*}

To satisfy the inequality above, it suffices to require
\begin{align*}
    \bar{p}_t\geq \frac{640 \kappa_t^{\bar{\rho}+2}\vartheta\sqrt{r_t}(\lambda+\|\sum_{k=1}^tz_kz_k^\top\|)^{1/2}}{\sigma_{\omega}^2\sqrt{t+\bar{c}}}.
\end{align*}
Similarly, a sufficient condition for \eqref{eq:pcsdg+ns} is given by

\begin{align*}
    \bar{p}_t\geq \frac{320\,r_t\,\kappa_t^{\bar{\rho}+2}}{\sigma_{\omega}^2\sqrt{t+\bar{c}}}.
\end{align*}
We then choose $\bar{P}_t$ as
\begin{align}
    \bar{p}_t:=\max \bigg\{\frac{640 \kappa_t^{\bar{\rho}+2}\vartheta\sqrt{r_t}(\lambda+\|\sum_{k=1}^tz_kz_k^\top\|)^{1/2}}{\sigma_{\omega}^2\sqrt{t+\bar{c}}},\; \frac{320\,r_t\,\kappa_t^{\bar{\rho}+2}}{\sigma_{\omega}^2\sqrt{t+\bar{c}}}\bigg\}.\label{eq:candbarp}
\end{align}

Note that, for the specific choice of $\bar{c}$ introduced later, one can show that $\bar{p}_t$ is determined by the first term for all $t$'s.

Next, we proceed to specify $\bar{c}$ such that
\begin{align}    B_*\frac{\bar{p}_t\sigma_{\omega}^2\big(K(\mathcal{C}_t) K^\top(\mathcal{C}_t)+\frac{\|P(\mathcal{C}_t)\|I}{\alpha_0}\big)}{\sqrt{t+\bar{c}}} B_*^\top\preceq \sigma_{\omega}^2I.\label{eq:noiseNbig}
\end{align}
This condition ensures that the effect of input perturbation noise on the closed-loop system does not exceed the magnitude of the process noise. 

A sufficient condition for \eqref{eq:noiseNbig} can be written as
 \begin{align}
     \frac{2  \, \kappa_t^2 \, \bar{\vartheta}_{B_*}^2 \bar{p}_t}{\sqrt{t+\bar{c}}}\leq 1 \label{eq:goodcbarfind}
 \end{align}
where we define $\bar{\vartheta}_{B_*} = \max\{1, \vartheta_{B_*}\}$ and $\|B_*\| \leq \vartheta_{B_*}$.

To derive a suitable upper bound for $\bar{p}_t$, we first establish the following inequalities.

We have
\begin{align}
    \log \frac{n\det V_t}{\delta\, \det(\lambda I)}
    &\le (n+m)\log\frac{n}{\delta}\Big(1 + t\log\frac{t}{\delta}\Big) \label{eq:seqp}\\
    &\le 2(n+m)\log\frac{2nt}{\delta}. \label{eq:seqpf1ns}
\end{align}
where \eqref{eq:seqp} holds under the choice
\begin{align}
    \lambda\geq \frac{\sum_{k=1}^t \|z_k z_k^\top\|^2}{t \log \frac{2t}{\delta}}.\label{eq:Lam_1ns}
\end{align}

By choosing $\epsilon$ appropriately such that
\begin{align}
    \sqrt{\lambda} \epsilon \leq \sigma_{\omega}\sqrt{2n(n+m)} \label{eq:epslamb}
\end{align}
it then follows from \eqref{eq:seqpf1ns} and \eqref{eq:epslamb} that
\begin{align*}
    r_t\leq 16\sigma_{\omega}^2n(n+m)\log \frac{2nt}{\delta}.
\end{align*}
Furthermore, from {Lemma} \ref{lem:costatebound2p}, we have
\begin{align*}
    \|z_t\|^2\leq  c^2_z(\kappa_*,\bar{\rho})\log \frac{2t}{\delta}.
\end{align*}
To obtain a computable value for $\bar{c}$, we require an upper bound on $\kappa_*$, which can be derived directly from Lemma \ref{lem:closeness_only_strong} and, in particular, from \eqref{eq:Sineqjk}:
\begin{align}
    \kappa_t^2 \le \kappa_*^2 \le \kappa_t^2 + 4\kappa_t^{6 - \bar{\rho}}, \quad \text{for all } t. \label{eq:goodvn}
\end{align}
From this, we can define
\begin{align*}
   \kappa_*\leq \bar{\kappa}_1(\bar{\rho}) := \sqrt{\kappa_1^2 + 4\kappa_1^{6 - \bar{\rho}}}. 
\end{align*}
Combining these results, and noting that $\kappa_t \le \bar{\kappa}_1$ as guaranteed by \eqref{eq:goodvn}, we can now derive and upper-bound for the left hand side of \eqref{eq:goodcbarfind} as follows
\begin{align}
\frac{1280 \bar{\vartheta}_{B_*}^2\kappa_t^{\bar{\rho}+4}\vartheta\sqrt{r_t}}{\sigma_{\omega}^2(t+\bar{c})}\big(\lambda+tc_z^2(\kappa_*,\bar{\rho})\log \frac{2t}{\delta})\big)^{1/2}\leq& \frac{5120 \bar{\vartheta}_{B_*}^2\bar{\kappa}_1^{\bar{\rho}+4}\vartheta\sqrt{\sigma_{\omega}^2n (n+m)\log \frac{2nt}{\delta}}}{\sigma_{\omega}^2(t+\bar{c})}\sqrt{\lambda} \label{eq:tm10}\\
  &+\frac{5120 \bar{\vartheta}_{B_*}^2\bar{\kappa}_1^{\bar{\rho}+4}\vartheta\sqrt{\sigma_{\omega}^2n (n+m)t}}{\sigma_{\omega}^2(t+\bar{c})}c_z(\bar{\kappa}_1,\bar{\rho})\log \frac{2nt}{\delta}.\label{eq:tm20}
\end{align}
We enforce the following condition:
\begin{align}
    \frac{5120 \bar{\vartheta}_{B_*}^2\bar{\kappa}_1^{\bar{\rho}+4}\vartheta\sqrt{\sigma_{\omega}^2n (n+m)}}{\sigma_{\omega}^2\sqrt{t+\bar{c}}}c_z(\bar{\kappa}_1,\bar{\rho})\log \frac{2nt}{\delta}\leq \frac{1}{2}.\label{eq:helpingCond}
\end{align}
Applying Lemma \ref{lem:Usefulalgebricin}, an appropriate choice of $\bar{c}$ that satisfies this condition is
\begin{align}
    \bar{c} = a^2 \left( 2 \log \!\left( \frac{4n a^2}{\delta} \right) + 1 \right)^2, \label{eq:magncbar}
\end{align}
where
\begin{align*}
    a =\frac{10240 \bar{\vartheta}_{B_*}^2\bar{\kappa}_1^{\bar{\rho}+4}\vartheta\sqrt{\sigma_{\omega}^2n (n+m)}c_z(\bar{\kappa}_1,\bar{\rho})}{\sigma_{\omega}^2}.
\end{align*}
With this choice of $\bar{c}$, condition \eqref{eq:helpingCond} ensures that the term \eqref{eq:tm20} is less than $1/2$ since $t+\bar{c} > 1$. Moreover, by choosing 
$\lambda=\bar{C}=\frac{\sigma_{\omega}^2\bar{c}}{40}$, we also guarantee the term \eqref{eq:tm10} is less that $1/2$.

Recalling \eqref{eq:Lam_1ns}, which provides an additional condition for tuning $\lambda$ based on the upper bound of the co-state, we require
\begin{align*}
    \lambda\geq c_z^2(\bar{\kappa}_1, \bar{\rho}).
\end{align*}
Therefore, the natural choice is to select
\begin{align}
    \lambda \geq \max \left\{ \frac{\sigma_{\omega}^2 \bar{c}}{80},\, c_z^2(\bar{\kappa}_1, \bar{\rho}) \right\} = \frac{\sigma_{\omega}^2 \bar{c}}{80},\label{eq:lamavvalieh}
\end{align}
where the equality holds due to the magnitude of $\bar{c}$ in \eqref{eq:magncbar}.

However, we slightly increase it to ensure closed-loop stability, as will be discussed in the proof of Theorem \ref{prop2}. We choose
\begin{align}
    \lambda = \frac{\sigma_{\omega}^2 \bar{c}}{80} (n+m) \log \frac{1}{\delta}, \label{eq:lamchosenns}
\end{align}

With $\bar{c}$ specified, $\lambda$ is determined. Then, from \eqref{eq:candbarp}, $\bar{p}_t$ is obtained as
\begin{align}
    \bar{p}_t:=\frac{640 \kappa_t^{\bar{\rho}+2}\vartheta\sqrt{r_t}(\lambda+\|\sum_{k=1}^tz_kz_k^\top\|)^{1/2}}{\sigma_{\omega}^2\sqrt{t+\bar{c}}}.\label{eq:candbarptp}
\end{align}

Recalling \eqref{eq:epslamb}, the appropriate choice of $\underline{\epsilon}(\bar{\kappa}_1, \bar{\rho})$ is given by
\begin{align}
     \underline{\epsilon}(\bar{\kappa}_1, \bar{\rho}) := \frac{\sigma_\omega \sqrt{2 n (n+m)}}{\sqrt{\lambda}}.\label{eq:epscrudebnd0102}
\end{align}

With $\underline{\epsilon}(\bar{\kappa}_1, \bar{\rho})$ defined in \eqref{eq:epscrudebnd0102}, and with $\lambda$ and $\bar{p}_t$ selected according to \eqref{eq:lamchosenns} and \eqref{eq:candbarptp}, respectively, the condition \eqref{eq:verygood01} is satisfied. Moreover, by following the same reasoning as in Lemma \ref{Stability_lemma18}, it can be concluded that any generated policy is $(\kappa_*, \gamma_*)$-strongly stabilizing, where
\begin{align*}
\kappa_* = \sqrt{\frac{2\|P(\mathcal{C}_t)\|}{\alpha_0}}, \quad \gamma_* = \frac{1}{2 \kappa_*^2}.
\end{align*}
By following a similar bootstrapping analysis as in Theorem~1, we can show that any policy generated by ARSLO$^+{(\bar{\rho})}$ is $(\kappa_*,\gamma_*)$-strongly stabilizing for any $\rho \in [0,8)$ with probability at least $1-3\delta$, given~(\ref{eq:epscrudebnd0102}).

Furthermore, similar to the analysis of Lemma \ref{Stability_lemma18} , for $t \leq 400\bigl(n+m+\log(t/\delta)\bigr)$ the definition of \(\lambda\) ensures that the sufficient conditions  (\ref{eq:Term101}) and (\ref{eq:Term201}) are satisfied simultaneously.
\end{proof}

\begin{lemma} \label{lem:closeness_only_strong} 
Suppose that the ARSLO$^{+}(\bar{\rho})$ algorithm, for some $\bar{\rho} \in [0,\,8)$, is run with parameters $\underline{\epsilon}(\bar{\kappa}_1, \bar{\rho})$, $\lambda$, and $\bar{p}_t$ tuned as in (\ref{eq:epscrudebnd0102}), (\ref{eq:lamchosenns}), and (\ref{eq:candbarptp}). Then, with probability at least $1-3\delta$, the following statements hold:
\begin{align}
    P(\mathcal{C}_t) \preceq P_*\preceq\;  P(\mathcal{C}_{t+1})+ 2\kappa_{t+1}^{6-\bar{\rho}}\alpha_0I\preceq\;  P(\mathcal{C}_{t+1})+ 2\kappa_{*}^{6-\bar{\rho}}\alpha_0I, \label{eq:Sineqjk}
\end{align}
and
\begin{align}
    \big\|H_{t+1}^{-1}H_t\big\| \;\leq\; 1+2{\kappa_*^{3-\bar{\rho}/2}} \;=:\; {1+\xi(\kappa_*,\bar{\rho}).} \label{def:xirho}
\end{align}
where $H_k=P^{1/2}(\mathcal{C}_k)$.
\end{lemma}

\begin{proof}

We begin by proving the first claim. In Lemma~\ref{lem:closeness}, we have already established 
\begin{align*}
   P(\mathcal{C}_t) \preceq P_* 
\end{align*}
and we have also shown that
\begin{align}
   \nonumber \nonumber  P_*- P(\mathcal{C}_t)&\preceq \frac{16\| P(\mathcal{C}_t)\|^2\mu_t}{\alpha_0^2} \|P(\mathcal{C}_t)\| \|\begin{pmatrix}
I \\
K(\mathcal{C}_t)
\end{pmatrix}^\top V^{-1}_t \begin{pmatrix}
I \\
K(\mathcal{C}_t) 
\end{pmatrix} \| I\\
&= 4\kappa_t^4 \mu_t\|P(\mathcal{C}_t)\| \|\begin{pmatrix}
I \\
K(\mathcal{C}_t)
\end{pmatrix}^\top V^{-1}_t \begin{pmatrix}
I \\
K(\mathcal{C}_t) 
\end{pmatrix} \| I \label{eq:Sbadnis}. 
\end{align}
Moreover, since the parameters $\underline{\epsilon}(\bar{\kappa}_1, \bar{\rho})$, $\lambda$, and $\bar{p}_t$ have been selected such that
\begin{align}
   \mu_t  \| P(\mathcal{C}_t)\| \|V_t^{-1}\|I\preceq \frac{\alpha_0}{4\kappa_t^{\bar{\rho}}}I \label{eq:satsfiedIneq1}
\end{align}
which indeed is a sufficient condition for 
\begin{align*}
     \mu_t  \| P(\mathcal{C}_t)\| V_t^{-1}\preceq\frac{\alpha_0}{4\kappa_t^{\bar{\rho}}}I
\end{align*}
which fulfills generation of $(\kappa_*, \gamma_*)$-strongly stabilizing policies for any $\bar{\rho}\in [0,\; 8)$. We also have: 
\begin{align}
    \|\begin{pmatrix}
I \\
K(\mathcal{C}_t)
\end{pmatrix}^\top\begin{pmatrix}
I \\
K(\mathcal{C}_t) 
\end{pmatrix} \|\leq 2 \kappa_t^2. \label{eq:satsfiedIneq2}
\end{align}
Substituting (\ref{eq:satsfiedIneq1}) and (\ref{eq:satsfiedIneq2}) into (\ref{eq:Sbadnis}) yields
\begin{align}
    P_* \preceq\;  P(\mathcal{C}_{t})+ 2\kappa_t^{6-\bar{\rho}}\alpha_0I\; \preceq\;  P(\mathcal{C}_{t})+ 2\kappa_*^{6-\bar{\rho}}\alpha_0I \label{eq:ineqpscd}
\end{align}
where the last inequality follows from $\kappa_t \leq \kappa_*$.

Applying similar steps as in (\ref{eq:lemprfbound}) on (\ref{eq:ineqpscd}) implies
  \begin{align*}
\|H_{t+1}^{-1}H_t\|^2 \leq 1+4{\kappa_*^{6-\bar{\rho}}}
 \end{align*}
 which results in: 
 \begin{align*}
     \|H_{t+1}^{-1}H_t\|\leq \sqrt{ 1+4{\kappa^{6-\bar{\rho}}}}\leq 1+2{\kappa_*^{3-\bar{\rho}/2}}=:1+\xi(\kappa_*,\bar{\rho}).
 \end{align*}
\end{proof}

\begin{remark} \label{eq:goodremarkb}
The condition in (\ref{def:xirho}) holds between any two policies because the upper and lower bounds on $P_*$ in (\ref{eq:Sineqjk}) are valid for arbitrary time steps. This matrix inequality is, in fact, used to derive (\ref{def:xirho}). Therefore, for any pair of time steps $(t^{\prime}, t)$, we have
\begin{align*}
    \|H_{t^{\prime}}^{-1} H_t\| \leq 1 + \xi(\kappa_*, \bar{\rho}).
\end{align*}
\end{remark}

\subsection{Proof of Theorem \ref{prop2}} 

\begin{proof}
Let $M=A_*+B_*K_t$ where $K_t'$s are $(\kappa, \gamma)-$ strong sequential stabilizing policies generated by ARSLO algorithm. The closed loop dynamics then is written as follows:
    \begin{align*}
        x_t=M_1x_0+\sum_{s=1}^{t-1}M_{s+1}(B_*\nu_{s}+\omega_{s+1}) 
    \end{align*}
   where 
   \begin{align*}
       M_t=I,\;\; M_{s}=M_{s+1}(A_*+B_*K_s)=\prod_{j=s}^{t-1}(A_*+B_*K_j), \; \forall 1\leq s\leq t-1. 
   \end{align*}
   Recalling Definition \ref{def:sequentially}, we can set $A_*+B_*K_j= H_j L_j H_j^{-1}$, then for all $1\leq s< t$ where $\|L_j\|\leq 1-\gamma_j$ one can write
  \begin{align*}
     \nonumber  \|M_s\|\leq &\lVert (H_{t-1}L_{t-1}H^{-1}_{t-1})...(H_{s+1}L_{s+1}H^{-1}_{s+1}) (H_{s}L_{s}H^{-1}_{s})\rVert\\
     \leq &\|H_{t-1}\| \; \|(L_{t-1}H^{-1}_{t-1}H_{t-2})...(L_{s+2}H^{-1}_{s+2}H_{s+1})(L_{s+1}H^{-1}_{s+1}H_{s})\|\:\|L_s\|\:\|H_s\|^{-1}\\
    \nonumber  \leq & \|H_{t-1}\| \bigg(\prod_{j=s+1}^{t-1}\|L_j\|\bigg) \bigg(\prod_{j=s+1}^{t-1} \|H_{j}^{-1}H_{j-1}\|\bigg) (1-\gamma_s)\|H_s^{-1}\|\\
    \leq &\sqrt{P(\mathcal{C}_{t-1})} \bigg(\prod_{j=s+1}^{t-1}(1-\gamma_j)\bigg) \bigg(\prod_{j=s+1}^{t-1} \|H_{j}^{-1}H_{j-1}\|\bigg)  (1-\gamma_s)\sqrt{\frac{2}{\alpha_0}}\\
\leq  &\kappa_{{t-1}}\bigg(\prod_{j=s+1}^{t-1}(1-\gamma_j)\; \bigg) \bigg(\prod_{j=s+1}^{t-1} \|H_{j}^{-1}H_{j-1}\|\bigg)(1-\gamma_s)\\
   \leq &{\kappa_*} (1-{\gamma_*})^{t-s} \bigg(\prod_{j=s+1}^{t-1} \|H_{j}^{-1}H_{j-1}\|\bigg)\\
   \leq &{\kappa_*} (1-{\gamma_*})^{t-s} \big(1+\xi(\kappa_*,\bar{\rho})\big)^{N(t-1)-N(s+1)} \\
   \leq& {\kappa_*} (1-{\gamma_*})^{t-s} \big(1+\xi(\kappa_*,\bar{\rho})\big)^{N(t-s-1)}
  \end{align*}  
where we applied 
\begin{align*}
     \|H_{j}^{-1}H_{j-1}\|  
    \begin{cases}
       = 1 & \text{if there is no update in the policy}, \\
       \leq  1+\xi(\kappa_*,\bar{\rho}) & \text{otherwise (by Lemma \ref{lem:closeness_only_strong})},
    \end{cases}
\end{align*}
 $\kappa_{t-1} = \sqrt{2P(\mathcal{C}_{t-1}) / \alpha_0}$ by definition, and that $\kappa_t \leq \kappa_*$ for any $t$. Also, we use the fact that $N(t)-N(s)\leq N(t-s)$ which follows from the monotonicity of epoch lengths. This monotonicity is a direct consequence of the determinant growth identity for positive definite matrices:
\begin{align*}
    \det(V_{k+1})=\det(V_{k})\big(I+z_k^\top V^{-1}_{k}z_k\big).
\end{align*}
Since \(V_k \succeq V_{k-1}\), thus, the incremental growth of \(\det(V_k)\) decreases over time. Consequently, to satisfy the update condition $\det(V_k)\;\ge\; \bigl(1+\beta(\bar{\rho})\bigr)\,\det(V_{\tau(k)})$ a longer interval between updates is required. In other words, epoch lengths are nondecreasing, which is an established fact for OFU-based algorithms with determinant-type policy update rules. This implies that update events become less frequent as time increases. Hence, for any interval \([s,t)\), the number of updates within this interval is at most the number of updates that could occur in an initial interval of the same length, establishing $N(t)-N(s)\;\le\;N(t-s)$.

Given the policy update rule, for some tunable \(\beta_{\tau(t)}(\bar{\rho})\), we aim to upper-bound the number of policy updates by time \(s\).  
Similar to the analysis in Lemma \ref{Thm.BoundR1}, the number of policy updates by time \(s\), denoted \(N(s)\), can be upper-bounded as
\begin{align}
    N(s)\leq \log_{1+\beta(\bar{\rho})} \bigg(\frac{\lambda+\|\sum_{k=1}^s z_kz_k^\top\|}{\lambda}\bigg)^{n+m}. \label{eq:lamineq}
\end{align}

The only difference compared to Lemma \ref{Thm.BoundR1} is that \(\beta_{\tau(k)}(\bar{\rho})\) is not constant; hence we use its minimum over all epochs to upper-bound \(N(s)\).

Note that by Lemma \ref{lem:closeness_only_strong}, we have $\xi(\kappa_*,\bar{\rho})=2\kappa_*^{3-\bar{\rho}/2}$ and moreover
\begin{align*}
    \kappa_t^2 \le \kappa_*^2 \le \kappa_t^2 + 4\kappa_t^{6 - \bar{\rho}}, \quad \text{for all } t 
\end{align*}
and then
\begin{align*}
    \kappa_1^2 \le \kappa_*^2 \le \kappa_1^2 + 4\kappa_1^{6 - \bar{\rho}}. 
\end{align*}

Using this, we choose
\begin{align}
    \beta(\bar{\rho}) =\begin{cases}
        2(\kappa_{1}^2 + \kappa_{1}^{(6-\bar{\rho})})^{\frac{(6-\bar{\rho})}{4}} & \text{if } \bar{\rho}<6, \\
        2 & 6\leq\bar{\rho}<8
    \end{cases} \label{eq:goodbetatun}
\end{align}
which guarantees that
\begin{align}
    1+\beta(\bar{\rho})\geq 1+\xi(\kappa_*,\bar{\rho}).  \label{eq:propusef}
\end{align}

Note that for $6 \leq \bar{\rho} < 8$, we have $\xi(\kappa_{*}, \bar{\rho}) \leq 2$. This is why, over this specific range, we choose $\beta_{\tau(k)}(\bar{\rho}) = 2$, which allows us to control the growth of the state norm, as shown below. With this choice, we obtain
\begin{align*}
     \nonumber (1-\gamma_*)^{\frac{s}{2}}(1+\xi(\kappa_*,\bar{\rho}))^{N(s)}\leq & (1-\gamma_*)^{\frac{s}{2}}\bigg(\frac{\lambda+\|\sum_{k=1}^s z_kz_k^\top\|}{\lambda}\bigg)^{n+m}
\end{align*}
where we used the property $a^{\log_b c} \le c$ for $b \ge a > 1$.

Using the standard bound
\(\big\|\sum_{k=1}^s z_k z_k^\top\big\|\le c_z^2(\kappa_*,\bar{\rho})\,t\log\frac{t}{\delta}\), for $\delta\leq 1/e$ we can further upper-bound 
\begin{align*}
\bigg(\frac{\lambda+\|\sum_{k=1}^s z_kz_k^\top\|}{\lambda}\bigg)^{n+m}&\leq \bigg(1+ \frac{c^2_z(\kappa_*,\bar{\rho})s \log \frac{s}{\delta}}{\lambda}\bigg)^{n+m}\\
&\leq \bigg(1+ \frac{c^2_z(\kappa_*,\bar{\rho})s^2 \log \frac{1}{\delta}}{\lambda}\bigg)^{n+m}\\
    &\leq \bigg(1+ \frac{c_z(\kappa_*,\bar{\rho})s \sqrt{\log\frac{1}{\delta}}}{\sqrt{\lambda}}\bigg)^{2(n+m)}.
\end{align*}
The second inequality follows from the property that $\log a b\leq a \log b$ for $a\geq1$ and $b\geq e$. The third inequality uses the fact that $(1+x^2)\leq (1+x)^2$ for any $x\geq0$.

Recall from (\ref{eq:lamchosenns}) that we chose
\begin{align*}
    \lambda=\frac{\sigma_{\omega}^2\bar{c}}{40}(n+m)\log \frac{1}{\delta}
\end{align*}
Considering the definition of \(\bar{c}\) in (\ref{eq:magncbar}), it is straightforward to verify that
\begin{align*}
    \lambda\geq 64(n+m)^2\bar{\kappa}_1^8 c^2_z(\bar{\kappa}_1^2,\bar{\rho})\log\frac{1}{\delta}.
\end{align*}
This then gives
\begin{align*}
\bigg(1+ \frac{c_z(\kappa_*,\bar{\rho})s \sqrt{\log\frac{1}{\delta}}}{\sqrt{\lambda}}\bigg)^{2(n+m)}\leq \bigg(1+ \frac{s}{8(n+m)\bar{\kappa}_1^4}\bigg)^{2(n+m)}\leq e^{\frac{s}{4\bar{\kappa}_1^4}}
\end{align*}
where the last inequality follows from the standard property
 $(1+\frac{x}{a})^a\leq e^{x/a}$ for any $x\geq 0$.

Then we have
\begin{align}
   \nonumber (1-\gamma_*)^{\frac{s}{2}}(1+\xi(\kappa_*,\bar{\rho}))^{N(s)}\leq & (1-\gamma_*)^{\frac{s}{2}}e^{\frac{s}{4\bar{\kappa}_1^4}}\\
   \nonumber = & (1-\frac{1}{2\kappa_*^2})^{\frac{s}{2}} e^{\frac{s}{4\kappa_*^4}}\\
 \nonumber \leq & e^{\frac{-s}{4\kappa_*^2}}e^{\frac{s}{4\kappa_*^4}}=e^{-(\frac{1}{4\kappa_*^2}-\frac{1}{4\kappa_*^4})s}. 
\end{align}
Hence, we can upper bound
\begin{align*}
    \|M_1\|\leq \kappa_* (1-\gamma_*)^{\frac{t-s}{2}}
\end{align*}
and
\begin{align*}
    \|M_s\|\leq \kappa_* (1+\xi(\kappa_*,\bar{\rho})) (1-\gamma_*)^{\frac{t-s}{2}}.
\end{align*}
Using this bound, the upper bound on the state norm becomes
    \begin{align}
        \nonumber \|x_t\|\leq& \|M_1\|\|x_1\|+\sum_{s=1}^{t-1}\|M_{s+1}\|\|B_*\eta_j+\omega_{s+1}\|\\
       \nonumber \leq &\kappa_*  (1-\gamma_*)^{(t-1)/2} \|x_1\|+\kappa_* (1+\xi(\kappa_*,\bar{\rho}))\sum_{s=1}^{t-1}(1-{\gamma_*})^{(t-s-1)/2}  \|B_*\eta_s+\omega_{s+1}\| \\
       \nonumber \leq &\kappa_*  e^{-{\gamma_*(t-1)}/{2}}\, \|x_1\|+\kappa_* (1+\xi(\kappa_*,\bar{\rho})) \max_{1\leq s< t}\|B_*\eta_{s}+\omega_{s+1}\|\sum_{t=1}^{\infty}(1-{\gamma_*})^{t/2}\\
  \nonumber  \leq & \kappa_*  e^{-{\gamma_*(t-1)}/{2}}\, \|x_1\|+\frac{2\kappa_*}{\gamma_*} (1+\xi(\kappa_*,\bar{\rho})) \max_{1\leq s< t}\|B_*\eta_{s}+\omega_{s+1}\|.
    \end{align}
   This completes the proof.
    
% Now, with the choice of $\bar{c}$ in (\ref{eq:magncbar}), which guarantees
% \begin{align*}
%     B_*\frac{\bar{p}_t\sigma_{\omega}^2\big(K(\mathcal{C}_t) K^\top(\mathcal{C}_t)+\frac{\|P(\mathcal{C}_t)\|I}{\alpha_0}\big)}{\sqrt{t+\bar{c}}} B_*^\top\preceq \sigma_{\omega}^2I
% \end{align*}
% and applying the Hanson–Wright concentration inequality, we obtain
% \begin{align*}
%    \max_{1\leq s\leq t} \| B_*\eta_s+\omega_{s}\|\leq \sqrt{20 \sigma_{\omega}^2n\log \frac{t}{\delta}}. 
% \end{align*}

% Consequently, 
% \begin{align}
%     \|x_t\|\leq \kappa_*  e^{-{\gamma_*(t-1)}/{2}}\, \|x_1\|+\frac{2\kappa_*}{\gamma_*} (1+\xi(\kappa_*,\bar{\rho}))\sqrt{20 n \sigma_{\omega}^2 \log \frac{t}{\delta}}.\label{eq:bundx++}
% \end{align}

\end{proof}

The following lemma gives appropriate upper-bounds for $\|x_t\|^2$ and $\|z_t\|^2$.

\begin{lemma} \label{lem:costatebound2p} 
Fix $\delta\in (0,1)$. Consider ARSLO$^{+}(\bar{\rho})$ equipped with the illustrated policy update mechanism. 
Suppose that the generated policies are $(\kappa_*,\gamma_*)$-strongly stabilizing. 
Then, with probability at least $1-\delta$, the following bound holds:
     \begin{align}
    \|x_t\|^2\leq 4\frac{\kappa_*^2}{\gamma_*^2}(1+\xi(\kappa_*,\bar{\rho}))^2\big(\frac{\|x_1\|}{\sqrt{\log \frac{1}{\delta}}}+\sqrt{20n \sigma_{\omega}^2}\big)^2\log \frac{t}{\delta}=:X_t^2(\kappa_*,\bar{\rho}).\label{eq:Xsttae+}
\end{align} 
And, with $z_t$ defined by (\ref{eq:defztn}), we have
    \begin{align}
        \|z_t\|^2\leq   {64 \kappa_*^6(1+\kappa_*^2)} (1+\xi(\kappa_*,\bar{\rho}))^2\bigg(\frac{\|x_1\|}{\log \frac{1}{\delta}}+\sqrt{20n \sigma_{\omega}^2}\bigg)^2\log \frac{t}{\delta}=:c_z^2(\kappa_*, \bar{\rho})\log \frac{t}{\delta}=:\zeta_t^2(\kappa_*,\bar{\rho}).\label{eq:zetaValllvv}
    \end{align}
\end{lemma}

\begin{proof}
The proof proceeds exactly as in Lemma~\ref{lem:costatebound}, as we still have
\begin{align*}    B_*\frac{\bar{p}_t\sigma_{\omega}^2\big(K(\mathcal{C}_t) K^\top(\mathcal{C}_t)+\frac{\|P(\mathcal{C}_t)\|}{\alpha_0}I\big)}{\sqrt{t+\bar{c}}} B_*^\top
    \preceq \sigma_{\omega}^2 I,
\end{align*}
which holds true by the choice of $\bar{c}$ in (\ref{eq:magncbar}).
\end{proof}

\subsection{Proof of Corollary \ref{cl:strong}}

\begin{proof}
   The proof follows from the definition of $\lambda$, noting that by (\ref{eq:lambda_valuefn01}), $\lambda$ has the same order as $a^2$, namely $\mathcal{O}(\bar{\kappa}_1^{\,22+\bar{\rho}})$. Moreover, by the definition of $\bar{\kappa}_1$ in (\ref{eq:kappabar_no_seq_Def}), it is of order $\mathcal{O}(\max\{\kappa_1,\;\kappa_1^{(6-\bar{\rho})/2}\})$, and $\kappa_1$ itself is upper-bounded by $\kappa_*$. Finally, since $\underline{\epsilon}(\bar{\kappa}_1,\bar{\rho})$ defined in (\ref{eq:epsstst0}) is of order $\mathcal{O}(1/\sqrt{\lambda})$, the claim follows.
\end{proof}

\section {Regret Bound Analysis of ARSLO$^+(\bar{\rho})$ Algorithm}
\label{sec:regasloplusap}

% \textcolor{blue}{\begin{lemma} \label{lem:StatebondReg+}
% With probability at least $1-\delta$ we have
%    \begin{align}
%         \|z_t\|^2\leq  \frac{16 (1+\kappa_*^2)}{\gamma_*^2}\bigg(\frac{\|x_1\|}{\log \frac{1}{\delta}}+\sqrt{20n \sigma_{\omega}^2}\bigg)^2\log \frac{2t}{\delta}+2\bar{\alpha}^2(\kappa_*, n+m,\delta):=\zeta_t^{2}(\kappa_*) \label{eq:zetaStae+}
%     \end{align}
%   Furthermore,  
%     \begin{align}
%     \|x_t\|^2\leq \frac{8}{\gamma_*^2}\big(\frac{\|x_1\|}{\log \frac{1}{\delta}}+\sqrt{20n\sigma_{\omega}^2}\big)^2\log \frac{2t}{\delta}+2\bar{\alpha}^2(\kappa_*, n+m,\delta)=:X_t^2 (\kappa_*) \label{eq:Xsttae+}.
% \end{align} 
% \end{lemma}}

% \begin{proof}
%    The upper bounds follow directly from (\ref{eq:zetaValllvv}), (\ref{eq:bundx++}), and the fact that $\kappa_* \geq \kappa$.
% \end{proof}

\subsection{Proof of Theorem \ref{thm:RegretBoundARSLO+}}

\begin{proof}
The regret bound decomposition in (\ref{eq:R1}--\ref{eq:R6}) remains valid for ARSLO$^{+}(\bar{\rho})$, with the difference that this time the "good event" is $\mathcal{E}^{\bar{\rho}}_k$ given by (\ref{eq:Godevent12}) rather than $\mathcal{E}_k$.
 The other main difference lies in the policy update criterion, $\det(V_t)\geq (1+\beta)\det(V_{\tau})$: for this algorithm, $\beta$ is adjusted according to (\ref{eq:goodbetatun}), rather than being fixed at $\beta = 1$ as in the original ARSLO. 

In the subsequent lemmas, we bound each term individually.

\begin{lemma} \label{Thm.BoundR1+}
The following upper bound holds for $|R_1(t)|$:
	\begin{align*}
	|R_1(t)|\lesssim \|P_*\|^{7-\frac{\bar{\rho}}{2}}n(n+m)\log^2\frac{t}{\delta}.
	\end{align*}
\end{lemma}
\begin{proof}
  The proof follows the same structure as Lemma~\ref{Thm.BoundR1}, with the only difference being that the policy update rule now satisfies $\det(V_t) \geq (1 +\beta_{\tau(k)}(\bar{\rho}))\det(V_{\tau(k)})$ where $\beta_{\tau(k)}(\bar{\rho})$ is defined in~(\ref{eq:goodbetatun}).
The only remaining technicality is to upper-bound the number of policy updates $N(t)$, which we have already established in~(\ref{eq:lamineq}):
\begin{align*}
    N(s)\leq \log_{1+\beta(\bar{\rho})} \bigg(\frac{\lambda+\|\sum_{k=1}^s z_kz_k^\top\|}{\lambda}\bigg)^{n+m}. 
\end{align*}
Moreover, by (\ref{eq:goodbetatun}), for any $\bar{\rho}\in[0,8)$ we have the property (\ref{eq:propusef})
\begin{align*}
    1+\beta(\bar{\rho})\geq 3 \quad \forall k.
\end{align*}
which implies
\begin{align*}
    N(s)&\leq \log_{3} \bigg(\frac{\lambda+\|\sum_{k=1}^s z_kz_k^\top\|}{\lambda}\bigg)^{n+m} \\
    &\leq 4(n+m) \log_{1+\xi(\kappa_*,\bar{\rho})}\frac{t}{\delta}
\end{align*}
where the last inequality holds on the event $\mathcal{E}_{t}^{\bar{\rho}}$ by our choice of $\lambda$ in (\ref{eq:lamavvalieh}), which satisfies
$\lambda \ge c_z^2(\kappa_,\bar{\rho})$.

Thus, similar to the ARSLO algorithm case,
\begin{align*}
    |R_1(t)|\leq& \|P_*\|X_t^{2}(\kappa_*, \bar{\rho})+ 4 \|P_*\|(n+m) X_t^{2}(\kappa_*,\bar{\rho}) \log_{3}\frac{t}{\delta}.
\end{align*}
Using the upper bound on $X_t^2(\kappa_*,\bar{\rho})$ given in (\ref{eq:Xsttae+}), and the fact that
$\xi(\kappa_*,\bar{\rho})\lesssim \|P_*\|^{(3-\bar{\rho}/2)/2}$ by (\ref{def:xirho}), which hold on the event $\mathcal{E}_t^{\bar{\rho}}$,  and the definition of $\kappa_*$, we obtain
\begin{align}
    X^2_t(\kappa_*,\bar{\rho})\lesssim \|P_*\|^{6-\frac{\bar{\rho}}{2}} n \log \frac{t}{\delta}. \label{eq:lessimxbar}
\end{align}
Therefore,
\begin{align*}
    |R_1(t)|\lesssim  \|P_*\|^{7-\frac{\bar{\rho}}{2}}n(n+m)\log^2\frac{t}{\delta}.
\end{align*}
\end{proof}

\begin{lemma} \label{R2+}
	Let $\delta\in (0,1)$. Then following upper bound holds for $|R_2(t)|$:
	\begin{align*}
	|R_2(t)| \lesssim \|P_*\|^{4-\frac{\bar{\rho}}{4}}\sqrt{n \log^2 \frac{t}{\delta}}
	\end{align*}	
with probability at least $1 - \delta/4$.
\end{lemma}
\begin{proof}
  The argument follows the same steps as in Lemma~\ref{R2}. In particular, on the event $\mathcal{E}_t^{\bar{\rho}}$ we have    
    \begin{align*}
	|R_2(t)|\leq 2\sigma_{\omega}\|P_*\|X_t(\kappa_*,\bar{\rho})\sqrt{t\log \frac{8}{\delta}}. 
	\end{align*}
with probability at least $1-\delta/4$.
    
 Applying the bound $X_t(\kappa_*,\bar{\rho})$ given in \eqref{eq:lessimxbar} yields the stated result.
\end{proof}

\begin{lemma}\label{thm:boundR3+}
	Fix $\delta\in (0,1)$. Then the term $R_3(t)$ admits the following bound:
	\begin{align*}
	\nonumber R_3(3)\lesssim \|P_*\|_*\sqrt{t \log^3 \frac{t}{\delta}}
	\end{align*}
    with probability at least $1-\delta/4$.
\end{lemma}
\begin{proof}
The proof follows directly from Lemma~\ref{thm:boundR3}, with the only difference being that the bound now holds on the event $\mathcal{E}_t^{\bar{\rho}}$.
\end{proof}

\begin{lemma}\label{thm:boundR4+}
		Let $\bar{\rho}\in [0,\,8)$. Then the term $R_4(t)$ admits the upper bound		
\begin{align*}
R_4(t)\lesssim\,\vartheta\, \beta_*(\|P_*\|,\bar{\rho}) \,\sqrt{n^2(n+m)^3 \|P_*\|^{9-\bar{\rho}/2}\, t\, \log^4\frac{nt}{\delta}},
\end{align*}
where 
\begin{align*}
    \beta_*(\|P_*\|,\bar{\rho}) =\begin{cases}
        \big(\|P_*\| +\|P_*\|^{(6-\bar{\rho})/2}\big)^{\frac{(6-\bar{\rho})}{4}} & \text{if } \bar{\rho}<6, \\
        2 & 6\leq\bar{\rho}<8
    \end{cases}. 
\end{align*}
\end{lemma}

\begin{proof}
   	The argument follows the same steps as in Lemma~\ref{thm:boundR4}, with the main difference being that $\beta$ is no longer equal to $1$ and is instead tuned according to (\ref{eq:goodbetatun}). Its upper bound on the event $\mathcal{E}_t^{\bar{\rho}}$ satisfies
    \begin{align*}
    \beta(\bar{\rho})\lesssim \beta_*(\|P_*\|,\bar{\rho}) =\begin{cases}
        \big(\|P_*\| +\|P_*\|^{(6-\bar{\rho})/2}\big)^{\frac{(6-\bar{\rho})}{4}} & \text{if } \bar{\rho}<6, \\
        2 & 6\leq\bar{\rho}<8
    \end{cases}. 
\end{align*}
 With this adjustment, then  
	\begin{align}
 R_4(t)\leq &4\sum _{k=1}^{t} \bigg(1+\beta_*(\|P_*\|,\bar{\rho})\bigg)\,\|P(\mathcal{C}_k)\|\mu_{\tau(k)}\big({z}^\top_kV_{k}^{-1}{z}_k\big)1_{\mathcal{E}_k`}.\label{eq:bndR4+}
	\end{align}
Following the same steps in bounding $\mu_t$ from Lemma~\ref{thm:boundR4}, and using the fact that  
\begin{align}
    c_z^2(\kappa_*,\bar{\rho})\lesssim\, n \|P_*\|^{7-\frac{\bar{\rho}}{2}}.\label{eq;bndstateco+}
\end{align}
which holds on the event $\mathcal{E}_t^{\bar{\rho}}$, we obtain
\begin{align*}
   \mu_{\tau(k)}\leq \mu_k\lesssim& \vartheta \,\sqrt{r_k}\,\sqrt{\zeta_k(\kappa_*,\bar{\rho})\, k}\\
   \lesssim &\vartheta \sqrt{n(n+m)\frac{2nt}{\delta}}\sqrt{n\, k\, \|P_*\|^{7-\frac{\bar{\rho}}{2}} \log \frac{k}{\delta}}\\
   \lesssim & \vartheta\sqrt{n^2(n+m)\|P_*\|^{7-\bar{\rho}/2}\, k\, \log^{2}\frac{2nk}{\delta}}.
\end{align*}
As in Lemma~\ref{thm:boundR4}, we also use
\begin{align}
   \nonumber  \sum_{t=1}^{t-1}z_t^\top V_t^{-1}z_t&\leq 2\log \frac{\det(V_{t})}{\det(\lambda I)}\lesssim (n+m)\log \frac{nt}{\delta}
\end{align}
where the last inequality holds on the event $\mathcal{E}_t^{\bar{\rho}}$.

Substituting all bounds into (\ref{eq:bndR4+}) yields:
\begin{align}
	\nonumber R_4(t)\lesssim\,\vartheta\, \beta_*(\|P_*\|,\bar{\rho}) \,\sqrt{n^2(n+m)^3 \|P_*\|^{9-\bar{\rho}/2}\, t\, \log^4\frac{nt}{\delta}}
\end{align}
which completes the proof.
\end{proof}

\begin{lemma}\label{thm:boundR5+}
	Let $\delta\in (0,1)$. Then the term $R_5(t)$ admits the following upper bound:	
    \begin{align*}
     R_5(t)\lesssim \alpha_1 (n^4(n+m))^{\frac{1}{4}} \sqrt{\vartheta  \|P_*\|^{(50-\bar{\rho})/4} \log^3 \frac{nt}{\delta}} \,t^{1/4}
\end{align*}
with probability at least $1-\delta/4$.
\end{lemma}
\begin{proof}
 The proof follows the same steps as in Lemma~\ref{thm:boundR5}. Similar to (\ref{eq:arslobndq1}), we may write  
\begin{align}
    R_5\leq 4\sqrt{2\bar{p}_t \log \frac{8}{\delta}} \alpha_1\sigma_{\omega}\kappa_*^2X_t(\kappa_*, \bar{\rho}) (t+\bar{c})^{1/4} \label{eq:arslobndq1+}
\end{align}
which holds with probability at least $1-\delta/4$.

On the event $\mathcal{E}_t^{a}$, using the definition of $\bar{p}_t$ in~(\ref{eq:candbarptp}) and the fact that 
$r_t\lesssim n(n+m)\log \tfrac{nt}{\delta}$, together with the bound in 
(\ref{eq;bndstateco+}), we obtain
\begin{align}
    \bar{p}_t\lesssim \vartheta \sqrt{n^2(n+m) \|P_*\|^{(9+\bar{\rho}/2)} \, \log ^2\frac{nt}{\delta}}. \label{eq:upBndpbar+}
\end{align}
Moreover, we already derived the bound on $X_t^2(\kappa_*,\bar{\rho})$ in 
(\ref{eq:lessimxbar}):
\begin{align}
    X^2_t(\kappa_*,\bar{\rho})\lesssim \|P_*\|^{6-\frac{\bar{\rho}}{2}} n \log \frac{t}{\delta}. 
\end{align}
Substituting these expressions into~\eqref{eq:arslobndq1+} completes the proof and yields
\begin{align*}
    R_5(t)\lesssim \alpha_1 (n^4(n+m))^{1/4} \sqrt{\vartheta  \|P_*\|^{(50-\bar{\rho})/4} \log^3 \frac{nt}{\delta}} t^{1/4}.
\end{align*}
with probability at least $1-\delta/4$.
\end{proof}
\begin{lemma}\label{thm:boundR6+}
	The term $R_6(t)$ has the following upper bound.	
\begin{align*}
     R_6(t)\lesssim \alpha_1 \vartheta \sqrt{n^2(n+m)m^2 \|P_*\|^{(11+\bar{\rho}/2)} t \log^4 \frac{nt}{\delta}}
\end{align*}
with probability at least $1-\delta$.
\end{lemma}
\begin{proof}
The following inequality holds
\begin{align}
   \nonumber R_6(t) =&\sum _{k=1}^{t} \eta^\top_k \big(R+4\mu_{\tau(k)} \| P(\mathcal{C}_k)\|V^{-1}_{\tau(k)}\big) \eta_k 1_{\mathcal{E}^{\bar{\rho}}_k}\\
   \nonumber \leq& \sum _{k=1}^{t} \eta^\top_k \big(R+\alpha_0I\big) \eta_k 1_{\mathcal{E}^{\bar{\rho}}_k}\\
   \nonumber  \leq& 2\alpha_1 \sum _{k=1}^{t} \eta^\top_k\eta_k 1_{\mathcal{E}^{\bar{\rho}}_k}. 
\end{align}
The second inequality follows from (\ref{eq: verygood22}), namely
\begin{align*}
    \mu_t  \| P(\mathcal{C}_t)\|_* V_t^{-1}\preceq \frac{\alpha_0}{4\kappa_t^{\bar{\rho}}}I \quad \forall t 
\end{align*}
which is ensured on the event $\mathcal{E}_t^{\bar{\rho}}$ by appropriate tuning of $\lambda$, $\bar{p}_t$, and $\bar{\epsilon}$, as followed by the proof of Theorem \ref{Stability_thm200}.

The proof follows the same steps of teh counter part lemma for ARSLO algorithm, i.e., Lemma \ref{thm:boundR6}. By applying Hanson--Wright inequality (Lemma \ref{lem:HansonWright}), it yields 
\begin{align*}
    \sum_{k=1}^t  z_k^\top \Gamma_k z_k\, 1_{\mathcal{E}_k^{\bar{\rho}}}&\leq 5\sum_{k=1}^t\|\Gamma_k\|_*\log\frac{4k}{\delta}\leq 10m \|\Gamma_t\| \sigma_{\omega}^2 \log \frac{4t}{\delta}\kappa_*^2\sum_{k=1}^t \frac{1}{\sqrt{k+\bar{c}}}\\
    &\leq 20\,m\,\bar{p}_t\, \sigma_{\omega}^2\kappa_*^2\,\log \frac{4t}{\delta} \big(\sqrt{t+\bar{c}}-\sqrt{\bar{c}}\big)
\end{align*}
which holds with probability at least $1-\delta/4$ on the event $\mathcal{E}_t^{\bar{\rho}}$, where we also used the fact that $\|\Gamma_k\|_* \leq m\|\Gamma_k\|$.

Finally by upper-bound of $\bar{p}_t$ given by (\ref{eq:upBndpbar+}) we conclude that

\begin{align*}
    R_6(t)\lesssim \alpha_1 \vartheta \sqrt{n^2(n+m)m^2 \|P_*\|^{(11+\bar{\rho}/2)} t \log^4 \frac{nt}{\delta}}.
\end{align*}
with probability at least $1-\delta/4$.
\end{proof}

To complete the proof, it suffices to observe that the terms $R_4(t)$ and $R_6(t)$ are dominant, and depending on the value of $\bar{\rho}$, either term may dominate individually. Therefore, for any $\bar{\rho} \in [0,\,8)$, the regret upper bound can be expressed as
\begin{align*}
   \tilde{R}_{\text{ARSLO}^+(\bar{\rho})}(t) 
   \leq &\{\mathcal{O}\Big(\sqrt{\max\bigg\{n^2 (n+m)^3\|P_*\|^{9-\frac{\bar{\rho}}{2}}\,\beta_*^2\big(\|P_*\|,\bar{\rho}\big)\,\; , n^2(n+m)m^2\|P_*\|^{(11+\frac{\bar{\rho}}{2})}\bigg\}  \, t\; \log^4\frac{t}{\delta}}\,\Big)
\end{align*}
where 
\begin{align*}
    \beta_*\big(\|P_*\|,\bar{\rho}\big) =\begin{cases}
        \big(\|P_*\| +\|P_*\|^{(6-\bar{\rho})/2}\big)^{\frac{(6-\bar{\rho})}{4}} & \text{if } \bar{\rho}<6, \\
        2 & 6\leq\bar{\rho}<8.
    \end{cases} 
\end{align*}
The above bound is established on the event $\mathcal{E}^{\bar{\rho}}_t$, which holds with probability at least $1-5\delta$. Therefore, applying a union bound over the failure events associated with $\mathcal{E}^{a}_t$ and the regret analysis yields
\begin{align*}
   R_{\text{ARSLO}^+(\bar{\rho})}(t) 
   \leq &\{\mathcal{O}\Big(\sqrt{\max\bigg\{n^2 (n+m)^3\|P_*\|^{9-\frac{\bar{\rho}}{2}}\,\beta_*^2\big(\|P_*\|,\bar{\rho}\big)\,\; , n^2(n+m)m^2\|P_*\|^{(11+\frac{\bar{\rho}}{2})}\bigg\}  \, t\; \log^4\frac{t}{\delta}}\,\Big)
\end{align*}
with probability at least $1-6\delta$.

To optimize the regret bound, it suffices to choose $\bar{\rho}$ to balance the dominant terms in their dependence on $\|P_*\|$. Setting $\bar{\rho} = \bar{\rho}_* = 2$ achieves this balance. In terms of dimensional dependence, the dominant term is $n^2 (n+m)^3$, which yields 
\begin{align*}
   R_{\text{ARSLO}^+(\bar{\rho}_*)}(t) 
   \leq & \mathcal{O}\Big(\sqrt{\, n^2 (n+m)^3 \; \|P_*\|^{12}\; t \;\log^4\frac{t}{\delta} }\,\Big). 
\end{align*}
with holds probability $1-6\delta$ that completes the proof.
\end{proof}

\section {Supporting Analysis of Warm-up Phase}
Before providing the proof of Theorem \ref{thm:reg_warmup}, we need to specify an upper bound on the state norm during warm-up phase.

\begin{lemma} \label{lem:statewarmbnd}
Consider the linear system \((A_*, B_*)\) with initial state \(x_0 = 0\), and let \(K_0\) be a \((\kappa_0, \gamma_0)\)-stabilizing policy. Then, during the warm-up phase, the state satisfies
\begin{align}
    \|x_t\| \;\le\; \sigma_{\omega} \, \sqrt{10 \Bigl(n + m \kappa_0^2 \bar{\vartheta}_{B_*}^2\Bigr) \log \frac{t}{\delta}} \label{eq:stnrmwarmup}
\end{align}
and 
\begin{align}
    \|z_t\| \le \frac{12 \sigma_{\omega} \kappa_0^2}{\gamma_0} \sqrt{\Bigl(n + m \kappa_0^2 \bar{\vartheta}_{B_*}^2\Bigr) \log \frac{4 t}{\delta}}. \label{eq:znormbnwup}
\end{align}
with probability at least \(1-\delta\), for any \(\delta \in (0,1/e)\).
\end{lemma}

\begin{proof}
The closed-loop dynamics at any time $s$ can be written as
\begin{align*}
    x_{s+1}=(A_*+B_*K_0)x_s+B_*\nu_s+\omega_s.
\end{align*}
Unrolling this recursion to time $t$ gives
   \begin{align*}
        x_t=M_1x_1+\sum_{s=1}^{t-1}M_s(B_*\nu_{s}+\omega_{s+1}) 
    \end{align*}
   where 
   \begin{align*}
       M_t=I,\;\; M_{s}=M_{s+1}(A_*+B_*K_0)=\prod_{j=s}^{t-1}(A_*+B_*K_0), \; \forall 1\leq s\leq t-1. 
   \end{align*}

 From strong $(\kappa_0, \gamma_0)-$ stability of $K_0$ we have  decomposition $A_*+B_*K_0= H L H^{-1}$ with $\|L\|\leq (1-\gamma_0)$ and $\|H\|\|H^{-1}\|\leq \kappa_0$. This implies
  \begin{align*}
     \nonumber  \|M_s\|\leq & \|H\| \bigg(\prod_{j=s}^{t-1}\|L\|\; \bigg) \|H^{-1}\|\leq{\kappa_0} (1-\gamma_0)^{t-s}. 
  \end{align*}
 Therefore,
\begin{align*}
       \nonumber  \|x_t\|&\leq \kappa_0 (1-\gamma_0)^t x_0+ \max_{0\leq  s\leq t-1} \|(B_*\nu_{t-1-s}+\omega_{t-1-s})\| \|\sum_{s=0}^{t-1}M_s\|\\
       \nonumber &\leq \kappa_0 (1-\gamma_0)^t x_0+ \max_{0\leq  s\leq t-1} \|(B_*\nu_{t-1-s}+\omega_{t-1-s})\| \sum_{s=0}^{\infty}\|M_s\|\\
      \nonumber  &\leq\kappa_0 (1-\gamma_0)^t \|x_0\|+ \frac{\kappa_0}{\gamma_0}\max_{0\leq  s\leq t-1} \|(B_*\nu_{t-1-s}+\omega_{t-1-s})\| \\
       &\leq \kappa_0 e^{-\gamma_0 t}\|x_0\|+\frac{\kappa_0}{\gamma_0}\max_{0\leq  s\leq t-1} \|(B_*\nu_{s}+\omega_{s})\|
    \end{align*}

  Let $\zeta_t:=B_*\nu_{s}+\omega_{s}$. Then $\zeta_t \sim \mathcal{N}(0, (\sigma_{\omega}^2+2\kappa_0^2 \sigma_{\omega}^2 B_* B_*^\top)I)$, i.e., $\zeta_t$ is Gaussian. Applying the Hanson-Wright inequality (Lemma \ref{lem:HansonWright}) we obtain
  \begin{align*}
      \|B_*\nu_t+\omega_{t}\|^2&\leq 5 \operatorname{Tr}((\sigma_{\omega}^2+2\kappa_0^2 \sigma_{\omega}^2 B_*B_*^\top)I) \log \frac{T_0}{\delta}\\
      &\leq 5 \sigma_{\omega}^2(n+2\kappa_0^2 \operatorname{Tr}(B_*B_*^\top))\log \frac{T_0}{\delta}\\
      &\leq 10 \sigma_{\omega}^2(n+m\kappa_0^2\vartheta^2)\log \frac{T_0}{\delta}.
  \end{align*}
  Setting $x_0=0$ concludes the proof of the first claim.

  By definition, we also have
\begin{align*}
    z_k = \begin{pmatrix} x_k \\ K_0 x_k + \nu_k \end{pmatrix},
\end{align*}
which implies
\begin{align*}
    \|z_k\| &\le \|x_k\| + \|K_0 x_k\| + \|\nu_k\| \\
            &\le 2 \kappa_0 \|x_k\| + \|\nu_k\|.
\end{align*}

Applying (\ref{eq:stnrmwarmup}), with probability at least $1-\delta/2$, we have
\begin{align*}
    \|x_t\| \le \sigma_{\omega} \sqrt{10 \Bigl(n + m \kappa_0^2 \bar{\vartheta}_{B_*}^2\Bigr) \log \frac{4 t}{\delta}}.
\end{align*}

Moreover, by the Hanson-Wright inequality (Lemma \ref{lem:HansonWright}),
\begin{align*}
    \max_{1 \le k \le t} \|\nu_k\|^2 \le 10 \sigma_{\omega}^2 m \kappa_0^2 \log \frac{4 t}{\delta}.
\end{align*}

Combining these results gives
\begin{align*}
    \|z_t\| \le \frac{12 \sigma_{\omega} \kappa_0^2}{\gamma_0} \sqrt{\Bigl(n + m \kappa_0^2 \bar{\vartheta}_{B_*}^2\Bigr) \log \frac{4 t}{\delta}}
\end{align*}
that completes the proof. 
\end{proof}

\subsection{Proof of Theorem \ref{thm:reg_warmup}}

\begin{proof}
Let $t_w$ denote the duration of the warm-up phase which is determined by termination condition of the warm-up phase $\epsilon_w(t)\leq \varepsilon$:
\begin{align*}
    t_w = \min \Bigg\{ t \;\Big|\; \frac{r_t^w}{\lambda_{\min}(\bar{V}_t)} \le \varepsilon \Bigg\}.
\end{align*}

A lower bound on the minimum eigenvalue of the covariance matrix $\bar{V}_t$ follows from Theorem 20 in \cite{cohen2019learning}:
\begin{align*}
    \bar{V}_t \succeq \frac{t \sigma_{\omega}^2}{80} I, \quad t \ge 400(n+m+ \log \tfrac{1}{\delta}).
\end{align*}
Therefore, the condition on $t_w$ becomes
\begin{align}
    t_w \leq \min \Bigg\{ t \;\Big|\; t \ge \frac{80 r_t^w}{\sigma_{\omega}^2 \varepsilon^2} \Bigg\}. \label{eq:novelCondn}
\end{align}

Since $r_t^w = \mathcal{O}(\log t)$, it follows from \eqref{eq:novelCondn} that such a $t_w$ exists.

Using the definition of $r_t^w$ and the bound on $\|z_t\|$, we obtain
\begin{align*}
    r_t^w &\le 4 \sigma_{\omega}^2 n(n+m) \log \frac{n}{\delta} \\
          &\quad + 4 \sigma_{\omega}^2 n(n+m) \log\Bigl(1 + t \cdot \frac{144 \sigma_{\omega}^2 \kappa_0^4}{\gamma_0^2} (n + m \kappa_0^2) \log \frac{4 t}{\delta} \Bigr) \\
          &\quad + 2 \sigma_{\omega}^2 m \\
          &\le 10 \sigma_{\omega}^2 n(n+m) \log \frac{1200 (n+m) \sigma_{\omega}^2 \kappa_0^6}{\gamma_0^2} \frac{t}{\delta}.
\end{align*}

Then (\ref{eq:novelCondn}) can be simplified applying the inequality $x \ge a \log(b x)$ for $a \ge 3$, which holds if $x \ge 3 a \log(ab)$:

\begin{align*}
    t_w \leq \min \Bigg\{ t \ge 800 n(n+m) \frac{1}{\varepsilon^2} \log \frac{1200 (n+m) \sigma_{\omega}^2 \kappa_0^6}{\gamma_0^2 \delta}\bigg\}.
\end{align*}

For the ARSLO algorithm, using the definition $\varepsilon = \bar{\epsilon}(\bar{\kappa}_1)$, we obtain
\begin{align*}
    \frac{1}{\bar{\epsilon}^2(\bar{\kappa}_1)}
    &\lesssim n(n+m)\|P_*\|^{14}\bar{\vartheta}_{B_*}^4\vartheta^2 
    \log \frac{\|P_*\| n(n+m)\bar{\vartheta}_{B_*}\vartheta}{\delta} \\
    &\lesssim (n+m)^2\|P_*\|^{14}\bar{\vartheta}_{B_*}^6 
    \log \frac{\|P_*\|(n+m)\bar{\vartheta}_{B_*}}{\delta}.
\end{align*}
Consequently, the duration of the warm-up phase satisfies
\begin{align*}
    t_w
    \lesssim (n+m)^4\|P_*\|^{14}\bar{\vartheta}_{B_*}^6
    \log \frac{\|P_*\|(n+m)\bar{\vartheta}_{B_*}}{\delta}
    \log \frac{(n+m)\kappa_0}{\gamma_0 \delta}.
\end{align*}

For the ARSLO$^{+}(\bar{\rho})$ algorithm, we instead have
$\varepsilon = \underline{\epsilon}(\bar{\kappa}_1,\bar{\rho})$, which yields
\begin{align*}
    \frac{1}{\underline{\epsilon}^2(\bar{\kappa}_1,\bar{\rho})}
    &\lesssim n(n+m)^2\|P_*\|^{h}\bar{\vartheta}_{B_*}^4\vartheta^2 
    \log^2 \frac{\|P_*\| n(n+m)\bar{\vartheta}_{B_*}\vartheta}{\delta} \\
    &\lesssim (n+m)^3\|P_*\|^{h}\bar{\vartheta}_{B_*}^6 
    \log^2 \frac{\|P_*\|(n+m)\bar{\vartheta}_{B_*}}{\delta},
\end{align*}
where $h$ is defined in~\eqref{eq:Defhforre}. Therefore, the warm-up phase duration for ARSLO$^{+}(\bar{\rho})$ satisfies
\begin{align*}
    t_w
    \lesssim (n+m)^5\|P_*\|^{h}\bar{\vartheta}_{B_*}^6
    \log^2 \frac{\|P_*\|(n+m)\bar{\vartheta}_{B_*}}{\delta}
    \log \frac{(n+m)\kappa_0}{\gamma_0 \delta}.
\end{align*}

Additionally, from \eqref{eq:znormbnwup}, we have
\begin{align*}
    \|z_t\|^2 \lesssim \frac{\kappa_0^6}{\gamma_0^2} (n+m) \bar{\vartheta}_{B_*}^2 \log \frac{t}{\delta}.
\end{align*}

The regret accumulated during the warm-up phase can therefore be written as
\begin{align*}
    R_{\text{warm-up}}(t_w) &= \sum_{k=0}^{t_w-1} z_k^\top \begin{pmatrix} Q & 0 \\ 0 & R \end{pmatrix} z_k \\
    &\le \alpha_1 \sum_{k=0}^{t_w-1} \|z_k\|^2 \lesssim \alpha_1 t_w \|z_{t_w}\|^2.
\end{align*}

As a result, for the ARSLO algorithm, we obtain
\begin{align*}
    R_{\text{warm-up}}(t_w)
    \lesssim \alpha_1 (n+m)^5 \|P_*\|^{14} \bar{\vartheta}_{B_*}^8
    \frac{\kappa_0^6}{\gamma_0^2}
    \log^2 \frac{\|P_*\|(n+m)\bar{\vartheta}_{B_*}}{\delta}
    \log \frac{(n+m)\kappa_0}{\gamma_0 \delta},
\end{align*}
which implies
\begin{align*}
    R_{\text{warm-up}}(t_w)
    \le \mathcal{O}\!\left(
        \frac{\kappa_0^6\, \alpha_1\, \bar{\vartheta}_{B_*}^8}{\gamma_0^2}
        (n+m)^5 \|P_*\|^{14}
    \right).
\end{align*}

Similarly, for the ARSLO$^{+}(\bar{\rho})$ algorithm, we have
\begin{align*}
    R_{\text{warm-up}}(t_w)
    \lesssim \alpha_1 (n+m)^6 \|P_*\|^{h} \bar{\vartheta}_{B_*}^8
    \frac{\kappa_0^6}{\gamma_0^2}
    \log^3 \frac{\|P_*\|(n+m)\bar{\vartheta}_{B_*}}{\delta}
    \log \frac{(n+m)\kappa_0}{\gamma_0 \delta},
\end{align*}
and consequently,
\begin{align*}
    R_{\text{warm-up}}(t_w)
    \le \mathcal{O}\!\left(
        \frac{\kappa_0^6\, \alpha_1\, \bar{\vartheta}_{B_*}^8}{\gamma_0^2}
        (n+m)^6 \|P_*\|^{h}
    \right).
\end{align*}

This completes the proof.
\end{proof}

\section{Rationale for the Perturbation Noise Scaling} \label{prop:ordernoise}
Throughout this paper, for both proposed algorithms, we choose Gaussian perturbation noise
\(\eta_t \sim \mathcal{N}(0, \Gamma_t)\), with \(\|\Gamma_t\| = \mathcal{O}\!\big(1/\sqrt{t}\big)\). 
However, we have not yet discussed the rationale behind this particular choice. 
The following proposition formally justifies why this order of the covariance matrix is desirable for achieving a favorable regret upper bound, and explains how deviations from this scaling can worsen the regret bound. 
This choice of scaling has previously been employed in self-tuning–regulator–based algorithms, including CE-based approaches such as \cite{simchowitz2020naive} and \cite{jedra2022minimal}.

\begin{proposition} 
Consider the perturbation noise $\eta_t \sim \mathcal{N}\!\left(0, \Gamma_t\right)$ applied to the designed feedback control in the proposed algorithms. Any deviation of $\|\Gamma_t\|$ from the order $\mathcal{O}\!\left(1/\sqrt{t}\right)$ either prevents satisfaction of the sufficient stability conditions or increases the regret upper bound beyond order $\mathcal{O}(\sqrt{t})$.
\end{proposition}

\begin{proof}
We analyze two cases for the perturbation noise $\eta_t \sim \mathcal{N}(0,\Gamma_t)$, corresponding to deviations of $\|\Gamma_t\|$ from the nominal order $\mathcal{O}(t^{-1/2})$.

\textbf{Case (i):} $\|\Gamma_t\| = \mathcal{O}(t^{-1/2+\bar{\alpha}})$ for some $\bar{\alpha} > 0$.  

Applying the procedure for lower bounding the minimum eigenvalue of the covariance matrix $V_t$ (see Appendix~\ref{sec:MinEigVal}), we obtain
\[
\lambda_{\min}(V_t) = \mathcal{O}(t^{1/2-\bar{\alpha}}).
\]
Consider the sufficient stability conditions for the ARSLO and ARSLO$^+(\bar{\rho})$ algorithms, given in (\ref{eq: verygoodRep}) and (\ref{eq: verygoodRep2}), respectively. Since the exploration parameter $\mu_t$ scales on the order of $\mathcal{O}(\sqrt{t})$, the resulting growth rate of $\lambda_{\min}(V_t)$ is insufficient to satisfy these conditions. Consequently, this choice of perturbation noise prevents the fulfillment of the sufficient conditions required for closed-loop stability.

\textbf{Case (ii):} $\|\Gamma_t\| = \mathcal{O}(t^{-1/2-\bar{\alpha}})$ for some $\bar{\alpha} > 0$.  

In this case, the same argument (see Appendix~\ref{sec:MinEigVal}) yields
\[
\lambda_{\min}(V_t) = \mathcal{O}(t^{1/2+\bar{\alpha}}),
\]
which is sufficient to satisfy the stability conditions in (\ref{eq: verygoodRep}) and (\ref{eq: verygoodRep2}). However, under this choice of perturbation noise, an upper bound on the regret term $R_6(t)$ shows that its order increases from $\mathcal{O}(\sqrt{t})$ to $\mathcal{O}(t^{1/2+\bar{\alpha}})$, thereby degrading the overall regret bound.

Combining the two cases completes the proof.
\end{proof}

\vskip 0.2in
\bibliography{sample}

@inproceedings{abbasi2011regret,
  title={Regret bounds for the adaptive control of linear quadratic systems},
  author={Abbasi-Yadkori, Yasin and Szepesv{\'a}ri, Csaba},
  booktitle={Proceedings of the 24th Annual Conference on Learning Theory},
  pages={1--26},
  year={2011}
}

@article{cohen2019learning,
  title={Learning Linear-Quadratic Regulators Efficiently with only $\sqrt(T)$ Regret},
  author={Cohen, Alon and Koren, Tomer and Mansour, Yishay},
  journal={arXiv preprint arXiv:1902.06223},
  year={2019}
  }

@inproceedings{lale2022reinforcement,
  title={Reinforcement learning with fast stabilization in linear dynamical systems},
  author={Lale, Sahin and Azizzadenesheli, Kamyar and Hassibi, Babak and Anandkumar, Animashree},
  booktitle={International Conference on Artificial Intelligence and Statistics},
  pages={5354--5390},
  year={2022},
  organization={PMLR}
}

@inproceedings{fazel2018global,
  title={Global convergence of policy gradient methods for the linear quadratic regulator},
  author={Fazel, Maryam and Ge, Rong and Kakade, Sham and Mesbahi, Mehran},
  booktitle={International conference on machine learning},
  pages={1467--1476},
  year={2018},
  organization={PMLR}
}

@inproceedings{abbasi2019model,
  title={Model-free linear quadratic control via reduction to expert prediction},
  author={Abbasi-Yadkori, Yasin and Lazic, Nevena and Szepesv{\'a}ri, Csaba},
  booktitle={The 22nd International Conference on Artificial Intelligence and Statistics},
  pages={3108--3117},
  year={2019},
  organization={PMLR}
}

@article{campi1998adaptive,
  title={Adaptive linear quadratic Gaussian control: the cost-biased approach revisited},
  author={Campi, Marco C and Kumar, PR},
  journal={SIAM Journal on Control and Optimization},
  volume={36},
  number={6},
  pages={1890--1907},
  year={1998},
  publisher={SIAM}
}

@article{dean2018regret,
  title={Regret bounds for robust adaptive control of the linear quadratic regulator},
  author={Dean, Sarah and Mania, Horia and Matni, Nikolai and Recht, Benjamin and Tu, Stephen},
  journal={Advances in Neural Information Processing Systems},
  volume={31},
  year={2018}
}

@article{faradonbeh2020input,
  title={Input perturbations for adaptive control and learning},
  author={Faradonbeh, Mohamad Kazem Shirani and Tewari, Ambuj and Michailidis, George},
  journal={Automatica},
  volume={117},
  pages={108950},
  year={2020},
  publisher={Elsevier}
}

@article{mania2019certainty,
  title={Certainty equivalence is efficient for linear quadratic control},
  author={Mania, Horia and Tu, Stephen and Recht, Benjamin},
  journal={Advances in Neural Information Processing Systems},
  volume={32},
  year={2019}
}

@article{dean2020sample,
  title={On the sample complexity of the linear quadratic regulator},
  author={Dean, Sarah and Mania, Horia and Matni, Nikolai and Recht, Benjamin and Tu, Stephen},
  journal={Foundations of Computational Mathematics},
  volume={20},
  number={4},
  pages={633--679},
  year={2020},
  publisher={Springer}
}

@inproceedings{simchowitz2020naive,
  title={Naive exploration is optimal for online lqr},
  author={Simchowitz, Max and Foster, Dylan},
  booktitle={International Conference on Machine Learning},
  pages={8937--8948},
  year={2020},
  organization={PMLR}
}

@article{abbasi2011online,
  title={Online least squares estimation with self-normalized processes: An application to bandit problems},
  author={Abbasi-Yadkori, Yasin and P{\'a}l, D{\'a}vid and Szepesv{\'a}ri, Csaba},
  journal={arXiv preprint arXiv:1102.2670},
  year={2011}
}

@book{bertsekas2012dynamic,
  title={Dynamic programming and optimal control: Volume I},
  author={Bertsekas, Dimitri},
  volume={4},
  year={2012},
  publisher={Athena scientific}
}

@article{faradonbeh2020optimism,
  title={Optimism-based adaptive regulation of linear-quadratic systems},
  author={Faradonbeh, Mohamad Kazem Shirani and Tewari, Ambuj and Michailidis, George},
  journal={IEEE Transactions on Automatic Control},
  volume={66},
  number={4},
  pages={1802--1808},
  year={2020},
  publisher={IEEE}
}

@inproceedings{abeille2020efficient,
  title={Efficient optimistic exploration in linear-quadratic regulators via lagrangian relaxation},
  author={Abeille, Marc and Lazaric, Alessandro},
  booktitle={International Conference on Machine Learning},
  pages={23--31},
  year={2020},
  organization={PMLR}
}

@inproceedings{abeille2018improved,
  title={Improved regret bounds for thompson sampling in linear quadratic control problems},
  author={Abeille, Marc and Lazaric, Alessandro},
  booktitle={International Conference on Machine Learning},
  pages={1--9},
  year={2018},
  organization={PMLR}
}

@inproceedings{kargin2022thompson,
  title={Thompson Sampling Achieves $\mathcal{O}(\sqrt{T})$ Regret in Linear Quadratic Control},
  author={Kargin, Taylan and Lale, Sahin and Azizzadenesheli, Kamyar and Anandkumar, Animashree and Hassibi, Babak},
  booktitle={Conference on Learning Theory},
  pages={3235--3284},
  year={2022},
  organization={PMLR}
}

@inproceedings{cohen2018online,
  title={Online linear quadratic control},
  author={Cohen, Alon and Hasidim, Avinatan and Koren, Tomer and Lazic, Nevena and Mansour, Yishay and Talwar, Kunal},
  booktitle={International Conference on Machine Learning},
  pages={1029--1038},
  year={2018},
  organization={PMLR}
}

@inproceedings{jedra2022minimal,
  title={Minimal expected regret in linear quadratic control},
  author={Jedra, Yassir and Proutiere, Alexandre},
  booktitle={International Conference on Artificial Intelligence and Statistics},
  pages={10234--10321},
  year={2022},
  organization={PMLR}
}

@inproceedings{agarwal2019online,
  title={Online control with adversarial disturbances},
  author={Agarwal, Naman and Bullins, Brian and Hazan, Elad and Kakade, Sham and Singh, Karan},
  booktitle={International Conference on Machine Learning},
  pages={111--119},
  year={2019},
  organization={PMLR}
}

@book{liberzon2003switching,
  title={Switching in systems and control},
  author={Liberzon, Daniel},
  volume={190},
  year={2003},
  publisher={Springer}
}

@article{liberzon2002basic,
  title={Basic problems in stability and design of switched systems},
  author={Liberzon, Daniel and Morse, A Stephen},
  journal={IEEE control systems magazine},
  volume={19},
  number={5},
  pages={59--70},
  year={2002},
  publisher={IEEE}
}

@inproceedings{hespanha1999stability,
  title={Stability of switched systems with average dwell-time},
  author={Hespanha, Joao P and Morse, A Stephen},
  booktitle={Proceedings of the 38th IEEE conference on decision and control (Cat. No. 99CH36304)},
  volume={3},
  pages={2655--2660},
  year={1999},
  organization={IEEE}
}

@article{faradonbeh2017finite,
  title={Finite time analysis of optimal adaptive policies for linear-quadratic systems},
  author={Faradonbeh, Mohamad Kazem Shirani and Tewari, Ambuj and Michailidis, George},
  journal={arXiv preprint arXiv:1711.07230},
  year={2017}
}

@article{abbasi2011improved,
  title={Improved algorithms for linear stochastic bandits},
  author={Abbasi-Yadkori, Yasin and P{\'a}l, D{\'a}vid and Szepesv{\'a}ri, Csaba},
  journal={Advances in neural information processing systems},
  volume={24},
  year={2011}
}

@article{auer2002finite,
  title={Finite-time analysis of the multiarmed bandit problem},
  author={Auer, Peter and Cesa-Bianchi, Nicolo and Fischer, Paul},
  journal={Machine learning},
  volume={47},
  number={2},
  pages={235--256},
  year={2002},
  publisher={Springer}
}

\end{document}